\documentclass[11pt]{article}

\usepackage{jcd}
\bluehyperref
\usemedgeometry
\usepackage[numbers]{natbib}
\usepackage{algorithm}
\usepackage{algorithmic}
\usepackage{upgreek}
\usepackage{overpic}
\usepackage{pdfpages}

\usepackage{xr}
\newcommand{\score}{s}

\newcommand{\scorerv}{S}
\newcommand{\toto}{\rightrightarrows}

\newcommand{\rdmfcn}{\varphi}
\newcommand{\weak}{\textup{weak}}
\newcommand{\biject}{\mathfrak{S}}

\newcommand{\weakset}{W}
\newcommand{\weaklabel}{Y^\weak}

\usepackage{multirow}
\usepackage{tikz,forest}
\usetikzlibrary{arrows.meta}

\forestset{
    .style={
        for tree={
            base=bottom,
            child anchor=north,
            align=center,
            s sep+=1cm,
    straight edge/.style={
        edge path={\noexpand\path[\forestoption{edge},thick,-{Latex}] 
        (!u.parent anchor) -- (.child anchor);}
    },
    if n children={0}
        {tier=word, draw, thick, rectangle}
        {draw, diamond, thick, aspect=2},
    if n=1{%
        edge path={\noexpand\path[\forestoption{edge},thick,-{Latex}] 
        (!u.parent anchor) -| (.child anchor) node[pos=.2, above] {Y};}
        }{
        edge path={\noexpand\path[\forestoption{edge},thick,-{Latex}] 
        (!u.parent anchor) -| (.child anchor) node[pos=.2, above] {N};}
        }
        }
    }
}

\begin{document}

\begin{center}
  {\Large
    Predictive Inference with Weak Supervision}

  \vspace{.3cm}
  Maxime Cauchois$^1$ ~~~ Suyash Gupta$^1$ ~~~ Alnur Ali$^1$ ~~~ John Duchi$^{1,2}$

  \vspace{.2cm}
  Stanford University, Departments of $^1$Statistics and
  $^2$Electrical Engineering

  \vspace{.1cm}
  January 2022
\end{center}

% \title{Valid conformal predictive sets with partial labels}

% \maketitle

\begin{abstract}
  The expense of acquiring labels in large-scale statistical machine
  learning makes partially and weakly-labeled data attractive, though it is
  not always apparent how to leverage such data for model fitting or
  validation.  We present a methodology to bridge the gap between partial
  supervision and validation, developing a conformal prediction framework to
  provide valid predictive confidence sets---sets that cover a true label
  with a prescribed probability, independent of the underlying
  distribution---using weakly labeled data.  To do so, we introduce a
  (necessary) new notion of coverage and predictive validity, then develop
  several application scenarios, providing efficient algorithms for
  classification and several large-scale structured prediction problems.
  We corroborate the hypothesis that the new coverage definition
  allows for tighter and more informative (but valid) confidence
  sets through several experiments.
\end{abstract}

%% \begin{abstract}
%%   In machine learning, acquisition and full labeling of new data can often
%%   be costly, time-consuming or even down right impossible if the label space
%%   is too large.  This may additionally lead the practitioner to favor these
%%   few strongly supervised examples for training purposes.  In this scenario,
%%   partially labeled data, obtained through user feedback or noisier
%%   ``high-level" classification, often provides a cheaper, more attractive
%%   alternative to this data scarcity.  On the other hand, conformal
%%   prediction provides a framework for uncertainty quantification, and
%%   transforms any supervised model into valid confidence sets with strong
%%   guarantees of coverage, but it can only do so by leveraging independent
%%   fully supervised data.

%%   In this work, we bridge the gap between partial supervision and conformal
%%   prediction: we introduce a new notion of coverage, more suited to
%%   partially supervised data, and propose 
%% \end{abstract}

%% TODO: Make sure W is given a weak mnemonic

\section{Introduction}
%When monitoring and evaluating existing models in production, a practitioner wishes to detect when and if the model falls below an acceptable level of performance. 
%This is a standard yet challenging task.
%Indeed, even if the practitioner is continuously monitoring her model and tracking its performance, she often only receives partial feedback, through unsupervised or at best only partially labeled data, which makes it harder to detect a significant drop in performance.
%Of course, postulating that such drop will eventually occur implicitly assumes that the evaluation of a model on out-of-sample data prior to deployment is insufficient to ensure generalization to future data. 
%This hypothesis is admittedly debatable, but an extensive line of work~\cite{SugiyamaKrMu07, Quionero-CandelaSuSc09,
%Ben-DavidBlCrPeVa10, WenYuGr14, LiptonWaSm18, CauchoisGuAlDu20} has already showed how distribution shift can greatly affect model performance over time.
%In consequence,  we propose in this paper methods allowing to evaluate our existing models, especially when the feedback available is scarce or incomplete.  

Consider the typical supervised learning pipeline that we teach students
learning statistical machine learning: we collect data in $(X, Y)$ pairs,
where $Y$ is a label or target to be predicted; we pick a model and loss
measuring the fidelity of the model to observed data; we choose the model
minimizing the loss and validate it on held-out data.  This picture obscures
what is becoming one of the major challenges in this endeavor: that of
actually collecting high-quality labeled
data~\cite{SculleyHoGoDaPhEbChYoCrDe15, Donoho17, RatnerBaEhFrWuRe17}.
Hand labeling large-scale training sets is often impractically expensive.
Consider, as simple motivation, a ranking problem: a prediction is an
ordered list of a set of items, yet available feedback is likely to be
incomplete and partial, such as a top element (for example, in web search a
user clicks on a single preferred link, or in a grocery, an individual buys
one kind of milk but provides no feedback on the other brands
present). Developing methods to leverage such partial and weak feedback is
therefore becoming a major focus, and researchers have developed methods to
transform weak and noisy labels into a dataset with strong,
``gold-standard'' labels~\cite{RatnerBaEhFrWuRe17, ZhangReCaDeRaShWaWu17}.

In this paper, we adopt this weakly labeled setting, but instead of
considering model fitting and the construction of strong labels, we focus on
validation, model confidence, and predictive inference, moving beyond point
predictions and single labels. Our goal is to develop methods to rigorously
quantify the confidence a practitioner should have in a model given only weak
labels.  First consider the standard supervised learning scenario for data
$(X, Y) \in \mc{X} \times \mc{Y}$: here, given a desired confidence level
$\alpha$, the goal, rather than to provide point estimates
$\what{Y}$ of $Y$ given $X$, is to give a confidence set mapping $\what{C}_n$
based on $(X_i, Y_i)_{i = 1}^n$ that guarantees the distribution-free
coverage
\begin{equation}
  \label{eqn:conformal-inference-guarantee}
  \P\left[ Y_{n+1} \in \what{C}_n(X_{n+1}) \right] \ge 1 - \alpha,
\end{equation}
where $(X_{n+1}, Y_{n + 1})$ is a new observation following the same
distribution as the first $n$ points.  Conformal inference provides
precisely these guarantees~\cite{VovkGaSh05, Lei14, LeiWa14, LeiGSRiTiWa18,
  BarberCaRaTi19a}.

There are many scenarios, however, where it is natural to transition away
from this strongly supervised setting with fully labeled examples.  Above we
note ranking: individuals are very unlikely to provide full
feedback~\cite{AilonMoNe08, DuchiMaJo13, NegahbanOhSh16}. In multi-label
image classification~\cite{BoutellLuShBr04, ElisseeffWe01}, a labeler may
identify a few items in a given scene but not all, leading to partial
labeled feedback. A major challenge in industrial machine learning
deployment is to monitor models once they are in production, where it
may be challenging to collect high-quality labels, but weak supervision---in
the form of clicks on a recommended website, or agreeing to a suggested text
message completion---is relatively easy and cheap to collect.
In all of these, developing valid confidence sets and measures for
our predictions is of growing importance, as we wish for models
to be trustable, usable, verifiable.

With this as motivation, we consider supervised learning problems where,
instead of directly observing the ground truth labels $\{ Y_i \}_{i=1}^{n}$,
we observe only noisy partial labeling.  We make this formal in two
equivalent ways. In the first, for each instance $i \in [n]$, there exists a
(random) function $\rdmfcn_i : \mc{Y} \to \mc{Y}^\weak$, that belongs to a
set $\Phi \subset \{ \mc{Y} \to \mc{Y}^\weak \}$ of partially supervising
functions, such that we only observe $\weaklabel_i = \rdmfcn_i(Y_i)$.
Equivalently, the pair $(\weaklabel_i, \rdmfcn_i)$ specifies a weak set
$\weakset_i \subset \mc{Y}$ that contains $Y_i$:
\begin{align}
  \label{eqn:weak-set-from-function}
  \weakset_i  \defeq \left\{ y \in \mc{Y} \mid  \rdmfcn_i(y)
  = \weaklabel_i \right\} \subset \mc{Y},
\end{align}
so that we observe a set $\weakset_i$ consistent with $Y_i$.  Instead of
observing strong labels $(X_i, Y_i) \simiid P$, we thus observe
only $(X_i, \rdmfcn_i, \weaklabel_i)_{i=1}^n$.
We consider two fundamental questions in this weakly-labeled setting:
(i) for what is it even possible to provide (distribution-free) coverage (e.g.,
true labels $Y$, weak labels $\weakset$, or something else), and
(ii) what methods can guarantee coverage?

While a first goal would be to produce a confidence mapping using $(X_i,
\weaklabel_i, \rdmfcn_i)_{i=1}^n$ guaranteeing the
coverage~\eqref{eqn:conformal-inference-guarantee}, as we prove in
Section~\ref{subsec:impossible-strong-coverage}, this is impossible without
further distribution assumptions and would in general produce uninformative
confidence sets.  We therefore relax our coverage desiderata, instead
seeking a confidence set $\what{C}_n : \mc{X} \toto \mc{Y}$ that covers the
weak counterpart $\weaklabel_{n+1} = \rdmfcn_{n+1}(Y_{n+1})$ of the true
label in the sense that
\begin{align}
  \label{eqn:partial-conformal-inference-coverage}
  \P\left[ \what{C}_n(X_{n+1}) \cap \weakset_{n+1} \neq \emptyset \right] 
  =
  \P \left[ \exists y \in \what{C}_n(X_{n+1}) ~
    \mbox{s.t.} ~ \rdmfcn_{n+1}(y) = \weaklabel_{n+1}  \right] \ge 1-\alpha.
\end{align}
One could certainly use standard conformal inference
approaches~\cite{VovkGaSh05, Lei14, BarberCaRaTi19a} to obtain this, but
naive application would yield predictions in the space $\mc{Y}^\weak$ of
weak labels---we wish to return usable labels and configurations in the
actual target space $\mc{Y}$ of interest.  The
condition~\eqref{eqn:partial-conformal-inference-coverage} is weaker than
the initial coverage~\eqref{eqn:conformal-inference-guarantee}, and any
confidence set satisfying the former will also satisfy the latter, though it
allows smaller confidence set sizes. A major challenge is that the function
$\rdmfcn_{n+1}$ representing an individual's weak supervision is \emph{a
  priori} unknown (e.g., we do not know precisely what items a labeler will
label in an image ahead of time).  Indeed, if we observe $\rdmfcn_{n+1}$
prior to our prediction, a simple and optimal way to achieve
coverage~\eqref{eqn:partial-conformal-inference-coverage} follows. First,
construct a valid confidence set mapping $\what{C}_{n,\weak} : \mc{X} \toto
\mc{Y}^\weak$ for $\weaklabel_{n+1}$ using classical conformal
methodology~\cite{VovkGaSh05, Lei14, BarberCaRaTi19a}, which as
in~\eqref{eqn:conformal-inference-guarantee} would guarantee
$\P(Y^\weak_{n+1} \in \what{C}_{n,\weak}(X_{n+1})) \ge 1 - \alpha$, Then
define $\what{C}_n(x) \subset \mc{Y}$ to include a single $y \in \mc{Y}$ for
each $y^\weak \in \what{C}_{n,\weak}(x)$ such that $\rdmfcn_{n+1}(y) =
y^\weak$.  Such an approach is unfortunately impossible: we do not know
ahead of time if an individual cares only about the top item of her ranking
or requires a ranking accurate up to the $10$th item.

Given the subtleties of
coverage~\eqref{eqn:partial-conformal-inference-coverage}, we dedicate
Section~\ref{sec:weak-conformal} to question (i) above: what types of
coverage are even possible. We devote
Sections~\ref{sec:weak-supervised-scores}
and~\ref{sec:structured-prediction-scores} to question (ii): the development
of methodologies that can guarantee the
coverage~\eqref{eqn:partial-conformal-inference-coverage}.  In the former
(Sec.~\ref{sec:weak-supervised-scores}), we provide a general recipe, while
in Section~\ref{sec:structured-prediction-scores} we provide more tailored
methods for large output spaces, such as those in structured prediction.
To provide some initial insights into the methods and potential
applications, we provide experiments on several real-world domains in
Section~\ref{sec:experiments}.

%% \jcdcomment{ In the above and the related work, what needs to happen is an
%%   actual description of what other people are doing that's related to
%%   us. For example, we need to describe the Bates approach, and show why it's
%%   \emph{different} than what we're going to do.

%%   Additionally, we should describe something on nested conformal
%%   methods---assuming they are actually relevant---as in the
%%   \cite{GuptaKuRa19} paper maybe? If we're going to use someone's results,
%%   we should describe why/how we'll use them, otherwise omit.}

\subsection{Related Work}

An extensive line of work addresses prediction with partially labeled data.
The major focus is on strong label recovery under weak supervision,
including in multiclass~\cite{CourSaTa11, NguyenCa08} and
multilabel~\cite{YuJaKaDh14} tasks as well as structured prediction
problems, such as ranking~\cite{HullermeierFuChBr08, KorbaGaAl18},
segmentation~\cite{TriggsVe08, PapandreouChMuYu15}, and natural language
processing~\cite{FernandesBr11, MayhewChTsRo19}. More recent work tackles
constructing strongly labeled datasets from disparate weak supervision
tasks~\cite{RatnerBaEhFrWuRe17, ZhangReCaDeRaShWaWu17}, while the
papers~\cite{CidGaSa14, RooyenWi18, CabannesRuBa20} provide generic
theoretical conditions allowing strong label recovery. Yet this literature
focuses primarily on point prediction problems, where a model only returns a
single label with the (putative) highest likelihood, in contrast to our
confidence-based approach, which provides calibrated uncertainty estimates
and guarantees valid confidence sets with virtually no distributional
assumptions.

Our work also connects to the substantial literature on conformal inference,
where the goal is to provide valid predictive confidence
sets~\eqref{eqn:conformal-inference-guarantee}. \citet{VovkGaSh05} introduce
the main techniques---that examples are exchangeable, and so essentially can
provide $p$-values for significance of one-another---and suggest the simple
and generic split-conformal algorithm for building valid confidence sets.
Essentially all conformalized confidence sets offer the coverage
guarantee~\eqref{eqn:conformal-inference-guarantee}, so it is of interest to
improve various aspects of the mappings $\what{C}_n$.  For example, works
focus on improving the precision of these methods and optimizing average
confidence set size~\cite{LeiGSRiTiWa18, SadinleLeWa19, HechtlingerPoWa19,
  RomanoBaSaCa19, AngelopoulosBaMaJo20, BarberCaRaTi19a}, or on bridging the
gap with other forms of coverage, like classwise~\cite{SadinleLeWa19} or
conditional~\cite{RomanoPaCa19, BarberCaRaTi19a, CauchoisGuDu21,
  RomanoSeCa20} coverage.

Along these lines, \citet{BatesAnLeMaJo21} generalize conformal inference to
offer error control with respect to loss functions beyond the 0-1 loss
(coverage or non-coverage) central to the
guarantee~\eqref{eqn:conformal-inference-guarantee}, taking structured
prediction problems as motivation---as we do. \citet{BatesAnLeMaJo21} focus
on settings where the loss function naturally reflects the structure of the
label space $\mc{Y}$, such as hierarchical classification problems where one
wishes to label an example $X$ at a resolution (level of the tree)
appropriate to the confidence with which it can be labeled. We view our
approaches as complementary to theirs: their approaches make sense for
scenarios with fully labeled data in which a particular loss function is
natural, for example in tree-structured hierarchical classification, where a
prediction can be made at a given level in the tree.  Conversely, 
approaches are sensible when one receives weakly supervised data and wishes
to make a single good prediction; think of a grocery store deciding which of
a large collection of shaving creams to stock, a ranking problem with the
where one wishes to make sure that each individual's desired shaving cream
is stocked (in this context, this is the
guarantee~\eqref{eqn:partial-conformal-inference-coverage}). 
In that respect, our approach relates to the expanded admission problem~\cite{FischScJaBa21}, which allows for multiple labels to be ``admissible",  except that we do not observe to strongly supervised labels in our setting.
Consequently,
we motivate our distinct coverage guarantees from a set of impossibility
results we present in the next section.  
Additionally, we pay special
attention (see Sec.~\ref{sec:structured-prediction-scores}) to developing
practical algorithms that scale to large label spaces, an important
consideration with real-world weak supervision.  

%% Finally, we mention that one of our algorithmic contributions to come (as
%% well as the framework of \citet{BatesAnLeMaJo21}) builds on the recent
%% ``nested sets'' interpretation of conformal inference, due to
%% \citet{GuptaKuRa21}.  In this interpretation, conformal inference
%% is---arguably, more naturally---viewed as directly optimizing the prediction
%% set size, which facilitates a cleaner algorithmic development.

% Unlike our present paper,  their procedures leverage strongly supervised exchangeable data, whereas we adapt the conformal inference framework in the context of partially supervised data and propose practical methods that offer a coverage guarantee more suited to this type of tasks.

\paragraph{Notation}
Throughout this paper, $[n]$ stands for the set $\{1, 2, \dots, n \}$.  We
use $C: \mc{X} \toto \mc{Y}$ to denote a set valued mapping $C: \mc{X} \to
2^\mc{Y} \defeq \{ \weakset \mid \weakset \subset \mc{Y} \}$.  $P$ is either
the probability distribution generating the data $(X,Y, \rdmfcn) \in \mc{X}
\times \mc{Y} \times \Phi$, or equivalently $(X, Y, \weakset) \in \mc{X}
\times \mc{Y} \times 2^{\mc{Y}}$, as both notations are equivalent for our
purpose, and $U \sim \uniform[0,1]$ defines a uniform random variable on
$[0,1]$.  $\biject(U,V)$ is the set of bijections between two sets $U$ and
$V$, using the shorthand $\biject_K \defeq \biject([K],[K])$; $(i, j)$ is the
transposition of elements $i$ and $j \in [K]$, and for $k \in \N$, $\Delta_k
\defeq \{p\in \R_+^k \mid p^T \ones = 1\} $ is the space of probability
distributions on $[k]$.

% -*- mode: latex -*- %

\providecommand{\Pstrong}{P_{\textup{strong}}}
\providecommand{\Pweak}{P_{\textup{weak}}}
\providecommand{\detpart}{\textup{Det}}
\newcommand{\lebesgue}{\textup{Leb}}

\section{Conformal inference with weakly supervised data}
\label{sec:weak-conformal}

The starting point of this paper is to delineate realistic goals in
weak-conformal inference by determining what is actually possible---as we
show, a form of weak coverage---and what is unachievable.  To that end, we
demonstrate that strong coverage~\eqref{eqn:conformal-inference-guarantee},
while desirable, is impossible without further distributional
assumptions. We thus relax our goals, presenting a general weak
conformal scheme (Section~\ref{subsec:general-weak-conform-scheme}) that
relies on weakly supervised data.

%We then introduce Algorithm~\ref{alg:partial-supervised-conformal}, a simple scheme for achieving coverage~\eqref{eqn:partial-conformal-inference-coverage} that only requires any model trained on past data (either fully supervised or not).

\subsection{The strong coverage dilemma with partially supervised data}
\label{subsec:impossible-strong-coverage}

Consider a fully supervised classification setting with feature space $
\mc{X}$ and output space $\mc{Y}$, and let $\Pstrong$ be a joint
distribution on $\mc{X} \times \mc{Y}$ representing strong, as
opposed to weak, supervision. In this fully supervised setting, we
observe samples $(X_i,Y_i) \simiid
\Pstrong$, in contrast to observing the weak label set $\weakset
\subset \mc{Y}$ satisfying only $Y \in \weakset$.  We first require
definitions of consistency and validity.
\begin{definition}
  A probability distribution $P$ on $(X,Y, \weakset) \in \mc{X} \times \mc{Y}
  \times 2^{\mc{Y}}$ is \emph{consistent} if $P(Y \in W) =1$.  For any
  consistent distribution $P$, $\Pweak$ and $\Pstrong$ denote the marginal
  distributions of $(X,\weakset)$ and $(X,Y)$, respectively, when
  $(X,Y,\weakset) \sim P$.
\end{definition}

%The fully supervised setting is a special case of partial labeling, where $\weakset = \{ Y \}$ is always a singleton.

\begin{definition}
  Let $\what{C}_n : \mc{X} \toto \mc{Y}$ be a (potentially randomized)
  procedure depending only on the weakly supervised sample
  $(X_i,\weakset_i)_{i=1}^n \in \mc{X} \times 2^\mc{Y}$.  Then
  $\what{C}_n$ provides \emph{$(1-\alpha)$-strong distribution free coverage}
  if for
  all consistent distributions $P$ on $\mc{X} \times \mc{Y} \times 2^\mc{Y}$
  and $(X_i,Y_i,\weakset_i)_{i=1}^{n+1} \simiid P$, we have
  Eq.~\eqref{eqn:conformal-inference-guarantee}, i.e.
  \begin{align*}
    \P \left[ Y_{n+1} \in \what{C}_n(X_{n+1}) \right] \ge 1 - \alpha,
  \end{align*}
  and that $\what{C}_n$ provides \emph{$(1-\alpha)$-weak distribution free
  coverage} if it satisfies
  Eq.~\eqref{eqn:partial-conformal-inference-coverage}, i.e.
  \begin{align*}
    \P \left[ \weakset_{n+1} \cap \what{C}_n(X_{n+1}) \neq \emptyset \right]
    \ge 1 - \alpha.
  \end{align*}
\end{definition}

With these definitions, we can provide the (negative) result that, on
average over the data set, any procedure satisfying strong distribution free
coverage~\eqref{eqn:conformal-inference-guarantee} must include every
individual label $y \in \weakset_{n+1}$ with probability at least
$1-\alpha$. To formalize this, for a confidence set mapping
$\what{C}_n : \mc{X} \toto \mc{Y}$ constructed with
$(X_i, W_i)_{i=1}^n$, define the function
\begin{align*}
  p_n(x,y) \defeq \P \left( y  \in \what{C}_n(x) \right),
\end{align*}
which is the probability, taken over the weakly supervised sample $(X_i,
W_i)_{i=1}^n$, that $\what{C}_n(x)$ contains the potential label $y$.  We
prove the following theorem in
Appendix~\ref{proof-thm:lower-bound-strong-coverage}.
\begin{theorem}
  \label{thm:lower-bound-strong-coverage}
  Suppose that $\what{C}_n : \mc{X} \toto \mc{Y}$
  provides $(1-\alpha)$-strong distribution free coverage.
  Then for all consistent distributions $P$ on $\mc{X} \times
  \mc{Y} \times 2^\mc{Y}$,
  \begin{align*}
    \E_{  \{ (X_i, \weakset_i) \}_{i=1}^ {n+1} \simiid \Pweak}
    \left[ \inf_{y \in  \weakset_{n+1}}p_n(X_{n+1}, y)\right]
    \ge 1-\alpha.
  \end{align*}
\end{theorem}

Theorem~\ref{thm:lower-bound-strong-coverage} essentially
states that
$\what{C}_n$ simultaneously includes each element
$y \in \weakset_{n+1}$ with probability at least $1-\alpha$.
The theorem is generally not improvable,
as $\weakset_{n+1}$ need not be a subset of $\what{C}_n(X_{n+1})$.  Indeed,
think of the trivial procedure $\what{C}_n$ that includes every label $y \in
\mc{Y}$ independently with probability $1-\alpha$: it obviously satisfies
strong distribution-free coverage but has no connection with
$\weakset_{n+1}$. As an additional immediate corollary, if the sets
$\weakset$ contain at least a fixed number of labels, then so does
$\what{C}_n(X_{n+1})$.

\begin{corollary}
  \label{cor:strong-cvg-size-dilemma}
  Suppose that $\what{C}_n : \mc{X} \toto \mc{Y}$ provides
  $(1-\alpha)$-strong distribution free coverage, and that $P(| \weakset|
  \ge L) = 1$ for some $L \ge 1$. Then
  \begin{align*}
    \E_{ \{ (X_i, \weakset_i) \}_{i=1}^ {n+1} \simiid \Pweak} \left[ \big|\what{C}_n(X_{n+1}) \big| \right] \ge L(1-\alpha).
  \end{align*}
\end{corollary}
\begin{proof}
  By Theorem~\ref{thm:lower-bound-strong-coverage},
  \begin{align*}
    \E\left[
      %% \E_{ \{ (X_i, \weakset_i) \}_{i=1}^ {n+1} \simiid \Pweak} \left[
      \big|\what{C}_n(X_{n+1}) \big|  \right]
    &= \E \left[ \sum_{y \in \mc{Y}} p_n(X_{n+1},y) \right]
    \ge \E\left[ |\weakset_{n+1}| \inf_{y \in \weakset_{n+1}} p_n(X_{n+1},y)  \right]
    \ge L(1-\alpha)
  \end{align*}
  as claimed.
  %%by Theorem~\ref{thm:lower-bound-strong-coverage}.
\end{proof}

An alternative perspective is to consider large-sample limits; often, the
procedure $\what{C}_n$ converges to some population confidence set mapping
$C : \mc{X} \toto \mc{Y}$ as $n \to \infty$, in that
\begin{equation}
  \label{eqn:limiting-C}
  \E\left[\big|\what{C}_n(X) \setdiff C(X)\big| \right] \to 0
\end{equation}
as $n \to \infty$, where the expectation is over both the construction of
$\what{C}_n$ and $X$ independent of $(X_i, W_i) \simiid \Pweak$. 
Typically, the limiting $C$ is a (nearly) deterministic
function\footnote{ In some cases, we use randomization over a single label
  to guarantee that $\P(Y \in C(X)) = 1 - \alpha$} of $x$; for example, the
standard construction~\cite[e.g.][]{VovkGaSh05, Lei14, BarberCaRaTi19a} 
takes $C(x) = \{y \in \mc{Y} \mid s(x, y) \le \tau\}$ for some scoring
function $s : \mc{X} \times \mc{Y} \to \R$ and threshold $\tau$, which is
deterministic.  In this case, we can show that we nearly have
$\weakset \subset C(X)$, so $C(X)$ must be large whenever $\weakset$ is. To
formalize, let
\begin{equation*}
  \detpart_C(x) \defeq \left\{ y
  \in \mc{Y} \mid \P(y \in C(x)) \in \{0, 1\} \right\}
\end{equation*}
be the labels that are deterministically in \emph{or} out of $C(x)$
(where the probability is over any randomization in the mapping $C$)
so that $\detpart_C(x) = \mc{Y}$ whenever $C$ is deterministic.
Then can show that $\weakset \subset C(X)$ with probability at least
$1 - \alpha$:
\begin{corollary}
  \label{cor:pop-conf-set-big-aSS}
  Suppose that $\what{C}_n : \mc{X} \toto \mc{Y}$ provides
  $(1-\alpha)$-strong distribution free coverage and satisfies the
  limit~\eqref{eqn:limiting-C}. Then
  \begin{align*}
    \P( \weakset \cap \detpart_C(X) \subset C(X) ) \ge 1-\alpha.
  \end{align*}
\end{corollary}
\noindent
See
Appendix~\ref{proof-cor:pop-conf-set-big-aSS} for a proof.

Theorem~\ref{thm:lower-bound-strong-coverage} and its corollaries suggest
that any procedure achieving strong (distribution-free) coverage necessarily
produces inefficient (large) confidence sets when one uses only weakly
supervised data. Even in cases where there is implicitly a single correct
label, such as the structured prediction problems \citet{CabannesRuBa20}
consider, where the weak labels $w$ that a single $x$ supports (those for
which $\P(\weakset = w \mid X = x) > 0$) have a single label $y$ in their
intersection $\cap_{w : \P(w \mid x) > 0} \{w\} = \{y\}$, never disallow
large weak sets $\weakset$. We thus must take a different tack, targeting
new coverage desiderata.

\paragraph{An aside: regression.}

Although we primarily focus on classification (where $\mc{Y}$ is finite),
our development applies equally to regression or other problems with
continuous or infinite response sets, e.g., $\mc{Y} = \R$, as nothing in
Theorems~\ref{thm:lower-bound-strong-coverage}
or~\ref{thm:partial-supervised-conformal} requires $\mc{Y}$ to be any
particular space. We leverage this in our
experiments (Sec.~\ref{sec:experiments}) in the sequel to give numerical
examples, touching on the $\R$-valued case here to
demonstrate the analogues of our theoretical results.

As an example, weak sets $\weakset$ in the continuous case may
be intervals, arising, for example, from measurements with limited
resolution.  
We adapt Corollary~\ref{cor:strong-cvg-size-dilemma} to regression by
replacing counting measure with the Lebesgue measure $\lebesgue$, where
the response set $\mc{Y} = \R$ and the weak sets $\weakset \subset \R$. Assuming
that the weak sets all have a minimal volume, any valid confidence
set mapping necessarily is large (on average) as well:

\begin{corollary}
  \label{cor:strong-cvg-size-dilemma-regression}
  Suppose that $\what{C}_n : \mc{X} \toto \mc{Y} = \R$ provides
  $(1-\alpha)$-strong distribution free coverage, and let $L > 0$. If
  $P(\lebesgue(\weakset_i) \ge L ) = 1$, for $i=1,\ldots,n+1$, then
  \begin{align*}
    \E \Big[\lebesgue\big(\what{C}_n(X_{n+1}) \big) \Big] \ge L(1-\alpha).
  \end{align*}
\end{corollary}
\noindent
An analogy to Corollary~\ref{cor:pop-conf-set-big-aSS} follows
nearly immediately from Corollary~\ref{cor:strong-cvg-size-dilemma-regression},
just as Corollary~\ref{cor:pop-conf-set-big-aSS} follows
Corollary~\ref{cor:strong-cvg-size-dilemma}. That is,
if there exists a deterministic set-valued mapping $C : \mc{X} \toto \mc{Y}$
for which $\E[\lebesgue(\what{C}_n(X) \setdiff C(X))] \to 0$, then
$\P(\weakset \subset \what{C}_n(X_{n+1})) \ge 1 - \alpha$.

%% Besides, even low-noise assumptions akin to~\cite{CabannesRuBa20},
%% where, conditionally on $x \in \mc{X}$, there is exactly one label that
%% always belongs to $\weakset$, do not prevent $\weakset$ from getting very
%% large.

%% Theorem~\ref{thm:lower-bound-strong-coverage} and Corollary~\ref{cor:pop-conf-set-big-aSS} thus suggest that in presence of partial supervision, one should not necessarily target strong distribution free coverage. 
%% For instance,  low-noise conditions akin to~\cite{CabannesRuBa20} assume that, conditionally on $x \in \mc{X}$, there is exactly one label $y^\star(x) \in \mc{Y}$ such that $\P(y^\star(x) \in \weakset \mid X=x) = 1$, but that does not prevent $\weakset$ to grow prohibitively large: in the most unfavorable case, we would have $y^\star(x) = 1$, and $\weakset \in \{ [K] \setminus \{ j \} \mid 2\le j\le K\}$, which leads to $|\weakset| = K-1$.

\subsection{A general weak-conformal scheme via scoring functions}
\label{subsec:general-weak-conform-scheme}

The theoretical limitations we identify motivate the weak
coverage~\eqref{eqn:partial-conformal-inference-coverage} we target instead
of the strong converage~\eqref{eqn:conformal-inference-guarantee}. Following
our discussion above, the new coverage definition stems from two desiderata:
if the problem is actually low-noise and there already exists a highly
predictive model we can leverage to build our confidence sets---roughly,
that conditional on $x$, a single label $y$ belongs to the weak sets
$\weakset$ with high probability and a model exists that can predict this
$y$---then while we should return this singleton even if we cannot guarantee
strong coverage. In the alternative perspective that we care only about the
value $\rdmfcn(Y)$---recall the weak
set~\eqref{eqn:weak-set-from-function}---providing any $y$ satisfying
$\rdmfcn(y) = \weaklabel$ should suffice for prediction.
We turn now to provide our general weak conformalization scheme.

Our starting point is via the typical output of a machine-learned model, a
scoring function $\score : \mc{X} \times \mc{Y} \to \R$ that ranks potential
labels (or responses) $y$ for an input example $x \in \mc{X}$.  We treat
$\score(x,y)$ is a \textit{non-conformity} score, meaning the model predicts
that values of $y$ for which $s(x,y)$ is small are more likely. Standard
examples of such scoring functions include $\score (x,y) \defeq -| y -
\what{\mu}(x)|$ in regression, where $\what{\mu} : \mc{X} \to \R$ predicts
$y \mid x$; or $\score (x,y) \defeq - \log p_y(x)$ in multiclass
classification, where $p_y(x)$ models the conditional probability of $y \mid
x$.  Throughout this section, we adopt a split-conformal
perspective~\cite{VovkGaSh05, BarberCaRaTi19a}, assuming the practitioner
provides a scoring function independent of the sample $(X_i, \rdmfcn_i,
\weaklabel_i)_{i=1}^n$ (the sample would typically be a \emph{validation
  set}), and we show how to transform any such scoring function into a valid
(weakly-covering) confidence mapping.

Algorithm~\ref{alg:partial-supervised-conformal} starts from a simple
observation, assuming that the scoring function $\score$ is relatively
good. Given a query function $\rdmfcn$ and weak label
$\weaklabel$---equivalently, the weak set $\weakset = \{y \mid \rdmfcn(y) =
\weaklabel\}$---the most likely label should typically be the $y$ satisfying
$\rdmfcn(y) = \weaklabel$ minimizing $\score(x, y)$.

%% Consider, for a partial label $\weaklabel \in \mc{Y}^\weak$
%% and a random query function $\rdmfcn$, the \textit{most likely label} $y \in
%% \mc{Y}$ to be the one with minimal score $\score(x,y)$, while satisfying
%% $\rdmfcn(y) = \weaklabel$: if the scoring function induces a consistent
%% ordering of labels, i.e., that higher probability labels have lower scores,
%% there is a high chance that $y$ is the strong label $Y$ itself.  The partial
%% supervision allows us to restrict ourselves to a smaller set $\weakset$ of
%% potential labels (depending on $\weaklabel$ and $\rdmfcn$), while the
%% scoring function provides an order for the remaining ones.

\begin{algorithm}
  \caption{Partially supervised conformalization}
  \label{alg:partial-supervised-conformal}
  \begin{algorithmic}
    \STATE {\bf Input:} sample $\{(X_i,\weaklabel_i, \rdmfcn_i)
    \}_{i=1}^{n}$; score function $\score: \mc{X} \times \mc{Y} \to \R$
    independent of the sample; desired coverage $1-\alpha \in (0,1)$

    \STATE For each $i \in [n]$, compute
    \begin{align}
      \label{eqn:min-partial-score}
      \scorerv_i \defeq \min_{y: ~ \rdmfcn_i(y) = \weaklabel_i} \score(X_i, y).
    \end{align}

    \STATE Set 
    $\what{t}_n \defeq (1+n^{-1}) (1-\alpha) \text{-quantile of} ~ \{ \scorerv_i  \}_{i=1}^n$.
    \STATE \textbf{Return:} predictive set mapping
    $\what{C}_n : \mc{X} \toto \mc{Y}$
    defined by
    \begin{equation*}
      \what{C}_n(x) \defeq
      \left\{ y \in \mc{Y} \mid \score(x,y) \le \what{t}_n \right\}.
    \end{equation*}
  \end{algorithmic}
\end{algorithm}

\noindent
Note that a completely equivalent scheme to the
scores~\eqref{eqn:min-partial-score} with label mappings $\rdmfcn$ and weak
labels $\weaklabel_i$ uses weak sets $\weakset_i$, where we
replace the scores~\eqref{eqn:min-partial-score} with
\begin{equation*}
  \scorerv_i \defeq \min_{y \in \weakset_i} \score(X_i, y).
\end{equation*}
In either case,
Algorithm~\ref{alg:partial-supervised-conformal} achieves valid
coverage weak~\eqref{eqn:partial-conformal-inference-coverage}:

\begin{theorem}
  \label{thm:partial-supervised-conformal}
  Let $(X_i, Y_i, \rdmfcn_i)_{i=1}^{n+1} \simiid P$ and $\weaklabel_i =
  \rdmfcn_i(Y_i)$ for $i \in [n+1]$. Then
  Algorithm~\ref{alg:partial-supervised-conformal} returns a confidence set
  mapping satisfying
  \begin{align*}
    \P \left[ \text{There exists } y \in \what{C}_n(X_{n+1})~ s.t ~ \rdmfcn_{n+1}(y) = \rdmfcn_{n+1}(Y_{n+1})  = \weaklabel_{n+1}  \right] \ge 1-\alpha.
  \end{align*}
\end{theorem}
\begin{proof}
  Let $\scorerv_i \defeq \min_{\rdmfcn_i(y) = \weaklabel_i} \score(X_i,y)$
  for each $i \in [n+1]$.  By definition of $\what{C}_n$, we have
  \begin{align*}
    \left\{y\in \what{C}_n(X_{n+1}) \mid \rdmfcn_{n+1}(y)
    = \weaklabel_{n+1} \right\}
    = \left\{y\in \mc{Y} \mid \rdmfcn_{n+1}(y) = \weaklabel_{n+1}
    \text{ and } \score(X_{n+1},y) \le \what{t}_n \right\},
  \end{align*}
  which is nonempty if and only if
  \begin{align*}
    \scorerv_{n+1} \defeq
    \min_{\rdmfcn_{n+1}(y)  = \weaklabel_{n+1}} \score(X_{n+1}, y) \le \what{t}_n.
  \end{align*}
  As $\{ \scorerv_i \}_{i=1}^{n+1}$ are i.i.d., this occurs with
  probability at least $1-\alpha$~\cite[e.g.][Lemma 1]{TibshiraniBaCaRa19}.
\end{proof}

% AA TODO: need to double check
%The argument is similar to that of Theorem 2 in Barber et al.~(2020).

\section{Constructing effective conformal prediction sets}
\label{sec:weak-supervised-scores}

Algorithm~\ref{alg:partial-supervised-conformal} provides a generic method
for conformalization in the presence of partially supervised data, and it
makes no assumptions on the input score function
$\score$. Though the coverage
guarantee~\eqref{eqn:partial-conformal-inference-coverage} holds
regardless, we can delineate a few additional desiderata that
the predictive sets and score functions $\score$ should satisfy
to make them more practically useful, which is our focus in this section:
\begin{itemize}
\item The score function $\score$ must allow the practitioner to efficiently
  carry out the computation of the partial infimum
  scores~\eqref{eqn:min-partial-score}.
\item The lower level sets $\what{C}_n(x) = \{y \in \mc{Y} \mid \score(x, y)
  \le \what{t}_n\}$ should be efficiently representable.
\item The confidence sets $\what{C}_n(x)$ should be small, as smaller
  confidence sets (for a fixed confidence level $\alpha$) carry more
  information.
\end{itemize}
Deferring our discussion of computational efficiency to
Section~\ref{sec:structured-prediction-scores}, in this section we only
focus on the last desideratum, implicitly assuming computation is tractable
(e.g., if $\mc{Y}$ is small). We first
(Sec.~\ref{subsec:size-optimal-scores}) develop conditions sufficient for
optimally-sized confidence sets to even exist---a few subtleties
arise---before giving greedy algorithms for confidence set-size
minimization, describing their properties, and providing a few optimality
guarantees in Sections~\ref{subsec:greedy-algorithm}
and~\ref{subsec:submodular-optim-bound}.

\subsection{Size-optimal scoring mechanism}
\label{subsec:size-optimal-scores}

As in standard approaches to conformal inference~\cite{Lei14, LeiWa14,
  BarberCaRaTi19a}, we aim to construct a confidence set mapping
$\what{C}_n$ with minimal average size over $X \sim P_X$. Our starting point
is simply to define size-optimality, where to achieve exact coverage and
size guarantees, we allow randomization of our confidence sets via an
independent variable $U \sim \uniform[0,1]$.

\newcounter{savecounter}
\setcounter{savecounter}{\theequation}

\begin{definition}
  A randomized confidence set mapping $C_{1-\alpha} : \mc{X} \times [0,1]
  \toto \mc{Y}$ is \emph{marginally size-optimal} at level
  $\alpha$ if it solves
  \renewcommand{\theequation}{\textsc{Marg}}
  \begin{equation}
    \label{eqn:marginal-optimal-score-partial-supervision}
    \begin{split}
      \minimize_{C : \mc{X} \times [0,1] \toto \mc{Y}}
      & ~ \E_{X, U \sim \uniform[0,1]} \left[ |C(X,U)| \right] \\
      \subjectto & ~
      \P( \weakset \cap C(X,U) \neq \emptyset) \ge 1 - \alpha.
    \end{split}
  \end{equation}
  It is \emph{conditionally size-optimal} at level $\alpha$
  if for almost every
  $x \in \mc{X}$, $C(x, \cdot)$ solves
  \renewcommand{\theequation}{\textsc{Cond}}
  \begin{align}
      \label{eqn:conditional-optimal-score-partial-supervision}
      \minimize_{C: [0,1] \toto \mc{Y}} ~
      \left\{ \E_{U \sim \uniform[0,1]} \left[ |C(U)| \right] ~\mbox{s.t.}~
      ~ \P( \weakset \cap C(U) \neq \emptyset \mid X=x) \ge 1- \alpha \right\}.
    \end{align}
\end{definition}

\setcounter{equation}{\thesavecounter}

Even with full knowledge of the distribution $P$, techniques for finding
marginally size-optimal confidence
sets~\eqref{eqn:marginal-optimal-score-partial-supervision} are not
immediately apparent; as a consequence, we focus on the conditional case
first. Even in this case, it is in general non-trivial to obtain smallest
confidence sets. Yet as we follow the standard
practice~\cite{BarberCaRaTi19a, Lei14, VovkGaSh05} in conformal prediction
of defining confidence sets via the scores $\score$ (recall
Alg.~\ref{alg:partial-supervised-conformal}) as
$C_t(x) = \{y \mid \score(x, y) \le t\}$, our confidence
sets have the natural nesting property that
$C_t(x) \subset C_{t'}(x)$ whenever $t < t'$. Abstracting away the
particular form of $C$ to enable a purely set-based focus,
we thus consider nested confidence sets, where we show that
optimality guarantees are possible.

%% %\begin{definition}
%% %A randomized confidence set mapping  $C^\text{Cond}_{1-\alpha}  : \mc{X} \times [0,1] \toto \mc{Y}$ is conditionally size-optimal if it satisfies, for almost every $x \in \mc{X}$,
%% %\begin{align}
%% %\label{eqn:conditional-optimal-score-partial-supervision}
%% %C^\text{Cond}_{1-\alpha}(x, \cdot) \defeq &\argmin_{C: [0,1] \toto \mc{Y}} ~ \left\{ \E_{U \sim \mc{U}[0,1]} |C(U)| \mid
%% % ~ \P( \weakset \cap C(U) \neq \emptyset \mid X=x) \ge 1- \alpha \right\}.
%% %\end{align}
%% %\end{definition}
%% Computing a conditionally size-optimal mapping is not necessarily efficiently doable either, as it may require enumerating all the subsets of $\mc{Y}$, but we develop in Section~\ref{subsec:greedy-algorithm} an algorithm that computes it exactly under some conditions over $P$, and in general returns a close approximation.
%% % AA NOTE: what is \mc{Y}^\star?  should this just be \mc{Y}?
%% %Moreover, for conformalization purposes, we need to use a sequence of potential confidence set mappings that satisfies the following nestedness property.
%% We next introduce the following nestedness property. 

\begin{definition}
  \label{assptn:nested-confidence-sets}
  A collection of mappings $\{ C_\eta : \mc{X} \times [0,1] \toto
  \mc{Y}\}_{\eta \in (0,1)}$ is \emph{nested} if
  \begin{equation*}
    P ( C_{\eta_1}(X,U) \subset C_{\eta_2}(X,U) ) = 1
    ~ \text{for all} ~  0< \eta_1 < \eta_2 < 1.
  \end{equation*}
\end{definition}

\newcommand{\nest}{^\textup{nest}}

There is an immediate equivalence between
score-based conformalization schemes and nested collections of
confidence mappings~\cite[cf.][]{GuptaKuRa19}: we simply define
\begin{align}
  \label{eqn:score-optimal-conditional}
  \score\nest(x, y , u)
  \defeq \inf \left\{ \eta\in (0,1) \mid y \in C_{\eta}(x, u) \right\}.
\end{align}
The next lemma formalizes this equivalence (see
Appendix~\ref{proof:equivalence-score-nested} for a proof).
\begin{lemma}
  \label{lemma:equivalence-score-nested}
  Assume the confidence set mappings $\{C_\eta \}_{\eta \in (0,1)}$ are nested
  and $s\nest(x,y,U)$
  has continuous distribution for $U \sim \uniform[0,1]$. Then
  \begin{align*}
    C_{\eta}(x, U) = \left\{ y \in \mc{Y} \mid \score\nest(x,y,U) \le \eta
    \right\} ~~ \mbox{with~} U\mbox{-probability}~ 1.
  \end{align*}
\end{lemma}
\noindent
That is, obtaining weak coverage for nested confidence mappings
is equivalent to obtaining weak coverage using the
scoring function $\score\nest$, which
Alg.~\ref{alg:partial-supervised-conformal} provides, that is,
equivalent to choosing the smallest $\eta \in (0, 1)$ such that $\P(\weakset
\cap C_\eta(X, U) \neq \emptyset) \ge 1 - \alpha$.
A second useful distributional property of the nested
scores~\eqref{eqn:score-optimal-conditional} is that, assuming
the confidence sets $C_\eta$ are conditionally valid, we can provide
strong distributional results on $\score\nest$. To make this
precise, we say that $C_\eta$ is
\emph{conditionally valid for the weak labels $\weakset$} if
for each $\eta \in (0, 1)$ and with $\P$-probability $1$ over $X$,
\begin{equation}
  \label{eqn:conditional-weak-validity}
  \P\left(C_\eta(x, U) \cap \weakset \neq \emptyset
  \mid X = x\right) = \eta.
\end{equation}
We then have the following
uniformity property as an immediate
consequence of Lemma~\ref{lemma:equivalence-score-nested}:
\begin{lemma}
  \label{lemma:scores-uniform}
  In addition to the conditions of Lemma~\ref{lemma:equivalence-score-nested},
  assume that
  $C_\eta$ is conditionally valid~\eqref{eqn:conditional-weak-validity}
  for the weak label $\weakset$. Then
  the minimum score~\eqref{eqn:min-partial-score} is independent of $X$ and
  satisfies
  \begin{align*}
    \inf_{y \in \weakset} s\nest(x, y, U)  \sim \uniform[0,1].
  \end{align*}
\end{lemma}
\begin{proof}
  With $U$-probability 1,
  $\inf_{y \in \weakset} \score\nest (x, y, U)
  \le \eta$ if and only if $C_\eta(x, U) \cap \weakset \neq \emptyset$,
  and so
  $\P(\inf_{y \in \weakset} \score\nest(x, y, U) \le \eta \mid X = x)
  = \P(\weakset \cap C_\eta(x, U) \neq \emptyset \mid X = x)
  = \eta$.
\end{proof}
% AA NOTE: are there some "Cond's" missing in the above?  also, fyi, i think there were a few typos above, i.e., i replaced instances of "S" (the old notation) w/ "W" (the new notation) ...

%% \jcdcomment{
%%   I think we should add some additional commentary
%%   about when there are \emph{good} score functions that give more or less
%%   conditional coverage}

To illustrate this lemma, suppose that there exist nested conditionally
size-optimal mappings $\{ C^\textup{cond}_\eta \}_{\eta \in (0,1)}$ solving
problem~\eqref{eqn:conditional-optimal-score-partial-supervision}: in that
case, they satisfy the conditions for application of
Lemma~\ref{lemma:scores-uniform} so that the induced scores $\scorerv_i$ are
uniform; Alg.~\ref{alg:partial-supervised-conformal} will thus compute
$\what{t}_n = (1 - \alpha) + O_P(n^{-1/2})$ as $\what{t}_n$ is the
$(1-\alpha)$ quantile of $\scorerv_i \simiid \uniform[0,1]$.  So---in the
case that we have (near) conditional
coverage---Alg.~\ref{alg:partial-supervised-conformal} maintains it.
Notably, given a score function $\score$, not necessarily
the nested score~\eqref{eqn:score-optimal-conditional},
but strong in the sense that it models $(X, Y)$ well enough that
for each $\alpha$, we can choose $t$ so that
$\P(\score(x, Y) \le t \mid X = x) = \alpha$, then
the confidence sets Algorithm~\ref{alg:partial-supervised-conformal}
returns are indeed nested, and Lemma~\ref{lemma:scores-uniform}
applies to the induced nested score $\score\nest$.
Unfornately,
optimal
nested sets need not always exist (see Example~\ref{exm:cond-vs-cond=prox}
below), but we can provide natural conditions on the distribution
of $\weakset \mid X = x$ sufficient to allow such nested coverage,
which we do in the next subsection.

%If actually computable, these scores do not only guarantee weak coverage as in Theorem~\ref{thm:partial-supervised-conformal}, but even weak conditional coverage. 
%Of course, in practice, we can only approximate the confidence sets $C^\text{Cond}_{\eta}(X)$, but the conformalization step of Alg.~\ref{alg:partial-supervised-conformal} ensures that, even with noisy estimates $\what{C}^\text{Cond}_{\eta}(X)$  of $C^\text{Cond}_{\eta}(X)$ will produce confidence sets that satisfy Theorem~\ref{thm:partial-supervised-conformal}: the algorithm will simply return $\what{C}^\text{Cond}_{\what{t}_n}(X)$ instead of $\what{C}^\text{Cond}_{1-\alpha}(X)$, to account for the potential error in the sequence of confidence set  mappings.

\subsubsection{From conditionally to marginally valid confidence sets}

%Suppose that we can compute the sequence of confidence set mappings $\{  C^\text{Cond}_{\eta}\}_{\eta \in (0,1)}$, and for simplicity that they satisfy Assumption~\ref{assptn:nested-confidence-sets}.

\newcommand{\Ccond}{C^{\textup{cond}}}

Our initial criterion~\eqref{eqn:marginal-optimal-score-partial-supervision}
is purely marginal: we wish to compute a marginally size-optimal confidence
set. Conveniently, conditionally size-optimal mappings
can yield marginally size-optimal problems. In particular,
assume that the mappings
$\{\Ccond_\eta\}_{\eta \in (0, 1)}$ are conditionally
size-optimal~\eqref{eqn:conditional-optimal-score-partial-supervision}
and satisfy $\P(\weakset \cap \Ccond(x, U) \neq \emptyset \mid X=x)
\ge \eta$.
The following proposition shows how to transform these into marginally
size-optimal confidence sets.

\begin{proposition}
  Let the mappings $\{\Ccond_\eta\}$ be conditionally
  size-optimal~\eqref{eqn:conditional-optimal-score-partial-supervision}
  as above, and define the average
  $\textup{size}(x, \eta) \defeq \E_U[|\Ccond_\eta(x, U)|]$.
  Let $\score_{\textup{marg}}$ be any minimizer of
  \begin{align*}
    \E \left[ \textup{size}(X, \score(X))
      \right] ~~ \mbox{s.t. } ~~  \E[\score(X)] \ge 1 - \alpha
  \end{align*}
  over $\score : \mc{X} \to [0, 1]$.
  Then a solution to the initial marginal
  problem~\eqref{eqn:marginal-optimal-score-partial-supervision} is
  \begin{align*}
    C^\textup{marg}_{1-\alpha}(x,u) \defeq  \Ccond_{\score_\textup{marg}(x)}(x, u).
  \end{align*}
\end{proposition}
\noindent
More directly, any conditionally size-optimal sets---which are at least
easier to \emph{characterize} as they need only randomize over $U \sim
\uniform[0, 1]$---yield marginally size-optimal confidence sets in a
relatively straightforward way: one chooses the probability of miscoverage,
$\score(x)$, minimizing the expected confidence set size.

\begin{proof}
  That $C^\text{Marg}_{1-\alpha}$ provides valid $1 - \alpha$ coverage
  is nearly immediate: by conditional size optimality, we have
  $\P( \weakset \cap C^\text{marg}_{1-\alpha}(X , U) \neq \emptyset) =
  \E[t_\textup{marg}(X)] \ge 1-\alpha$.

  Let $C$ be any confidence set mapping such that $\P( \weakset \cap C(X,U)
  \neq \emptyset) \ge 1- \alpha$, and define $\score_C(x) \defeq \P(
  \weakset \cap C(x,U) \neq \emptyset \mid X=x) \in [0,1]$, which satisfies
  $\E[\score_C(X)] \ge 1-\alpha$.  By assumption on $\Ccond$, for each fixed
  $x \in \mc{X}$, the set $\Ccond_{\score_C(x)}(x,U)$ is
  size-optimal~\eqref{eqn:conditional-optimal-score-partial-supervision} at
  level $\score_C(x)$, so that for $P_X$-almost every $x \in
  \mc{X}$, we have
  \begin{align*}
    \textup{size}(x, \score_C(x)) = \E_{U \sim \uniform[0,1]} \left[
      |\Ccond_{\score_C(x)}(x,U))| \right] \le
    \E_{U \sim \uniform[0,1]} \left[ |C(x, U)| \right].
  \end{align*}
  Integrating both sides of the inequality over $X \sim P_X$, and using the
  assumed optimality condition on $\score_\textup{marg}$, we obtain
  \begin{align*}
    \E \left[ \textup{size}(X, \score_\textup{marg}(X))\right]
    \le \E \left[\textup{size}(X, \score_C(X)) \right]
    \le \E_{X, U} \left[ |C(X,U))| \right].
  \end{align*}
  The left-hand size is the average size of
  $C^\textup{margin}_{1-\alpha}$.
\end{proof}

%
%On the other hand, for a conditional coverage equal to $\eta \defeq \eta^\text{Marg}_{1-\alpha}(x)$, we know that the smallest $\eta$-conditionally valid confidence set precisely is $C^\text{Cond}_\eta (x,U)$, and it has an average size equal to $\Ccond(x, \eta)$. 
%It therefore only remains to find the distribution of $\eta :\mc{X} \to (0,1)$  that minimizes the overall average size. 
%In particular, $\eta^\text{Marg}_{1-\alpha}(x)$ will approach $1$ if $C^\text{Cond}_\eta (x,U)$ remains small even for high values of $\eta$, but will remain closer to $0$ when $C^\text{Cond}_\eta (x,U)$ grows large as $\eta \to 1$, for instance if the distribution of $S$ given $X=x$ has very high entropy.

\subsection{Greedy algorithms for confidence set-size minimization}
\label{subsec:greedy-algorithm}

\newcommand{\Cgreedy}[1][]{
  \ifthenelse{\isempty{#1}}{%
    C^\textup{gr}
  }{%
    C^{\textup{gr},{#1}}
  }
}

Given the distribution---or a model of the distribution---of the weak
set $\weakset$ conditional on $x$, we propose a natural greedy
algorithm to construct a confidence set satisfying the weak coverage
constraint: at each step,
Algorithm~\ref{alg:greedy-weakly-supervised-optimal-scoring} adds the
label that increases coverage the most until the confidence set
achieves a desired level. As we show presently, there are natural
families of distributions where this greedy algorithm is optimal;
however, there are failure modes, of which we also provide an
example. In Section~\ref{subsec:submodular-optim-bound}, we relate
this greedy construction to submodular optimization to provide
general guarantees of confidence set size and coverage.

\begin{algorithm}[h]
  \caption{Greedy weakly supervised scoring mechanism}
  \label{alg:greedy-weakly-supervised-optimal-scoring}
  \begin{algorithmic}
    \STATE {\bf Input:} model for the distribution of $\weakset$ given $X=x$;
    coverage rate $\eta \in (0,1)$

    \STATE \textbf{for each} $j \in [K]$ define recursively
    \begin{equation*}
      y_j(x) \defeq \argmax_{y \in \mc{Y}}
      \, P \!\left( y\in \weakset,
      \, \bigcap_{i = 1}^{j - 1} \{y_i(x) \not \in \weakset\}
      \mid  X = x \right).
    \end{equation*}
    %% \STATE Define
    %% \begin{align*}
    %%   y_1(x) \defeq \argmax_{y \in \mc{Y}} P \left( y\in \weakset \mid  X = x \right),
    %% \end{align*} 
    %% and for $j \in \{ 2, \dots  K\}$ (breaking ties uniformly at random),
    %% \begin{align*}
    %%   y_j(x) \defeq \argmax_{y \in \mc{Y}} P \left( y\in \weakset, y_1(x) \notin \weakset, \dots, y_{j-1}(x) \notin \weakset \mid  X = x \right).
    %% \end{align*}
    \STATE \textbf{for each} $j \in [K]$ define
    $\Cgreedy[j](x) \defeq \{ y_i(x) \mid i \le j\}$
    and \textbf{set}
    \begin{align*}
      j(x,\eta) \defeq \min \left\{ j \in [K] \mid
      P\left( \weakset \cap \Cgreedy[j](x) \neq \emptyset \mid X=x \right)
      \geq \eta \right\}.
    \end{align*}
    
    \STATE \textbf{set}
    \begin{align*}
      t_\eta(x) \defeq \dfrac{\eta - P( \Cgreedy[j(\eta,x) - 1](x, U)
        \cap \weakset \neq \emptyset \mid X=x )}{
        P ( \Cgreedy[j(x,\eta)](x, U) \cap \weakset \neq \emptyset
        \mid X=x )
        - P ( \Cgreedy[j(x,\eta) - 1](x, U) \cap \weakset \neq \emptyset
        \mid X=x )}.
    \end{align*}
    
    \STATE \textbf{return} function $\Cgreedy_\eta : \mc{X} \times [0,1]
    \toto \mc{Y}$ defined by
    \begin{align*}
      \Cgreedy_\eta(x, u) \defeq
      \begin{cases} 
        \Cgreedy[j(x,\eta)](x) &\mbox{if } u < t_\eta(x), \\
        \Cgreedy[j(x, \eta)-1](x) & \mbox{otherwise.}
      \end{cases}
\end{align*}
\end{algorithmic}
\end{algorithm}
% AA NOTE: i think you want to delete "\eta" from some of the subscripts above? => Yep thanks!

Alg.~\ref{alg:greedy-weakly-supervised-optimal-scoring} returns a
nested sequence $\{ \Cgreedy_\eta(x, U) \}_{\eta \in (0,1)}$, where $U
\sim \uniform[0, 1]$ randomizes to achieve an appropriate level.
While the sequence need not necessarily solve
problem~\eqref{eqn:conditional-optimal-score-partial-supervision} (see
Example~\ref{exm:cond-vs-cond=prox} to come), there are natural
sufficient conditions for
Algorithm~\ref{alg:greedy-weakly-supervised-optimal-scoring} to return
a size-optimal set, of which we present two.  As the first particular
case, consider that conditional on $x$, labels $y \in \mc{Y}$ belong
to $\weakset$ independently:
\begin{definition}
  \label{def:independence-structure-probability-distribution}
  A probability distribution $P$ on $\weakset \in
  2^{\mc{Y}}$ has \emph{label-independent structure} if $\{ \indic{y \in
    \weakset} \}_{y \in \mc{Y}}$ are independent random variables when
  $\weakset \sim P$.
\end{definition}
\noindent
We might expect $\weakset$ to exhibit label independence when all
labels $y \in \mc{Y}$ satisfy $\pi(y \mid x) \ll 1$, with the
exception of a single label $y^\star(x)$, for which $\pi(y^\star(x)
\mid x) \approx 1$, as will often be the case in low-noise
classification settings.

Another scenario occurs when the label space exhibits a hierarchical tree
structure, as one may expect in image classification~\cite{DengDoSoLiLiFe09}
or structured prediction tasks~\cite{CabannesRuBa20}. When the weak sets
$\weakset$ obey the same structure the distribution---they are subtrees of
the global tree---we say the labels have a tree structure (see
Figure~\ref{fig:exmp-distribu-tree}):
\begin{definition}
  \label{def:tree-structured-probability-distribution}
  A probability distribution $P$ on $\weakset \in
  2^{\mc{Y}}$ has a \emph{tree structure} if for all $w_1, w_2 \subset
  \mc{Y}$,
  \begin{align}
    \label{eqn:label-tree-structure}
    P(\weakset=w_1)>0 ~\text{and} ~ P(\weakset=w_2)>0 ~ \text{imply}
    ~ w_1 \cap w_2 \in \{ w_1, w_2, \emptyset \}.
  \end{align}
\end{definition}

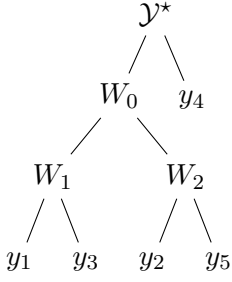
\begin{figure}[t]
  \begin{center}
    \begin{tabular}{cc}
      \begin{minipage}{.4\columnwidth}
        \centering
  \begin{forest} 
    [$\mc{Y}^\star$
      [$\weakset_0$
        [$\weakset_1$
          [$y_1$] 
          [$y_3$] 
        ]   
        [$\weakset_2$
          [$y_2$]   
          [$y_5$] 
        ]
      ]
      [$y_4$]
    ] 
  \end{forest}
      \end{minipage} &
      \begin{minipage}{.55\columnwidth}
        \caption{  \label{fig:exmp-distribu-tree}
          A tree-structured~\eqref{eqn:label-tree-structure}
          distribution for $\weakset$ given $X=x$, with $\mc{Y}^ \star =
          \{1, 2, 3, 4 \}$. The
          possible configurations for $\weakset$ are the singletons
          $\{y_1\}, \{y_2\}, \{y_3\}, \{y_4\}$,
          the two pairs
          $\weakset_1 = \{y_1,y_3\}$ and $\weakset_2 = \{y_2, y_5\}$,
          $\weakset_0 = \{y_1, y_2, y_3, y_5 \}$, and $\mc{Y}^\star$
          itself.}
      \end{minipage}
    \end{tabular}
  \end{center}
\end{figure}

Both definitions (independent labels and hierarchically-structured
weak labels) are sufficient to guarantee size-optimality for the
greedy confidence sets
Algorithm~\ref{alg:greedy-weakly-supervised-optimal-scoring}
constructs. The next Proposition, whose proof we provide in
Appendix~\ref{proof-prop:tree-structured-distribution}, makes this
formal.

\begin{proposition}
  \label{prop:tree-structured-distribution}
  Suppose the probability law $\mc{L}(\weakset \mid X=x)$ has either
  label-independent structure
  (Def.~\ref{def:independence-structure-probability-distribution}) or
  a tree structure
  (Def.~\ref{def:tree-structured-probability-distribution}).  Then for
  all $\eta \in (0,1)$, $\Cgreedy_\eta$ is conditionally
  size-optimal, and therefore is a minimizer in
  equation~\eqref{eqn:conditional-optimal-score-partial-supervision}.
\end{proposition}

In general, even with perfect knowledge of the distribution of
$\weakset \mid X = x$, the nested greedy confidence sets
$\Cgreedy_\eta$ need not be size-optimal, as there may be weak sets
appearing with high probability while their constituents do not, so
that the conditionally size-optimal sets $\{\Ccond_\eta\}_{\eta \in
  [0,1]}$ are not nested.  The next example illustrates one
such failure mode:

\begin{example}
  \label{exm:cond-vs-cond=prox}
  Let the distribution of
  $\weakset$ be
  \begin{align*}
    \weakset = 
    \begin{cases} 
      \left\{ 1, 2   \right\} &\mbox{w.p. } 0.3 \\
      \left\{  1,3 \right\} & \mbox{w.p. }   0.25 \\
      \left\{ 2\right\} & \mbox{w.p. }   0.2 \\
      \left\{ 3 \right\} & \mbox{w.p. }   0.15 \\
      \left\{ 1 \right\} & \mbox{w.p. }   0.1. \\
    \end{cases}
  \end{align*}
  Then for $\eta = 0.9$, it is immediate that
  $\Ccond_{\eta}(x, u) = \{ 2, 3\}$, but
  $\Cgreedy_{\eta}(x, u) = \{1, 2, 3\} \text{ or } \{ 1, 2\}$
  depending on whether $u < 1/3$.  In addition,
  $\Ccond_{eta'}(x, u) = \{ 1, 3\}$ when $\eta' = 0.85$,
  showing that in this case, the confidence set mappings
  $\Ccond_\eta$ need not be nested.
\end{example}

\subsection{A general upper bound for the greedy approach}  
\label{subsec:submodular-optim-bound}

As we saw in the previous section, reasonable conditions on label
distributions guarantee that the greedy mappings $\{\Cgreedy_\eta
\}_{\eta \in (0,1)}$ solve
problem~\eqref{eqn:conditional-optimal-score-partial-supervision}, while
pathologies (as in Example~\ref{exm:cond-vs-cond=prox}) exist.  In
this section, we show that even in general cases, the sizes of the
confidence sets $\Cgreedy$ and $\Ccond$ cannot be too far apart.
We motivate our approach by noting the similarity between
problem~\eqref{eqn:conditional-optimal-score-partial-supervision} and
the minimum set cover problem familiar in submodular
optimization~\cite{Vazirani01, GolovinKrSt14}, which we recall.  Let
$f : 2^\mc{Y} \to [0, 1]$ be a monotone submodular coverage function,
meaning that for each $A \subset B \subset \mc{Y}$ and $y \in \mc{Y}
\setminus B$, $f$ satisfies $f(A) \leq f(B)$, $f(A \cup \{y\}) - f(A)
\geq f(B \cup \{y\}) - f(B)$, $f(\mc{Y}) = 1$, and $f(\emptyset) = 0$.
A solution to the minimum set cover problem is
\begin{align}
  C^\star_\eta \in \argmin_{C \subset \mc{Y}} ~ \left\{ |C| ~ \textrm{s.t.}
  ~ f(C) \ge \eta \right\}. \label{eq:min-set-cover}
\end{align}
A classical result combinatorial optimization of \citet{Wolsey82}
bounds the size of the set that a natural greedy algorithm for
problem~\eqref{eq:min-set-cover} returns. To state the result, we
introduce a bit of notation.  For any set $C \subset \mc{Y}$ and $y
\in \mc{Y}$, we define
\begin{align*}
  \Delta(C, y) \defeq f(C \cup \{y\}) - f(C),
\end{align*}
increase in coverage from adding $y$ to $C$.  At each step $j \in
[K]$, the greedy algorithm chooses
\begin{align*}
  y_j \defeq \argmax_{y \in \mc{Y}} \Delta(\{ y_1 \dots, y_{j-1} \}, y),
\end{align*}
and stops at the first step $j(\eta) \le K$ such that $f( \{
y_1,\ldots,y_{j(\eta)} \} ) \ge \eta$. For the greedy set
$\Cgreedy[j] \defeq \{
y_1, \dots, y_{j} \}$, define the constant
\begin{align*}
  K_{f, \eta} \defeq \min \Bigg\{&
  \frac{\eta}{\eta - f(\Cgreedy[j(\eta)-1])}, \\
  & \max_{\substack{y \in \mc{Y}, \, j \le j(\eta) \\ \Delta(\Cgreedy[j], y)>0}}
  \Bigg( \frac{\Delta(\emptyset, y)}{\Delta(\Cgreedy[j], y)} \Bigg), 
  \frac{\max_{y \in \mc{Y}} \Delta(\emptyset, y)}{
    \max_{y \in \mc{Y} \setminus \Cgreedy[j(\eta) - 1]}
    \Delta(\Cgreedy[j(\eta) - 1], y)} 
\Bigg\}.
\end{align*}
We then have the following result.

\begin{lemma}[\citet{Wolsey82}, Theorem 1]
  \label{thm:wolsey}
  Let $f : 2^\mc{Y} \to [0,1]$ be a monotone submodular coverage function.
  Then
  \begin{align*}
    |\Cgreedy[j(\eta)]|
    \leq \Big(1 + \log K_{f,\eta} \Big) \cdot |C^\star_\eta|
  \end{align*}
\end{lemma}

Given the apparent similarity between the
problems~\eqref{eq:min-set-cover}
and~\eqref{eqn:conditional-optimal-score-partial-supervision}, we
would like to leverage Lemma~\ref{thm:wolsey} to establish a similar
guarantee for Alg.~\ref{alg:greedy-weakly-supervised-optimal-scoring}.
To apply Lemma~\ref{thm:wolsey} to
Alg.~\ref{alg:greedy-weakly-supervised-optimal-scoring}, we provide
the natural analogous quantities, leveraging the notation in the
algorithm and working conditional on $X = x$. Define $f_x(C) \defeq
P(\weakset \cap C \neq \emptyset \mid X = x)$, which is immediately a
submodular coverage function, and for each $x$ we have increment
function $\Delta_x(C, y) = P(\weakset \cap C = \emptyset, y \in
\weakset \mid X = x)$.  Because the greedy sets $\Cgreedy_\eta(x, u)$
may be randomized but always satisfy $\Cgreedy_\eta(x, 1) \subset
\Cgreedy_\eta(x, 0)$, we provide a slight alternative to the constant
$K_{f,\eta}$, defining
\begin{equation}
  \label{eqn:submodular-bound} 
  \begin{aligned}
    K_{P, \eta, x} \defeq \min \Bigg\{ & \frac{\eta}{\eta - P( \weakset \cap
      \Cgreedy_\eta(x, 1) \neq \emptyset \mid X=x )}, 
    \\ &  \hspace{-0.75in} 
    \max_{\substack{y \in \mc{Y}, \, j \le j(x, \eta) \\
        \Delta_x( \Cgreedy[j](x), y)>0}}
    \left( \frac{\Delta_x(\emptyset, y)}{\Delta_x(
      \Cgreedy[j](x), y)} \right), 
    % \\ &  \hspace{0.25in} 
    \frac{\max_{y \in \mc{Y}}
      \Delta(\emptyset, y)}{\max_{y \in \mc{Y}}
      \Delta(\Cgreedy_\eta(x, 1), y)} 
    \Bigg\}.
  \end{aligned}
\end{equation}
Invoking Lemma~\ref{thm:wolsey} and simplifying
gives the following result, which bounds the size of the greedy
set by a logarithmic quantity times the size of the best (deterministic)
covering set.

% AA NOTE: i think the 0 and 1 were reversed in the cor. statement, below?
\begin{corollary}
  \label{cor:submodular-deterministic-bound}
  Let $\Cgreedy_\eta : \mc{X} \times [0,1] \toto \mc{Y}$ be the
  confidence set mapping
  Algorithm~\ref{alg:greedy-weakly-supervised-optimal-scoring}
  outputs.  Then for all $x \in \mc{X}$ and $u \in [0, 1]$,
  \begin{align*}
    |\Cgreedy_\eta(x, u)| \le
    |\Cgreedy_\eta(x,0)|
    & \le
    \Big(1 + \log K_{P, \eta, x} \Big) \cdot
    \min_{C \subset \mc{Y}} \left\{
    |C| ~ \mbox{s.t.}~ P(\weakset \cap C \neq \emptyset \mid x) \ge \eta
    \right\}.
  \end{align*}
\end{corollary}

We can roughly interpret the three terms inside the minimum
in~\eqref{eqn:submodular-bound} as follows. The first term is large
when the greedy algorithm nearly attains the required coverage on the
iteration just before terminating, and therefore measures (in a sense)
how ``wasteful'' the algorithm is.  The second term is large when
choosing a label earlier would have improved the coverage more, and so
expresses a kind of regret.  The third term measures how often the
labels $y \in \mc{Y}$ co-occur in $\weakset$.  Though the bound is
a functional of the discrete derivative
$\Delta_x(C,y)$ and small when the ``local'' information
in $\Delta_x(C, y)$ gives good indicators of globally optimal sets $C$,
it can be hard to compute explicitly;
we therefore evaluate the size of the sets that
Alg.~\ref{alg:greedy-weakly-supervised-optimal-scoring} generates for
a few experimental examples in Section~\ref{sec:experiments}.

\section{Efficient conformalization for large output spaces}
\label{sec:structured-prediction-scores} 

%In Section~\ref{sec:weak-supervised-scores},  we developed an algorithm that only requires weakly labeled data to construct a scoring mechanism that we can feed Algorithm~\ref{alg:greedy-weakly-supervised-optimal-scoring} with.

%On the other hand, for a lot of machine learning tasks, we can leverage pre-trained models to construct our score functions.
%Our goal in this section is to provide methods for adapting these models into score functions that make the infimum score computation~\eqref{eqn:min-partial-score} computationally feasible and the resulting confidence set efficiently representable,  even in cases where the space $\mc{Y}$ is too large to enumerate all its elements. 
%We focus here on two structured prediction tasks: ranking and matching. 
%Theses two types of problems share the same label space $\mc{Y}  = \biject_K$, and the goal is, for each feature $X_i \in \mc{X}$, either to return a ranking of $K$ elements or a perfect matching of a bipartite graph with $2K$ vertices. 
%
%We describe for both tasks how to efficiently carry out each step of Alg.~\ref{alg:partial-supervised-conformal} when the score functions satisfy specific shape constraints.

While Section~\ref{sec:weak-supervised-scores} provides a generic treatment
on for producing scoring functions and associated confidence sets of minimal
size, in typical practice, a (pre-trained) model provides a predictive
scoring function, which may not be directly associated to a probability
metric, and we wish to leverage such models. This is of
particular interest when the label space $\mc{Y}$ is
large, as in structured prediction problems~\cite{Taskar05, CabannesRuBa20},
where computational efficiency becomes a main challenge.  In this
section, we thus first introduce a general method for computing and
representing confidence set mappings of the form $\{ y \mid \score(x,y) \le
\hat{t}_n\}$, and then describe how to efficiently carry out
Alg.~\ref{alg:partial-supervised-conformal} in ranking and matching problems.

\subsection{Conformal confidence sets with sequential partitioning}
\label{subsubsec:efficient-confidence-sets}

We seek to efficiently compute and represent the confidence set
$\what{C}_n(x)$ for any instance $x \in \mc{X}$, typically for a task where
the label space contains more configurations than are efficiently enumerable
($K!$ for matching and ranking problems over $K$ items). At the same time,
recalling that $\what{t}_n$ denotes the threshold
Algorithm~\ref{alg:partial-supervised-conformal}, if our confidence sets are
to be informative they should include relatively few configurations $y \in
\mc{Y}$ satisfying $ s(x, y) \le \what{t}_n$.  To the end of computing the
set $\what{C}_n(x) = \{y \mid \score(x, y) \le \what{t}_n\}$ in
Alg.~\ref{alg:partial-supervised-conformal},
we focus on methods for computing a given number
$M$ of configurations with the smallest score
$\score(x,y)$. This is essentially without loss of generality:
while we may not know the appropriate $M = M_x = |\what{C}_n(x)|$
to guarantee coverage,
if for each $M \in \N$ we can find the $M$ best configurations
in time polynomial in $M$, then by sequentially
doubling $M$ until we obtain
an element $y \in \mc{Y}$ such that $s(x,y) > \what{t}_n$, we
achieve time polynomial in $M_x$.
Algorithm~\ref{alg:efficient-computation-configurations} builds on this
intuition to return a valid confidence set.

We remark briefly that an alternative approach is to conformalize directly on
the size $M$ of the confidence set: suppose we learn a function $\what{M} :
\mc{X} \to \N$ predictive of the rank (according to $\{ \score(x,y) \}_{y
  \in \mc{Y}}$) of the first ``compatible" configuration, i.e predictive of
\begin{align*}
M_i \defeq \text{rank of the first configuration}~y\in \mc{Y}~\text{such that}~\rdmfcn_i(y) = Y_i^\weak.
\end{align*}
In that case,  if we let 
$
\what{Q}_n \defeq \left(1+n^{-1}\right) \left(1-\alpha\right) \text{-quantile of} ~ \{M_i  -  \what{M}(X_i)\}_{i=1}^n,
$
we would only need to return
\begin{align*}
  \what{C}_n(x) \defeq
  \left\{ \what{M}(x) +  \what{Q}_n ~\text{best configurations}~ y \in \mc{Y}
  ~\text{ordered by}~ s(x,y) \right\}.
\end{align*}
This approach makes prediction more efficient (as we know in advance the
number of configurations to compute), but the computational effort of the
conformalization step~\eqref{eqn:min-partial-score} increases, as we must
compute the rank of the best constrained configuration for each instance.

\subsubsection{Returning $M$ best configurations with sequential partitioning}

\tikzset{every picture/.style={line width=0.6pt}} %set default line width to 0.75pt        
\begin{figure}
\centering
\begin{tikzpicture}[x=0.65pt,y=0.65pt,yscale=-1,xscale=1]

%Shape: Circle [id:dp3893549514645547] 
\draw   (4,139) .. controls (4,90.68) and (43.18,51.5) .. (91.5,51.5) .. controls (139.82,51.5) and (179,90.68) .. (179,139) .. controls (179,187.32) and (139.82,226.5) .. (91.5,226.5) .. controls (43.18,226.5) and (4,187.32) .. (4,139) -- cycle ;
%Straight Lines [id:da7960086437763865] 
\draw    (29,78) -- (154,200) ;
%Straight Lines [id:da7477579202726059] 
\draw    (166,92) -- (111,158) ;
%Right Arrow [id:dp5082973932781738] 
\draw   (179,131) -- (217.4,131) -- (217.4,121) -- (243,141) -- (217.4,161) -- (217.4,151) -- (179,151) -- cycle ;
%Right Arrow [id:dp4299097507269909] 
\draw   (418,129) -- (453.4,129) -- (453.4,118.5) -- (477,139.5) -- (453.4,160.5) -- (453.4,150) -- (418,150) -- cycle ;
%Shape: Circle [id:dp3613836899106613] 
\draw   (243,139) .. controls (243,90.68) and (282.18,51.5) .. (330.5,51.5) .. controls (378.82,51.5) and (418,90.68) .. (418,139) .. controls (418,187.32) and (378.82,226.5) .. (330.5,226.5) .. controls (282.18,226.5) and (243,187.32) .. (243,139) -- cycle ;
%Straight Lines [id:da5241646100210418] 
\draw    (268,78) -- (393,200) ;
%Shape: Circle [id:dp8369495856087752] 
\draw   (477,139.5) .. controls (477,91.18) and (516.18,52) .. (564.5,52) .. controls (612.82,52) and (652,91.18) .. (652,139.5) .. controls (652,187.82) and (612.82,227) .. (564.5,227) .. controls (516.18,227) and (477,187.82) .. (477,139.5) -- cycle ;
%Straight Lines [id:da8504791090333597] 
\draw    (403.5,89) -- (348.5,155) ;
%Straight Lines [id:da49058003419630203] 
\draw    (639,95) -- (585,159) ;
%Straight Lines [id:da4667068794878637] 
\draw    (502,78.5) -- (627,200.5) ;
%Straight Lines [id:da1623654117027613] 
\draw [color={rgb, 255:red, 255; green, 0; blue, 0 }  ,draw opacity=1 ]   (249,170) -- (363,170) ;
%Straight Lines [id:da07462199992878116] 
\draw [color={rgb, 255:red, 255; green, 0; blue, 0 }  ,draw opacity=1 ]   (483,170) -- (597,170) ;
%Curve Lines [id:da05516956564382869] 
\draw    (60,41) .. controls (88.57,28.19) and (114.22,50.32) .. (108.29,82.52) ;
\draw [shift={(108,84)}, rotate = 281.98] [color={rgb, 255:red, 0; green, 0; blue, 0 }  ][line width=0.75]    (10.93,-3.29) .. controls (6.95,-1.4) and (3.31,-0.3) .. (0,0) .. controls (3.31,0.3) and (6.95,1.4) .. (10.93,3.29)   ;
%Curve Lines [id:da5914049251614841] 
\draw    (36,222) .. controls (59.28,223.94) and (64.68,223.06) .. (75.03,206.57) ;
\draw [shift={(76,205)}, rotate = 481.43] [color={rgb, 255:red, 0; green, 0; blue, 0 }  ][line width=0.75]    (10.93,-3.29) .. controls (6.95,-1.4) and (3.31,-0.3) .. (0,0) .. controls (3.31,0.3) and (6.95,1.4) .. (10.93,3.29)   ;
%Curve Lines [id:da049977326297285796] 
\draw    (178,188) .. controls (180.91,176.36) and (172.53,173.19) .. (154.68,168.44) ;
\draw [shift={(153,168)}, rotate = 374.74] [color={rgb, 255:red, 0; green, 0; blue, 0 }  ][line width=0.75]    (10.93,-3.29) .. controls (6.95,-1.4) and (3.31,-0.3) .. (0,0) .. controls (3.31,0.3) and (6.95,1.4) .. (10.93,3.29)   ;
%Curve Lines [id:da903042165873359] 
\draw    (518,227) .. controls (541.28,228.94) and (546.68,228.06) .. (557.03,211.57) ;
\draw [shift={(558,210)}, rotate = 481.43] [color={rgb, 255:red, 0; green, 0; blue, 0 }  ][line width=0.75]    (10.93,-3.29) .. controls (6.95,-1.4) and (3.31,-0.3) .. (0,0) .. controls (3.31,0.3) and (6.95,1.4) .. (10.93,3.29)   ;
%Curve Lines [id:da7520924934522966] 
\draw    (643,206) .. controls (645.9,194.42) and (635.75,183.77) .. (621.56,173.16) ;
\draw [shift={(620,172)}, rotate = 396.25] [color={rgb, 255:red, 0; green, 0; blue, 0 }  ][line width=0.75]    (10.93,-3.29) .. controls (6.95,-1.4) and (3.31,-0.3) .. (0,0) .. controls (3.31,0.3) and (6.95,1.4) .. (10.93,3.29)   ;
%Curve Lines [id:da906241956890359] 
\draw    (593,42) .. controls (575.36,34.16) and (560.6,47.45) .. (565.67,75.28) ;
\draw [shift={(566,77)}, rotate = 258.31] [color={rgb, 255:red, 0; green, 0; blue, 0 }  ][line width=0.75]    (10.93,-3.29) .. controls (6.95,-1.4) and (3.31,-0.3) .. (0,0) .. controls (3.31,0.3) and (6.95,1.4) .. (10.93,3.29)   ;
%Curve Lines [id:da4859670704413639] 
\draw    (489,81) .. controls (504.81,92.16) and (501.59,101.59) .. (501.06,110.09) ;
\draw [shift={(501,112)}, rotate = 270] [color={rgb, 255:red, 0; green, 0; blue, 0 }  ][line width=0.75]    (10.93,-3.29) .. controls (6.95,-1.4) and (3.31,-0.3) .. (0,0) .. controls (3.31,0.3) and (6.95,1.4) .. (10.93,3.29)   ;

% Text Node
\draw (178,100) node [anchor=north west][inner sep=0.75pt]  [font=\footnotesize]  {$ \textsc{Partition}$};
% Text Node
\draw (418,100) node [anchor=north west][inner sep=0.75pt]  [font=\footnotesize]  {$ \textsc{Update}$};
% Text Node
\draw (33,32.4) node [anchor=north west][inner sep=0.75pt]    {$\mc{Y}_{1}^{3}$};
% Text Node
\draw (12,210.4) node [anchor=north west][inner sep=0.75pt]    {$\mc{Y}_{2}^{3}$};
% Text Node
\draw (166,193.4) node [anchor=north west][inner sep=0.75pt]    {$\mc{Y}_{3}^{3}$};
% Text Node
\draw (491,213.4) node [anchor=north west][inner sep=0.75pt]    {$\mc{Y}_{2}^{4}$};
% Text Node
\draw (602,34.4) node [anchor=north west][inner sep=0.75pt]    {$\mc{Y}_{1}^{4}$};
% Text Node
\draw (462,64.4) node [anchor=north west][inner sep=0.75pt]    {$\mc{Y}_{4}^{4}$};
% Text Node
\draw (626,211.4) node [anchor=north west][inner sep=0.75pt]    {$\mc{Y}_{3}^{4}$};
% Text Node
\draw (63,81.4) node [anchor=north west][inner sep=0.75pt]  [font=\scriptsize]  {$y_{1}$};
% Text Node
\draw (103,191.4) node [anchor=north west][inner sep=0.75pt]  [font=\scriptsize]  {$y_{2}$};
% Text Node
\draw (28,126.4) node [anchor=north west][inner sep=0.75pt]  [font=\scriptsize]  {$y_{2,2}^{3}$};
% Text Node
\draw (98,117.4) node [anchor=north west][inner sep=0.75pt]  [font=\scriptsize]  {$y_{1,2}^{3}$};
% Text Node
\draw (146,122.4) node [anchor=north west][inner sep=0.75pt]  [font=\scriptsize]  {$y_{3,2}^{3}$};
% Text Node
\draw (128,151.4) node [anchor=north west][inner sep=0.75pt]  [font=\scriptsize]  {$y_3$};
% Text Node
\draw (291,76.4) node [anchor=north west][inner sep=0.75pt]  [font=\scriptsize]  {$y_{1}$};
% Text Node
\draw (340,120.4) node [anchor=north west][inner sep=0.75pt]  [font=\scriptsize]  {$y_{1,2}^{3}$};
% Text Node
\draw (378,125.4) node [anchor=north west][inner sep=0.75pt]  [font=\scriptsize]  {$y_{3,2}^{3}$};
% Text Node
\draw (364,149.4) node [anchor=north west][inner sep=0.75pt]  [font=\scriptsize]  {$y_{3}$};
% Text Node
\draw (337,188.4) node [anchor=north west][inner sep=0.75pt]  [font=\scriptsize]  {$y_{2}$};
% Text Node
\draw (255,124.4) node [anchor=north west][inner sep=0.75pt]  [font=\scriptsize,color={rgb, 255:red, 255; green, 0; blue, 0 }  ,opacity=1 ]  {$y_{2,2}^{3} =y_{4}$};
% Text Node
\draw (532,77.4) node [anchor=north west][inner sep=0.75pt]  [font=\scriptsize]  {$y_{1}$};
% Text Node
\draw (576,122.4) node [anchor=north west][inner sep=0.75pt]  [font=\scriptsize]  {$y_{1,2}^{4}$};
% Text Node
\draw (620,125.4) node [anchor=north west][inner sep=0.75pt]  [font=\scriptsize]  {$y_{3,2}^{4}$};
% Text Node
\draw (601,155.4) node [anchor=north west][inner sep=0.75pt]  [font=\scriptsize]  {$y_{3}$};
% Text Node
\draw (578,192.4) node [anchor=north west][inner sep=0.75pt]  [font=\scriptsize]  {$y_{2}$};
% Text Node
\draw (487,124.4) node [anchor=north west][inner sep=0.75pt]  [font=\scriptsize]  {$y_{4}$};
% Text Node
\draw (542,144.4) node [anchor=north west][inner sep=0.75pt]  [font=\scriptsize,color={rgb, 255:red, 255; green, 0; blue, 0 }  ,opacity=1 ]  {$y_{4,2}^{4}$};
% Text Node
\draw (511,175.4) node [anchor=north west][inner sep=0.75pt]  [font=\scriptsize,color={rgb, 255:red, 255; green, 0; blue, 0 }  ,opacity=1 ]  {$y_{2,2}^{4}$};

\end{tikzpicture}

\caption{Alg.~\ref{alg:efficient-computation-configurations} scheme for
  sequential partitioning: first, partition the subset containing the
  $m+1$-th best configuration, $y_{2,2}^3$ in this case, then compute both
  second-best configurations in the newly formed subsets of the
  partition---here $\mc{Y}_2^4$ and $\mc{Y}_4^4$. }
\label{fig:alg-config-scheme}
\end{figure}
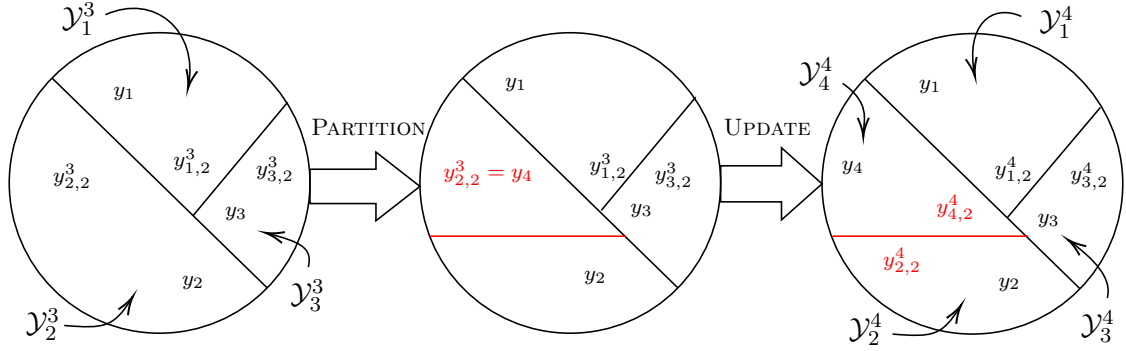

Let us now fix $M \ge 1$, and focus on retrieving the $M$ configurations
with the lowest scores.
Algorithm~\ref{alg:efficient-computation-configurations} provides a general
recipe using dynamic programming, and it is efficient as long
as we can efficiently compute certain partitions of the label space.
We require the following definition.
\begin{definition}
  \label{def:partition}
  A function $\textsc{Partition}: 2^\mc{Y} \times \mc{Y} \times \mc{Y} \to
  2^\mc{Y} \times 2^\mc{Y}$ is \emph{valid for a score function}
  $\score$ if, for
  every subset $\tilde{\mc{Y}} \subset \mc{Y}$ and pair of configurations
  $y_1, y_2 \in \tilde{\mc{Y}}$ satisfying
  \begin{align*}
    y_1 \in \argmin_{y \in \tilde{\mc{Y}}} \score(x,y)
    ~~ \mbox{and} ~~
    y_2 \in \argmin_{y \in \tilde{\mc{Y}} \setminus \{ y_1 \}} \score(x,y),
  \end{align*}
  \textsc{Partition}$(\tilde{\mc{Y}}, y_1, y_2)$ returns a partition
  $(\tilde{\mc{Y}}_1,\tilde{\mc{Y}}_2)$ of $\tilde{\mc{Y}}$ such that $y_1
  \in \tilde{\mc{Y}}_1$ and $y_2 \in \tilde{\mc{Y}}_2$.
\end{definition}
\noindent
We thus leverage two conditions: a valid \textsc{Partition} for our score
function $\score$ and, for each pair of subsets $\mc{Y}_1, \mc{Y}_2 \subset
\mc{Y}$ that it produces, we must be able to (efficiently) compute the
second-best configurations in $\mc{Y}_1$ and $\mc{Y}_2$, i.e.,
\begin{align*}
  y_{1,2} \in \argmin_{y \in \mc{Y}_1 \setminus \{ y_1\} } \score(x,y)
  ~~\text{and}~~
  y_{2,2} \in \argmin_{y \in \mc{Y}_2 \setminus \{ y_2\} } \score(x,y).
\end{align*}

Figure~\ref{fig:alg-config-scheme} encapsulates the main idea
Alg.~\ref{alg:efficient-computation-configurations}: at each step $m \in
[M]$, we maintain a partition $\{ \mc{Y}_j^m\}_{j=1}^m$ of $\mc{Y}$ such
that if $ y_{j}^m \in \argmin_{y \in \mc{Y}_j^m} s(x,y), $ then for all $j
\in [m]$, we have
\begin{align*}
  y_{j}^m \in \argmin_{y \in \mc{Y} \setminus \{ y_1^m, \dots, y_{j-1}^m \}} s(x,y),
\end{align*}
i.e., $y_{j}^m$ is the $j$-th best configuration in $\mc{Y}$.
Now, for each $j \in [m]$, let the configuration
$y_{j,2}^m \in \argmin_{y \in \mc{Y}_j^m \setminus \{ y_j^m \} } s(x,y)$
be the second-best configuration in $\mc{Y}_j^m$.
The key is then to observe that if we set
\begin{align*}
  \text{ind}(m) \defeq \argmin_{j \in [m]} s(x, y_{j,2}^m),
\end{align*}
then $y_{\text{ind}(m), 2}^m$ is the $(m+1)$st best configuration in
$\mc{Y}$.  The \textsc{Partition} function then divides
$\mc{Y}_{\text{ind}(m)}^m$ into two sets $\mc{Y}_{\text{ind}(m)}^{m+1}$ and
$\mc{Y}_{m+1}^{m+1}$ such that $y_{\text{ind}(m)}^m \in
\mc{Y}_{\text{ind}(m)}^{m+1}$ and $y_{\text{ind}(m),2}^m \in
\mc{Y}_{m+1}^{m+1}$.  Under the assumption that \textsc{Partition} is valid
(Def.~\ref{def:partition}) for the score $\score$, the following lemma
guarantees the validity of
Algorithm~\ref{alg:efficient-computation-configurations}.

\begin{algorithm}[t]
  \caption{Sequential partitioning}
  \label{alg:efficient-computation-configurations}
  \begin{algorithmic}
    \REQUIRE score function
    $\score: \mc{X} \times \mc{Y} \to \R$;
    valid (Def.~\ref{def:partition})
    \textsc{Partition}: $2^\mc{Y} \times  \mc{Y} \times \mc{Y} \to 2^\mc{Y} \times 2^\mc{Y}$;
    instance $x \in \mc{X}$
   
    \STATE \textbf{initialize:}
    Compute $y_1^1 \defeq \argmin_{y \in  \mc{Y}} s(x,y)$ and $y_{1,2}^1 \defeq \argmin_{y \in \mc{Y} \setminus \{ y_1^1 \}} s(x,y)$.
   \STATE Set $\mc{Y}^m_1 \defeq \mc{Y}$ \hfill\COMMENT{Initialize the partition}
   
   \FOR{$m = 1, 2, \ldots, M - 1$}
  
   \STATE $\text{ind}(m) \defeq \argmin_{j \in [m]} s(x, y_{j,2}^m)$
   \hfill\COMMENT{Find the $m+1$-th best configuration}
   \STATE $y_{m+1}^{m+1} \defeq y_{\text{ind}(m),2}^m$
   
   \FOR{$j \in [m] \setminus \{ \text{ind}(m) \}$}
   \STATE $( \mc{Y}^{m+1}_j, y_j^{m+1}, y_{j,  2}^{m+1}) \defeq 
   (\mc{Y}^m_j, y_j^m, y_{j,2}^m)$
   \hfill\COMMENT{All subsets $\{ \mc{Y}^m_j \}_{j \neq \text{ind}(m)}$ remain identical}
   \ENDFOR
   
   \STATE $\mc{Y}^{m+1}_{\text{ind}(m)},  \mc{Y}^{m+1}_{m+1} \defeq \textsc{Partition}( \mc{Y}^{m}_{\text{ind}(m)},  y_{\text{ind}(m)}^m,  y_{m+1}^{m+1})$   
   \hfill\COMMENT{Update the partition for $ \mc{Y}^m_{\text{ind}(m)}$}
   \STATE $y_{\text{ind}(m)}^{m+1} \defeq y_{\text{ind}(m)}^{m}$ 
   and $y_{\text{ind}(m),  2}^{m+1} \defeq \argmin_{y \in \mc{Y}^{m+1}_{\text{ind}(m)} \setminus \{ y_{\text{ind}(m)}^{m+1} \} } s(x,y)$ 
   \STATE $ y_{m+1, 2}^{m+1} \defeq \argmin_{y \in \mc{Y}^{m+1}_{m+1} \setminus \{ y_{m+1}^{m+1} \} } s(x,y)$
   \hfill\COMMENT{Compute second-best configurations}
      \ENDFOR

    \RETURN{$\{ y_m^M \}_{m=1}^M$}
      \end{algorithmic}
\end{algorithm}

\begin{lemma}
  \label{lem:validity-config-algo}
  Assume the \textsc{Partition} function is valid for the score
  function $\score$.  Then
  Algorithm~\ref{alg:efficient-computation-configurations} returns a set of
  configurations $\{ y_j \}_{j=1}^M$ such that for each $j \in [M]$,
  \begin{align*}
    y_j \in \argmin_{y \in \mc{Y} \setminus \{ y_1, \dots, y_{j-1} \}} s(x,y).
\end{align*}
\end{lemma}
\begin{proof}
  This follows by an induction over $m \ge 1$,  which guarantees that at every step $m \ge 1$,  $\{ \mc{Y}_j^m \}$ is a partition of $\mc{Y}$ such that $y_j^m = \argmin_{y \in \mc{Y}_j^m}$ and
  \begin{align*}
    s(x,y_1^m) \le s(x,y_2^m) \dots \le s(x,y_m^m) \le \min_{y \in \mc{Y} \setminus \{ y_j^m \} } s(x,y).
  \end{align*}
  The property transitions from $m$ to $m+1$ as the~\textsc{Partition}
  function is valid, and we choose $y_{m+1}^{m+1}$ as the best second-best
  configuration, hence it is the $(m+1)$st best configuration.
\end{proof}

The existence of an efficient valid \textsc{Partition} function is
instance-dependent and typically requires a specific choice of scoring
function; we provide concrete implementations for two
types of structured prediction problems.

\subsection{Two structured prediction examples}
\label{subsec:structured-pred-examples}

While Algorithm~\ref{alg:efficient-computation-configurations} is generic,
we now show that efficient partitioning and minimization functions exist in
two structured prediction instances, so that we may efficiently carry out
both Algs.~\ref{alg:partial-supervised-conformal}
and~\ref{alg:efficient-computation-configurations} in the instance.

\subsubsection{Perfect matching scores and weak supervision}
\label{subsubsec:matching-prediction-scores}

A matching task consists of trying to optimally pair elements of a bipartite
graph given a feature vector $x \in \mc{X}$ (for example, one may wish to
identify paired amino acids in protein folding~\cite{TaskarGuKo03}). We
assume there exists a bipartite graph $G$ with disjoint sets $U$ and $V$ of
$K \ge 1$ nodes; each label $Y$ is then a perfect matching between $U$ and
$V$, i.e., a bijection $Y \in \mc{Y} = \biject(U,V)$.
General supervised approaches for perfect matching problems, such as
structured Support Vector Machines~\cite{TsochantaridisHoJoAl04} or
Adversarial Bipartite Matching~\cite{FathonyBeZhZi18}, generally learn
pairwise score functions $\varphi_{u,v} : \mc{X} \to \R$ for all $(u,v) \in
U \times V$, which measure the cost of adding the edge $e \defeq (u,v)$ for
a feature vector $x \in \mc{X}$, and then output a prediction
\begin{align*}
y^\star(x) \defeq \left\{ \argmin_{y \in \biject(U,V)} \sum_{u \in U} \varphi_{u,y(u)}(x) = \sum_{u\in U, v\in V} \indic{v = y(u)} \varphi_{u,v}(x) \right\},
\end{align*}
an instance of minimum cost perfect matching solvable
in time $\mc{O}(K^ 3)$ with the Hungarian algorithm.

To efficiently adapt this approach in the context of
Alg.~\ref{alg:partial-supervised-conformal}, we assume we have trained a set
of pairwise score functions $\{ \varphi_{u,v} \mid (u,v) \in U \times V\}$,
(e.g.\ using supervised training data) and wish to conformalize with
partially supervised data, using the following score function
\begin{align}
  \label{eqn:matching-score-function}
  \score^{\text{Matching}}(x, y) \defeq \sum_{u \in U, v  \in V} \indic{v=y(u)} \varphi_{u,v}(x,y).
\end{align}

In a matching problem, weak supervision can arise under the form of a
partial matching between subsets $U_i \subset U$ and $V_i \subset V$ of the
nodes, which we write $Y_i^\weak \in \biject(U_i,V_i)$: each $u \in U_i$ has
a matching element $Y_i^\weak(u) = Y_i(u) \in V_i$.  Computing the minimum
partial scores~\eqref{eqn:min-partial-score} in
Alg.~\ref{alg:partial-supervised-conformal} is then computationally
efficient, as it reduces to yet another minimum cost perfect matching
problem:
\begin{align*}
\scorerv_i\defeq \sum_{u \in U_i} \varphi_{u, Y_i^\weak(u)}(x) + \min_{\tilde{y} \in \biject(U \setminus U_i,  V \setminus V_i)} \Biggr\{ \sum_{\substack{u \in U \setminus U_i \\  v \in V \setminus V_i}} \indic{v=\tilde{y}(u)}  \varphi_{u,v}(x)  \Biggr\}.
\end{align*}

%We further explain in Section~\ref{subsubsec:efficient-confidence-sets} how to carry out the last step of Alg.~\ref{alg:partial-supervised-conformal} in the matching case, that is how to efficiently compute the different potential matchings to include in $\what{C}_n(x)$.

In the matching case,
Alg.~\ref{alg:efficient-computation-configurations} is equivalent to finding
the $M$-best minimal weight perfect matching in a bipartite graph,
which~\citet{ChegireddyHa87} efficiently solve.  In the context of
Alg.~\ref{alg:efficient-computation-configurations}, their approach
iteratively chooses an edge $e_m \in y_{\text{ind}(m)}^m \setminus
y_{\text{ind}(m),2}^m$, then partitions the set of matchings $\mc{M} \in
\mc{Y}_{\text{ind}(m)}^m$ depending on whether $e_m \in \mc{M}$ or not. The
computation of each second-best configuration then amounts to solving at
most $K$ different perfect matching problems, resulting in an overall
$\mc{O}(MK^4)$ cost of the procedure.

\subsubsection{Ranking problems and partial labeling mechanisms}
\label{subsubsec:ranking-conformal-scores}

\newcommand{\rankscore}{\score^{\textup{rank}}}

Our second structured prediction example addresses partially supervised
ranking tasks.  The goal here is to predict a preference ranking $y \in
\mc{Y} = \biject_K$ of $K$ different items, documents, for a certain user or
query $x \in \mc{X}$, where $y(i)$ denotes the item of rank $i$. Typically,
one achieves this  by learning relevance functions $r_k: \mc{X} \to
\R$, which evaluate each item $1\le k \le K$ individually before aggregating
into a single ranking prediction~\cite{FreundIyScSi03, DuchiMaJo13, QiniLi13,
  CaoQiLiTsLi07}. We assume here that we have access to such relevance
functions.

In ranking tasks, there are two reasonable ways in which practitioners may
acquire partial supervision or user feedback.  The first
mechanism~\cite{CabannesRuBa20} assumes they only receive a subset of all
$\choose{K}{2}$ pairwise comparisons $\left( \indic{y^{-1}(i) < y^{-1}(j)}
\right)_{1\le i<j \le K}$ as a partial label, which is especially relevant
in cases where the practitioner solicits feedback from users by asking them
to compare a small number of items.  Unfortunately, carrying out the
computation~\eqref{eqn:min-partial-score} in
Alg.~\ref{alg:partial-supervised-conformal} is an NP-hard problem, namely
the minimal cost feedback arc set problem~\cite{AilonMoNe08,
  ZuylenHeJaWi07}, for which only an approximate solution is available (by
solving an integer linear program).

Another form of feedback, on which we focus in the rest of the section and
that allows to run both Algs.~\ref{alg:partial-supervised-conformal}
and~\ref{alg:efficient-computation-configurations} efficiently, instead
assumes that users only provide a fraction of their preferred ranking and
reveal $ \left( y(i) \right)_{i=1}^{K^\text{partial}}$ for some
$K^\text{partial} \le K$.  In order to construct score functions suitable to
the application of Alg.~\ref{alg:efficient-computation-configurations}, we
first introduce ranking-consistent score functions.

\begin{definition}
  A scoring function $\rankscore$ is \emph{ranking-consistent} with a set of
  relevance functions $\{ r_k : \mc{X} \to \R \}_{k \in [K]}$ if
  a for all $1\le i < j \le K$ and $y \in \biject_K$,
  \begin{align}
    \label{eqn:consistency-ranking-score}
    \rankscore(x, (i, j) \circ y)) \le \rankscore(x, y)  ~ \text{if} ~ r_{y(i)}(x) \le r_{y(j)}(x),
  \end{align}
  where $(i, j) \circ y$ denotes transposition of $i$ and $j$ in the
  permutation $y$.
\end{definition}
Such a scoring mechanism should always favor a ranking that gives a higher
rank to $y(j)$ than $y(i)$ if $r_{y(i)}(x) \le r_{y(j)}(x)$, i.e., if $y(j)$
has a greater relevance than $y(i)$.  It ensures in particular that the
$(m+1)$st best ranking is always a ``neighbor" of one of $m$ best; this is
an immediate property of the score function, as it can always increase
by swapping two elements $i$ and $j$ that are mis-ordered.

An example of ranking-consistent scoring function is the disagreement-based
scoring function~\cite{Kendall38, Kemeny59, DuchiMaJo13}
\begin{align}
  \label{eqn:kendall-ranking-score}
  \rankscore(x,y) \defeq \sum_{i<j} \psi\left( r_{y(i)}(x),  r_{y(j)}(x) \right),
\end{align}
where $\psi: \R^2 \to \R \ge 0$ is a function satisfying $\psi(a,b) = 0$
when $a \ge b$ and $\psi(a, b) > 0$ when $a < b$, non-increasing in the first
argument and non-decreasing in the second.  Unless we specify otherwise
we use $\psi(a, b) = \hinge{b - a}$ in our experiments.

Finding the configuration $y$ that minimizes the partial
score~\eqref{eqn:min-partial-score} of a ranking-consistent score function
is straightforward: it suffices to rank all the elements in $[K] \setminus
\left\{ y(i) \right\}_{i=1}^{K^\text{partial}}$ according to their relevance
scores $(r_j(x))_{j=1}^K$, and then append them to the first
$K^\text{partial}$ elements.  More interestingly, this property allows
efficiently retrieving the $M \ge 1$ best configurations with
Alg.~\ref{alg:efficient-computation-configurations}.  Throughout the loop,
we make sure that any set of permutations $\mc{Y}_j^m$ is a subset of
permutations consistent with a finite number of partial rankings (pairwise
comparisons), and that its best and second-best configurations $y_j^m$ and
$y_{j,2}^m$ only differ by a neighboring transposition of the form $(i+1,
i)$, satisfying
\begin{align}
  \label{eqn:ranking-second-best-config}
  y_{j,2}^m \defeq \argmin_{y \in \mc{Y}_j^m} \{ s(x,y) \mid \exists i \in [K],  ~ y = (i+1, i) \circ y_j^m  \}.
\end{align}
If we can guarantee this loop invariant, then there always exists $i_{j,m}
\in [K]$ such that $ y_{j,2}^m = (i_{j,m} +1, i_{j,m}) \circ y_j^m$, and we
only need to define the partition function on a smaller subset of
$2^{\mc{Y}} \times \mc{Y} \times \mc{Y}$: for any subset of permutations
$\tilde{\mc{Y}} \subset \mc{Y}$, $\tilde{y} \in \tilde{\mc{Y}}$ and $i \in
[K]$ such that $(i+1, i ) \circ \tilde{y} \in \tilde{\mc{Y}}$, we let
\begin{align}
\label{eqn:ranking-partition-function}
\begin{split}
  &\textsc{Partition}_\text{Ranking}( \tilde{\mc{Y}},  \tilde{y},  (i+1, i )
  \circ \tilde{y}) \defeq  \\
  & ~~~
  \tilde{\mc{Y}}   \cap \big\{ y \in \mc{Y} \mid y^{-1}(\tilde{y}(i)) < y^{-1}(\tilde{y}(i+1)) \big\}, 
   \tilde{\mc{Y}}   \cap \big\{ y \in \mc{Y} \mid y^{-1}(\tilde{y}(i)) < y^{-1}(\tilde{y}(i+1)) \big\},
  \end{split}
\end{align}
splitting $\tilde{\mc{Y}}$ according to whether $\tilde{y}(i)$ has a higher
rank than $\tilde{y}(i+1)$.

The next lemma, whose proof is in
Appendix~\ref{proof-lem:ranking-config-algo}, states that this partition
rule indeed guarantees that, at every step $m$ of the loop in
Algorithm~\ref{alg:efficient-computation-configurations}, the second-best
configuration in $\mc{Y}_j^m$ satisfies the
invariant~\eqref{eqn:ranking-second-best-config}.

\begin{lemma}
  \label{lem:ranking-config-algo}
  Assume the score function is
  ranking-consistent~\eqref{eqn:consistency-ranking-score} for a set of
  relevance functions $\{ r_k \}_{k=1}^K$.  Then
  Algorithm~\ref{alg:efficient-computation-configurations} with
  the~$\textsc{Partition}_\text{Ranking}$
  function~\eqref{eqn:ranking-partition-function} produces a sequence of
  partitions with second-best configurations satisfying
  equation~\eqref{eqn:ranking-second-best-config}.
\end{lemma}
\noindent
That is, Algorithm~\ref{alg:efficient-computation-configurations} is
correct.

\section{Experiments}
\label{sec:experiments}

In this section, we test our weakly supervised methods experimentally, in
different classification and regression problems, on both synthetic and real
datasets, with an emphasis on their computational efficiency and
informativeness.  The primary goal of this paper is not to provide
end-to-end models with only partially supervised data, but rather to
introduce a new form of coverage validity and show how to achieve it with
partially labeled data.  In contrast to the split-conformal
method~\cite{VovkGaSh05}, which requires fully supervised instances for both
training and validating, we only need these to train a model and form a
scoring function suitable for the application of
Alg.~\ref{alg:partial-supervised-conformal}. In some cases, standard models
already exist, such as in image classification~\cite{HeZhReSu16}.

To provide a meaningful comparison with existing conformal methods and test
for predictive set size efficiency, we use fully labeled real datasets, and
introduce different plausible forms of weak supervision on our calibration
and test sets before applying
Algorithm~\ref{alg:partial-supervised-conformal} to construct confidence
sets.  Our method displays similar behavior across all datasets and forms
of partial information. To provide a baseline, we also run a standard fully
supervised conformal scheme (FSC) using the strong labels $Y_i$
and true scores $\score(X_i,Y_i)$, which runs similarly as
Alg.~\ref{alg:partial-supervised-conformal}, but with threshold
\begin{align}
  \label{eqn:fsc-method-threshold}
  \what{t}^\text{full}_n \defeq (1+n^{-1})(1-\alpha)\text{-quantile of } \{ \score(X_i,Y_i) \}_{i=1}^n.
\end{align}
We can then estimate the gain in efficiency---in the form of decreased
confidence set sizes---that stems from the weakening of strong
coverage~\eqref{eqn:conformal-inference-guarantee} to weak
coverage~\eqref{eqn:partial-conformal-inference-coverage}.

%% \subsection{Weakly labeled multiclass problems}
\subsection{A toy classification example}
\label{sec:experiments:toy}

\definecolor{mygray}{rgb}{0.85,0.85,0.85}

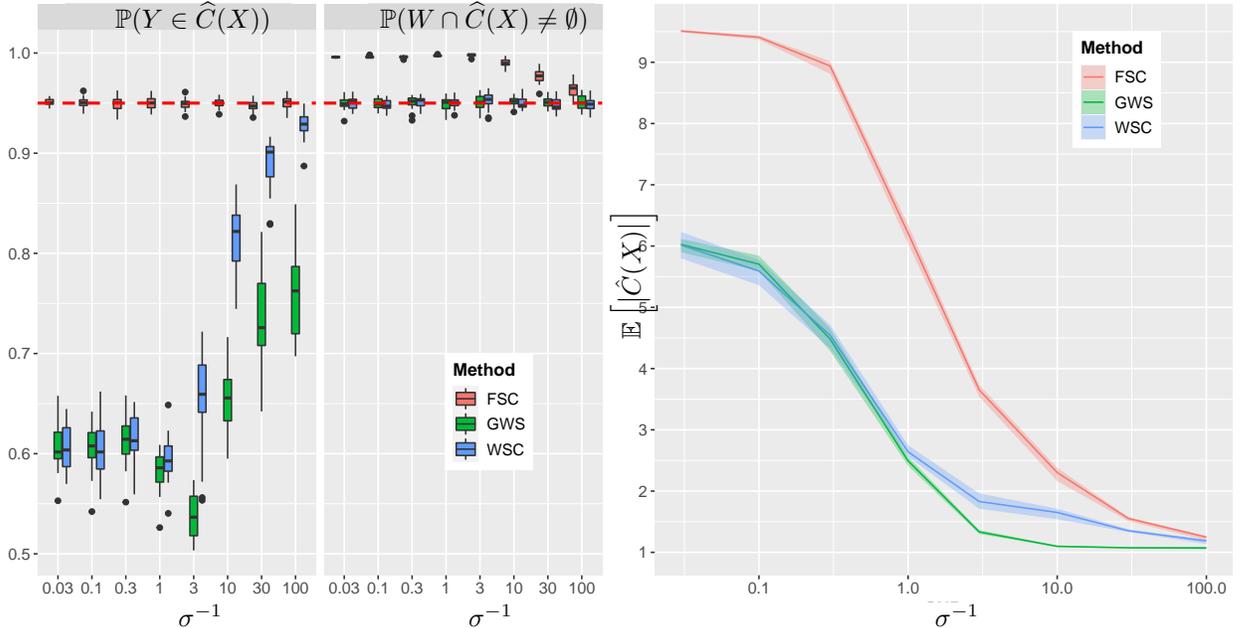
\begin{figure}
 \centering
  % To show a grid to better position drawings, uncomment grid
  \begin{overpic}[
  		%		grid, %		
  				scale=0.55]{%
     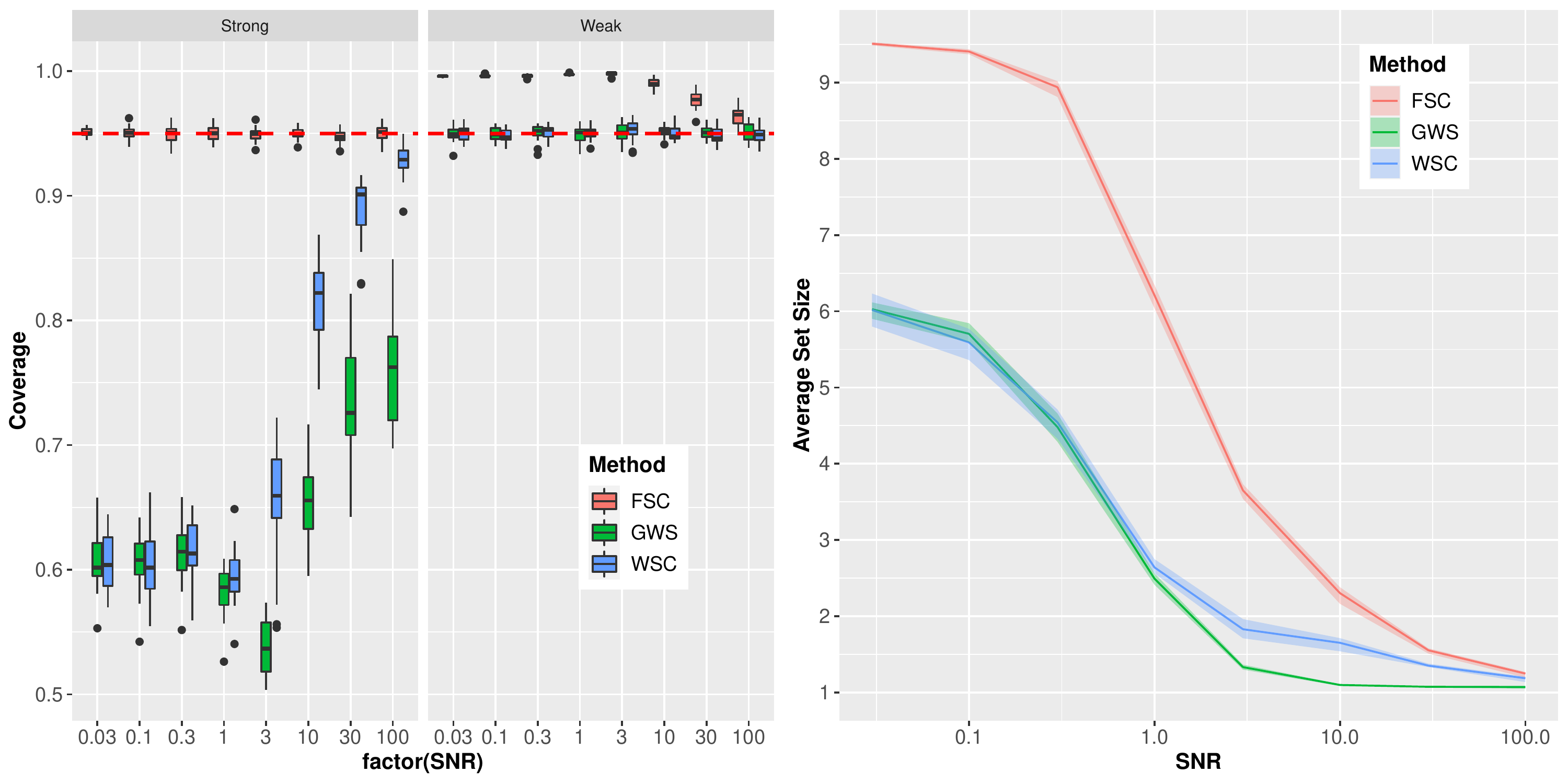}
    % Put white box over vertical axis label in left plot
    \put(-1,10){
      \tikz{\path[draw=white, fill=white] (0, 0) rectangle (.4cm, 6cm)}
    }
%    \put(0,10){\rotatebox{90}{
%        \small $\P(Y \in \what{C}(X))$ and $\P( S \cap \what{C}(X) \neq \emptyset)$}
%    }
    % Put white box over horizontal axis label in left plot
    \put(15, 0){
      \tikz{\path[draw=white, fill=white] (0, 0) rectangle (4cm, .4cm)}
    }
    \put(15, 0){
       \small $\sigma^{-1}$
    }
        \put(38, 0){
       \small $\sigma^{-1}$
    }
    % Put white box over horizontal axis label in right plot
    \put(70, 0){
      \tikz{\path[draw=white, fill=white] (0, 0) rectangle (4cm, .3cm)}
    }
    \put(75, 0){
       \small $\sigma^{-1}$
    }
    % Put white box over vertical axis label in right plot
    \put(49, 13){
      \tikz{\path[draw=white, fill=white] (0, 0) rectangle (.4cm, 5cm)}
    }
    % Vertical axis label on right plot
    \put(50, 22){\rotatebox{90}{
        \small $\E \left[ \big|\hat{C}(X)\big| \right]$}}

             \put(6,47.5){
      \tikz{\path[draw=mygray, fill=mygray] (0, 0) rectangle (3cm, .3cm)}
    }
       \put(30,47.5){
      \tikz{\path[draw=mygray, fill=mygray] (0, 0) rectangle (3cm, .3cm)}
    }
            \put(10,47.5){
     \small $\P(Y \in \what{C}(X))$
    }
       \put(31,47.5){
      \small $\P( \weakset \cap \what{C}(X) \neq \emptyset)$
    }
    
  \end{overpic}
  \caption{Results for the simulated multiclass data~\eqref{eqn:weak-multiclass-simulation-data}, over $N_\text{trials}=20$ trials. The left plot shows respectively the strong~\eqref{eqn:conformal-inference-guarantee} and weak~\eqref{eqn:partial-conformal-inference-coverage} coverage for the greedy weakly supervised (GWS),  the weakly supervised conformal (WSC) and the fully supervised conformal (FSC) confidence sets. The right plot displays the average confidence set size for these methods. }
  \label{fig:exp-multiclass-simu-coverage-avgsize}
\end{figure}

We first perform an experiment with a toy multiclass data set containing $K=10$ different classes and $d=2$ dimensional features. 
We consider a partially supervised problem on $\mc{X} \times \mc{Y} = \R^d \times [K]$ for which we wish to output valid confidence sets.
We use the following model: each potential response $y \in [K]$ has a noisy
score depending on the feature vector $X \in \R^d$ though a vector
$\theta_y^\star \in \R^d$,
\begin{align}
  \label{eqn:weak-multiclass-simulation-data}
  \{ \scorerv^\text{oracle}_{y} \}_{y \in [K]} \mid X=x  \sim \normal( \{x^T \theta_y^\star \}_{y \in [K]} , \sigma^2 I_K)
\end{align} 
Ideally, we would recover the strong label $Y \defeq \argmin_{y \in \mc{Y}}
\scorerv^\text{oracle}_{y}$, but our weakly supervised methods do not
observe $Y$ directly: instead, for a random instance-dependent threshold
$T$, we only have access to the weak set
\begin{align*}
  \weakset \defeq \{ y \in \mc{Y} \mid \scorerv^\text{oracle}_y \le T \}.
\end{align*} 
As motivation, consider a supervised learning task in which, out of all
potential responses, there is always only one ground truth, but there are
other labels that are ``good enough'' (i.e.\ have a low enough score) to answer
a certain query.  In this setting, a confidence set is weakly
valid~\eqref{eqn:partial-conformal-inference-coverage} as long as it
contains at least one label $y$ such that $\scorerv^\text{oracle}_y \le T$,
whereas it is strongly valid~\eqref{eqn:conformal-inference-guarantee} if it
contains $Y$.

We vary the signal-to-noise ratio $\sigma^{-1} \in \{ 10^{-2}, \dots, 10^2
\}$: when it is too small, no model (even an oracle one) can be highly
predictive, and a standard conformal method should provide large
uninformative confidence sets, whereas we expect our new definition of
coverage to yield smaller sets, as any label in $\weakset$ (i.e.\ with a low
enough score) provides valid coverage.

In this experiment, we compare three different methods.  The ``Greedy weakly
supervised'' (GWS) method only uses partially labeled data both when training
and conformalizating.  It first trains $K$ separate logistic regressions
with $\{ X_i \}$ as features and each $\{ \indic{y \in \weakset_i} \}$ for
all $y \in \mc{Y}$ as potential response, providing a model for $P(y \in
\weakset \mid X=x)$, and models the distribution of $\weakset$ given $X=x$
as label-independent (see
Definition~\ref{def:independence-structure-probability-distribution}).  It
then computes a nested sequence of confidence sets thanks to
Alg.~\ref{alg:greedy-weakly-supervised-optimal-scoring}, which we then feed
to the conformalization Algorithm~\ref{alg:partial-supervised-conformal}
using the nested scoring mechanism~\eqref{eqn:score-optimal-conditional}.

The second and third methods, the ``Weakly supervised conformal" (WSC) and
the ``Full supervised conformal" (FSC) methods respectively, use fully
supervised data for training: we first train a standard logistic regression
model $p_\theta(y \mid x) \propto \exp(\theta_y^T x) $ on $\{ (X_i, Y_i)
\}$, and then construct a scoring function using the Generalized Inverse
Quantile (GIQ) procedure that~\citet{RomanoSeCa20} introduce.  In the
conformalization step, the WSC method runs
Algorithm~\ref{alg:greedy-weakly-supervised-optimal-scoring} with partially
labeled calibration data, while the FSC method uses strongly labeled data to
compute the threshold $\what{t}_n^\text{full}$
in~\eqref{eqn:fsc-method-threshold}.  The threshold
$\what{t}_n$ in Alg.~\ref{alg:greedy-weakly-supervised-optimal-scoring} is
always smaller than $\what{t}_n^\text{full}$, so the FSC method returns
larger confidence sets than the WSC method.  We expect that
as the signal to noise ratio decreases, the gap between the GWS and WSC
confidence sets and the FSC confidence sets increases.

The precise experimental set-up is as follows: we simulate $n=10^4$ data
points, splitting them into training ($30\%$), calibration ($20\%$) and test
($50\%$) sets.  We draw each $\theta_y$ uniform on $\sphere^{d-1}$, $\{
X_i\}_{i=1}^n \simiid \normal(0,I_d)$, choosing weak threshold $T
\sim \uniform[\min_{y \in \mc{Y}} \{S^\text{Oracle}_{y}\},
  \max_{y \in \mc{Y}} \{S^\text{Oracle}_{y}\}]$.  We repeat the entire
process $N_\text{trials}=20$ times to account for uncertainty, presenting
our results in Figure~\ref{fig:exp-multiclass-simu-coverage-avgsize}.

As we expect, using an alternative weaker version of
coverage~\eqref{eqn:partial-conformal-inference-coverage} allows us to
significantly decrease the size of the confidence set (by up to a factor of
3), especially when the signal-to-noise ratio is small, as one must include
more classes in the confidence set to maintain strong coverage.  Indeed, we
can see that the strong coverage~\eqref{eqn:conformal-inference-guarantee}
for the GWS and WSC procedures fall well below $1-\alpha = 95\%$ in this
case, since they only strive for weak $1-\alpha$ coverage, which they
consistently achieve.  Since the GWS method aims to construct minimal
confidence sets, we expect that it produces smaller confidence sets
than the WSC method, which simply leverages an existing strongly supervised
model; we consistently observe this across different values of $\sigma$.

\subsection{Document ordering for query answering}

We now present the results of two experiments using
Alg.~\ref{alg:partial-supervised-conformal} in a ranking problem.  The first
one is a simulation of a standard ranking task, while the second focuses on
ranking documents' relevance to specific queries in the Microsoft LETOR
dataset~\cite{QiniLi13}.

\subsubsection{Ranking simulation study}
\label{sec:exp-ranking-simulation}
\begin{figure}
 \centering
  % To show a grid to better position drawings, uncomment grid
  \begin{overpic}[
  	%			 grid, %		
  				scale=0.55]{%
     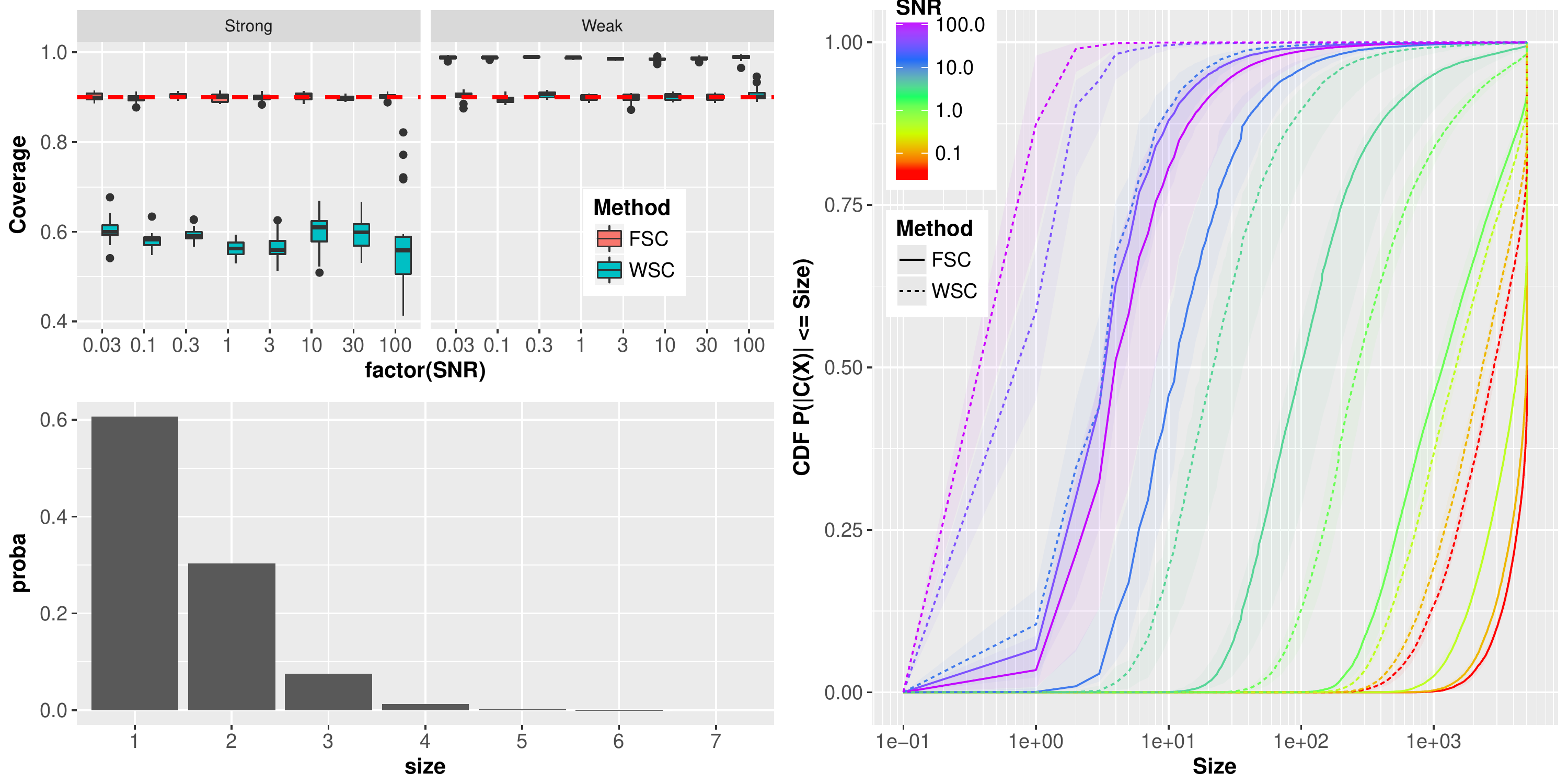}
    % Put white box over vertical axis label in left plot
    \put(-1,10){
      \tikz{\path[draw=white, fill=white] (0, 0) rectangle (.4cm, 6cm)}
    }
    \put(0,6){\rotatebox{90}{
        \small $\P( K_i^\text{partial} = t)$}
    }
%    \put(0,22){\rotatebox{90}{
%        \small $\P(Y \in \what{C}(X))$ and $\P( S \cap \what{C}(X) \neq \emptyset)$}
%    }
    % Put white box over horizontal axis label in left plot
    \put(15, 25){
      \tikz{\path[draw=white, fill=white] (0, 0) rectangle (4cm, .4cm)}
    }
    \put(14, 25.5){
       \small $\sigma^{-1}$
    }
        \put(37, 25.5){
       \small $\sigma^{-1}$
    }

    \put(25, 0){
      \tikz{\path[draw=white, fill=white] (0, 0) rectangle (4cm, .35cm)}
    }
    \put(22, 0){
       \small Size $t$
    }    
     
    % Put white box over horizontal axis label in right plot
    \put(70, 0){
      \tikz{\path[draw=white, fill=white] (0, 0) rectangle (4cm, .35cm)}
    }
    \put(75, 0){
       \small Size $t$
    }
    % Put white box over vertical axis label in right plot
    \put(49.5, 13){
      \tikz{\path[draw=white, fill=white] (0, 0) rectangle (.4cm, 5cm)}
    }
    % Vertical axis label on right plot
    \put(49.5, 20){\rotatebox{90}{
        \small $\P \left(|\what{C}(X)| \le t) \right)$}}
        
        \put(6,47.5){
      \tikz{\path[draw=mygray, fill=mygray] (0, 0) rectangle (3cm, .3cm)}
    }
       \put(30,47.5){
      \tikz{\path[draw=mygray, fill=mygray] (0, 0) rectangle (3cm, .3cm)}
    }
            \put(10,47.5){
     \small $\P(Y \in \what{C}(X))$
    }
       \put(31,47.5){
      \small $\P( \weakset \cap \what{C}(X) \neq \emptyset)$
    }
    
	\put(5, 50){
        \bf A}
    	\put(5, 25){
        \bf B}
    \put(52, 50){
        \bf C}
        
     \put(56, 49){
      \tikz{\path[draw=white, fill=white] (0, 0) rectangle (0.6cm, .4cm)}
    }
    \put(56.3, 49){
      \small $\sigma^{-1}$
    }
  \end{overpic}
  \caption{Results for the ranking simulation study~\eqref{eqn:ranking-simulation-setup-weak} over $N_\text{trials}=20$ trials.
  A: Strong~\eqref{eqn:conformal-inference-guarantee} and Weak~\eqref{eqn:partial-conformal-inference-coverage} coverage for the weakly supervised conformal (WSC) and the fully supervised conformal (FSC) confidence sets. 
  B: Density histogram of the variable $K^\text{partial}$~\eqref{eqn:ranking-simulation-setup-weak} governing the weak distribution in this example.
  C: Distribution of the confidence set size  $|\what{C}(X)|$ for different signal-to-noise ratios $\sigma^{-1}$.}
  \label{fig:exp-ranking-simu-summary}
\end{figure}
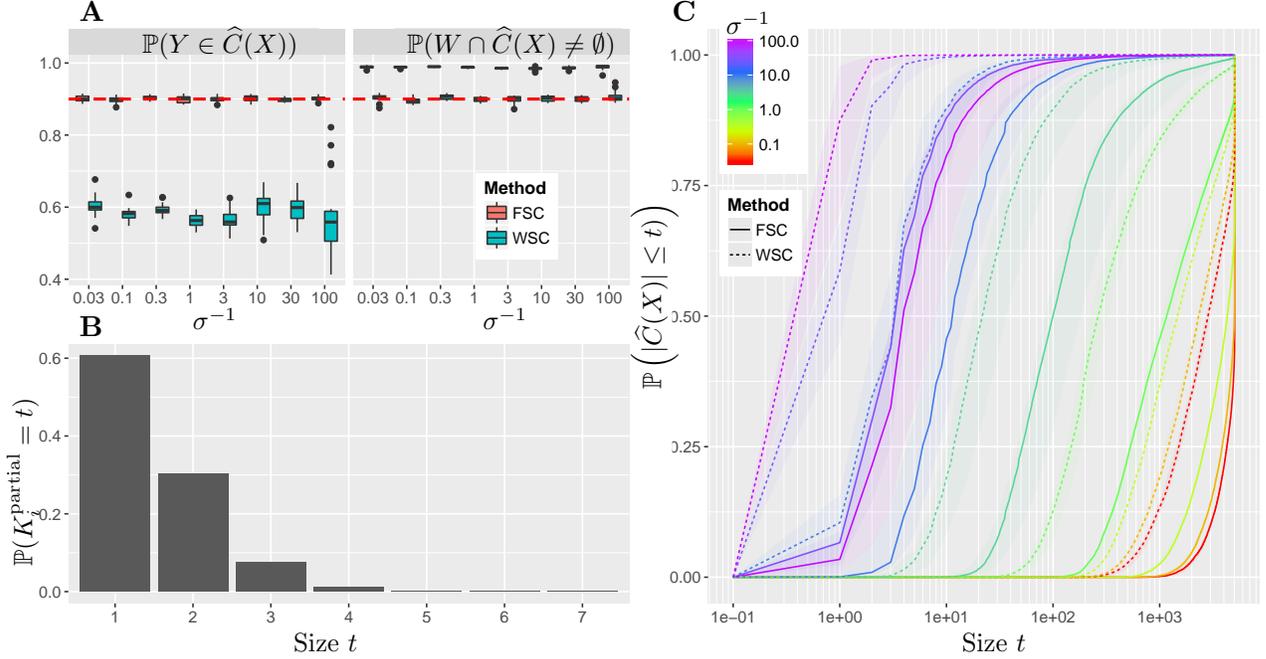

In a first simulation study, we aim to predict a ranking of labels $y \in [K]$ based on a feature vector $X \in \R^d$. 
Think here of a supervised problem where we want to rank users' preferences for a set of items.
Each user has an unknown relevance score $\scorerv^\text{Oracle}_y \in \R$ for each item $y \in [K]$, which induces a ground truth ranking over the labels:
\begin{align*}
Y \defeq \text{argsort} \{ \scorerv^\text{Oracle}_y \}_{y \in [K]} \in \biject_K.
\end{align*}
The problem is to recover this noisy ranking and produce valid confidence sets in $\biject_K$, but our weakly supervised methods do not observe the full ranking when conformalizing: they can only observe the ranking up to the $K_\text{partial} \le K$-th element, leading to the weak set
 \begin{align}
\label{eqn:ranking-simulation-setup-weak}
\weakset = \{ y  \in \biject_K \mid ~ \forall j \in [K_\text{partial}], y(j) = Y(j) \}.
\end{align}
In our experiment, we simulate $n=10^4$ i.i.d.\ different users, using the same (30,20,50) train/validation/test split as in Section~\ref{sec:experiments:toy}. 
With $K=7$ and $d=2$, we draw the user feature vector $X_i \simiid \normal(0,I_d)$,  and then conditionally on $X_i$, we produce normal item-wise relevance scores $\{ \scorerv^\text{Oracle}_{iy} \}_{i\in [n], y \in [K]}$ following the distribution~\eqref{eqn:weak-multiclass-simulation-data}. 
We finally simulate partial supervision by drawing the number of observed elements in the ranking $K_i^\text{partial}$ as $\min(K, 1+ A_i)$, where $A_i \simiid \text{Poi}(.5)$. 
The lower left panel of Figure~\ref{fig:exp-ranking-simu-summary} shows the overall distribution of this quantity: most users only reveal the first 1 to 3 items in their optimal ranking.

%We first focus on the same simulation setup~\eqref{eqn:weak-multiclass-simulation-data} with $n=10^4$, $K=7$, but now the label $Y_i \in \biject_K$ ($|\biject_K|=5040$) satisfies

%\begin{align}
%\label{eqn:ranking-simulation-setup}
%S^\text{Oracle}_{i, Y_i(1)} \le S^\text{Oracle}_{i, Y_i(2)} \le \cdots \le S^\text{Oracle}_{i, Y_i(K)},
%\end{align}
%i.e. we want to recover the ordering $Y_i$ on the scores $\{ S^\text{Oracle}_{i, j} \}_{j \in [K]}$.
%Additionally, we introduce weak supervision through partial ranking information: for each instance, there exists $K_i^\text{partial} \in [K]$ so that

We then produce strongly and weakly valid confidence sets at the $1-\alpha \defeq 90\%$ level.
We use the same scoring model for both the fully supervised conformal (FSC) and weakly supervised conformal (WSC) procedures: we learn linear individual relevance score functions $\{ r_y \}_{y \in [K]}$ (with fully supervised training data) via the ListNet procedure~\cite{CaoQiLiTsLi07}, which we briefly describe here.
Given a set of relevance scores $\{ r_y \}_{y \in \mc{Y}} \in \R^K$,  ListNet models the probability of a ranking $\pi \in \biject_K$ as    
\begin{align}
\label{eqn:listnet-model}
P_r(\pi) \defeq \prod_{y=1}^K \frac{\exp(r_{\pi(y)})}{\sum_{l=y}^k \exp(r_{\pi(l)})},
\end{align}
which gives each item $y \in \mc{Y}$ a top-1 probability (of ranking first) equal to
\begin{align*}
P^1_r(y)  \defeq P_r(\pi(1) = y) = \frac{\exp(r_{y})}{\sum_{l=1}^k \exp(r_{l})}.
\end{align*}
Given a training data set containing pairs $(X, R) \in \mc{X} \times \R^K$ of features/relevance scores,  we learn score mappings by minimizing the log-loss of the top-1 distribution over a set $\mc{F}$ of functions
\begin{align*}
\{ \hat{r}_y \}_{y \in [K]} \defeq \argmin_{ \tilde{r} \in \mc{F}^\mc{Y} } \left\{ \sum_{(X,R) \in \text{ training data}} \sum_{k=1}^K -P^1_{R}(k) \log\left( P^1_{\tilde{r}(X)}(k) \right)\right\}.
\end{align*}
%i.e we try to match the top-1 distribution induced by $P_r$ by its counterpart under $P_{\tilde{r}(x)}$.
In our experiment, we only observe the ranking (or even a fraction of),  not the true per-item relevance scores, hence, following common practice~\cite{CaoQiLiTsLi07}, we use $R_y = K - ~ \text{the rank of the item} = K - Y^{-1}(y)$ as a proxy for our observed item-wise relevance scores when training our model.

In our experimental set up, each relevance score function $r_y: \mc{X} \to \R$ ideally estimates the true conditional mean of the oracle scores, $x \mapsto x^T \theta_y^\star$.  
Given these individual scores,  we use the scoring mechanism~\eqref{eqn:kendall-ranking-score} with $\psi (x,y) \defeq (y-x)_+$ and conformalize using the strategy we describe in Section~\ref{subsubsec:ranking-conformal-scores}.
%Each confidence set is readily computable as the total number of configurations $7!=5040$ still allows to fully enumerate them.

The difference between the WSC and FSC methods is the conformalization step on calibration data: WSC runs Alg.~\ref{alg:partial-supervised-conformal} with partially supervised data to compute the score threshold $\what{t}_n$ while FSC uses strongly supervised data to return the more conservative threshold $\what{t}_n^\text{full}$ in~\eqref{eqn:fsc-method-threshold}.

Our results fit our initial expectations, in line with our first experiments: the size of the confidence set, as Figure~\ref{fig:exp-ranking-simu-summary}C shows,  benefits from the weaker definition of coverage: for any value of the signal to noise ratio $\sigma^{-1} > 0$, the WSC method produces much smaller and more informative confidence sets than the FSC method, as it only needs to include a ranking with the correct first $K_\text{partial}$ elements to provide valid coverage.  ,At the cost of the strong coverage falling below $1-\alpha$ (see Fig.~\ref{fig:exp-ranking-simu-summary}A),  and with little information (only the first $K_\text{partial} < K$ labels), the WSC method constructs predictive sets that are much smaller and yet still valid (in a weak sense). 

\subsubsection{Ranking experiment with Microsoft LETOR dataset}

\begin{figure}
 \centering
  % To show a grid to better position drawings, uncomment grid
  \begin{overpic}[
%  				 grid, %		
  				scale=0.55]{%
     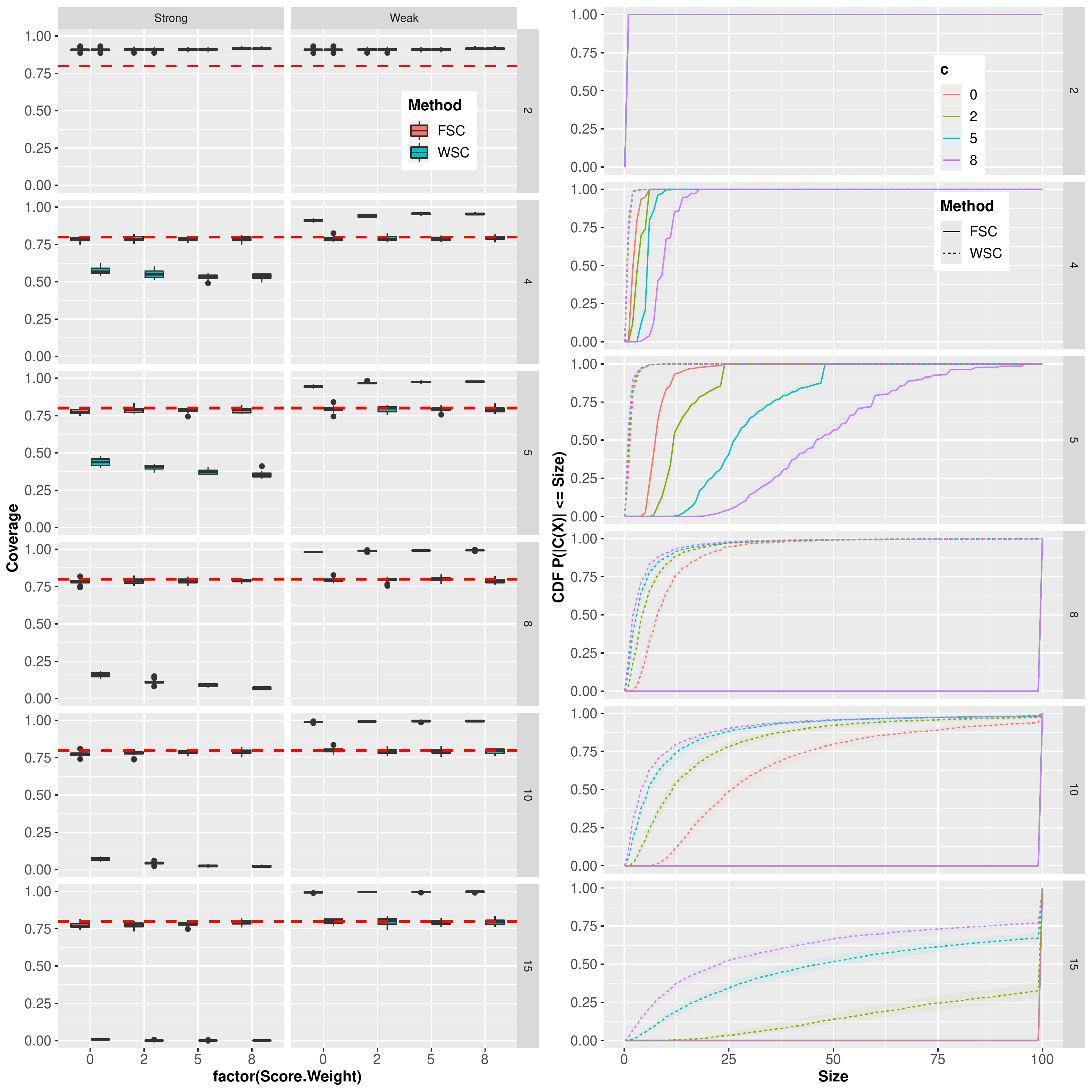}
    % Put white box over vertical axis label in left plot
    \put(-1,45){
      \tikz{\path[draw=white, fill=white] (0, 0) rectangle (.4cm, 6cm)}
    }
%    \put(0,22){\rotatebox{90}{
%        \small $\P(Y \in \what{C}(X))$ and $\P( S \cap \what{C}(X) \neq \emptyset)$}
%    }
    % Put white box over horizontal axis label in left plot
    \put(15, 0){
      \tikz{\path[draw=white, fill=white] (0, 0) rectangle (4cm, .4cm)}
    }
    \put(14, 0){
       \small $c$
    }
        \put(37, 0){
       \small $c$
    }
%         \put(81.2, 92.8){
%      \tikz{\path[draw=white, fill=white] (0, 0) rectangle (2.3cm, .3cm)}
%    }
%    \put(82.5, 93.5){
%       \small $c$
%    }

    % Put white box over horizontal axis label in right plot
    \put(70, 0){
      \tikz{\path[draw=white, fill=white] (0, 0) rectangle (4cm, .35cm)}
    }
    \put(75, 0){
       \small Size $t$
    }
    % Put white box over vertical axis label in right plot
    \put(49.5, 42){
      \tikz{\path[draw=white, fill=white] (0, 0) rectangle (.4cm, 5cm)}
    }
    % Vertical axis label on right plot
    \put(49.5, 42){\rotatebox{90}{
        \small $\P \left(|\what{C}(X)| \wedge M \le t) \right)$}}
        
        \put(6,97.5){
      \tikz{\path[draw=mygray, fill=mygray] (0, 0) rectangle (2.5cm, .3cm)}
    }
       \put(30,97.5){
      \tikz{\path[draw=mygray, fill=mygray] (0, 0) rectangle (2.5cm, .3cm)}
    }
            \put(10,97.5){
     \small $\P(Y \in \what{C}(X))$
    }
       \put(31,97.5){
      \small $\P( \weakset \cap \what{C}(X) \neq \emptyset)$
    }
    
    \put(47,88){
      \tikz{\path[draw=mygray, fill=mygray] (0, 0) rectangle (.2cm, 1cm)}
    }
	 \put(47.5, 95){\rotatebox{270}{
        \small $K=2$}
        }    

 \put(47,72){
      \tikz{\path[draw=mygray, fill=mygray] (0, 0) rectangle (.2cm, 1cm)}
    }
	 \put(47.5, 79){\rotatebox{270}{
        \small $K=4$}
        }    

 \put(47,56){
      \tikz{\path[draw=mygray, fill=mygray] (0, 0) rectangle (.2cm, 1cm)}
    }
	 \put(47.5, 63){\rotatebox{270}{
        \small $K=6$}
        }    

 \put(47,40){        
	\tikz{\path[draw=mygray, fill=mygray] (0, 0) rectangle (.2cm, 1cm)}
    }
	 \put(47.5, 47){\rotatebox{270}{
        \small $K=8$}
        }    
    
\put(47,24){    
	\tikz{\path[draw=mygray, fill=mygray] (0, 0) rectangle (.2cm, 1cm)}
    }
	 \put(47.5, 31){\rotatebox{270}{
        \small $K=10$}
        }    

 \put(47,8){
	\tikz{\path[draw=mygray, fill=mygray] (0, 0) rectangle (.2cm, 1cm)}
    }
	 \put(47.5, 15){\rotatebox{270}{
        \small $K=15$}
        }

    \put(97,88){
      \tikz{\path[draw=mygray, fill=mygray] (0, 0) rectangle (.2cm, 1.8cm)}
    }
	 \put(97.5, 95){\rotatebox{270}{
        \small $K=2$}
        }    

 \put(97,72){
      \tikz{\path[draw=mygray, fill=mygray] (0, 0) rectangle (.2cm, 1.8cm)}
    }
	 \put(97.5, 79){\rotatebox{270}{
        \small $K=4$}
        }    

 \put(97,56){
      \tikz{\path[draw=mygray, fill=mygray] (0, 0) rectangle (.2cm, 1.8cm)}
    }
	 \put(97.5, 63){\rotatebox{270}{
        \small $K=6$}
        }    

 \put(97,40){        
	\tikz{\path[draw=mygray, fill=mygray] (0, 0) rectangle (.2cm, 1.8cm)}
    }
	 \put(97.5, 47){\rotatebox{270}{
        \small $K=8$}
        }    
    
\put(97,24){    
	\tikz{\path[draw=mygray, fill=mygray] (0, 0) rectangle (.2cm, 1.8cm)}
    }
	 \put(97.5, 31){\rotatebox{270}{
        \small $K=10$}
        }    

 \put(97,8){
	\tikz{\path[draw=mygray, fill=mygray] (0, 0) rectangle (.2cm, 1.8cm)}
    }
	 \put(97.5, 15){\rotatebox{270}{
        \small $K=15$}
        }    
                          
	\put(5, 100){
        \bf A}
    \put(52, 100){
        \bf B}
  \end{overpic}
  \caption{Results for LETOR ranking dataset~\cite{QiniLi13} over $N_\text{trials}=20$ trials.  Each plot represents a different value of $K \in \{ 2, 4, 8, 10, 15\}$, the number of documents to rank, and we compare different scoring functions by varying the value of $c \in \{0, 2, 5,8\}$ in equation~\eqref{eqn:pairwise-ranking-loss}.
  A: Strong~\eqref{eqn:conformal-inference-guarantee} and Weak~\eqref{eqn:partial-conformal-inference-coverage} coverage for the weakly supervised conformal (WSC) and the fully supervised conformal (FSC) confidence sets. 
  B: Distribution of the confidence set size $|\what{C}(X)|$ for different numbers $K$ of suggested documents.  We display here the distribution of $\min(|\what{C}(X)|, M)$ for $M=100$.}
  \label{fig:exp-ranking-letor-summary}
\end{figure}
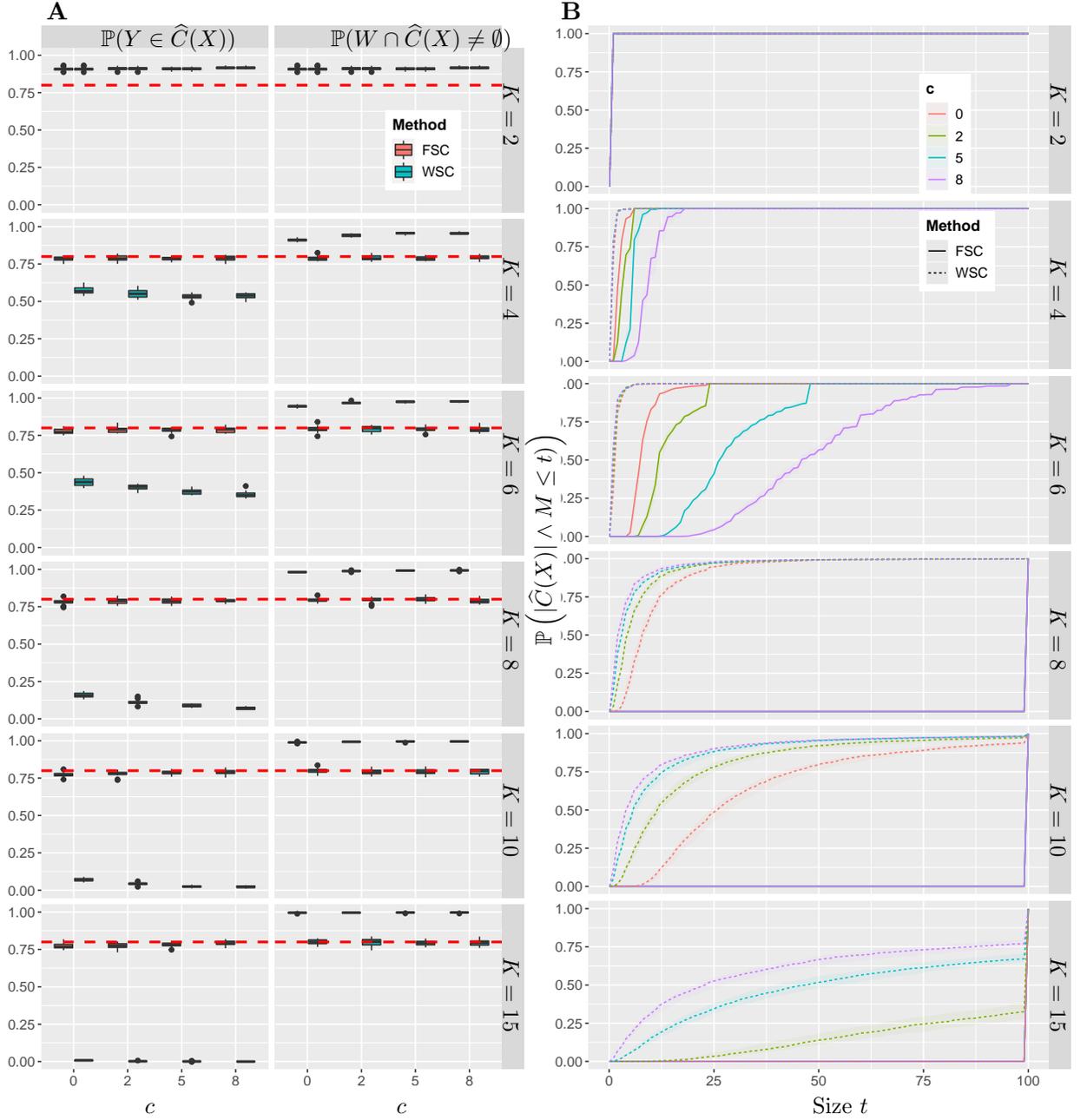

We now tackle a slightly different type of ranking problem:
we wish to rank a set of potential documents by order of relevance to a specific user query: documents more relevant to the query should occupy a higher position in the final ranking.
A search engine is a good example of such problem: a user makes a search query,  and the task is to sort Web pages that best answer that query among a (potentially large) set of potential pages.

\paragraph{Learning to rank with Microsoft LETOR dataset~\cite{QiniLi13}}
To study that problem, we design an experiment with Microsoft LETOR data set. 
For each potential query/document pair $(x,d)$,  the dataset aggregates several quantities of interest to determine whether $d$ is relevant to $x$ into a $d=46$-dimensional feature vector $\phi(x,  d) \in \R^d$.
For a query $x$ with a set of potentially relevant documents $D(x) \defeq \{ d_j  \}_{j=1}^{|D(x)|}$, our data set additionally contains a ranking $\Pi(x) \in \biject_{|D(x)|}$ that orders these documents according to their relevance.
Our goal is to retrieve that ranking using the feature vectors $ \{ \phi(x, d_j) \}_{j=1}^{D(x)}$.

\paragraph{A semi-synthetic weakly supervision set-up}
Here is how we construct weakly supervised calibration and test data sets.
For each split (calibration/test), we first sample $n=2000$ queries from the entire set of queries in LETOR validation and test datasets. 
For every query $X_i$, we select $K \in \{ 2, 4, 6, 8, 10, 20\}$ documents by first sorting $D(X_i)$ into $K$ equally sized subsets by relevance, so subset $\ell \in [K]$ contains the documents with rank $\Pi(x)_y$ for every $y \in \{ \frac{(\ell-1)|D(X_i)|}{K} +1,  \dots,  \frac{\ell |D(X_i)|}{K} \}$,
and then drawing one document from each box uniformly at random. 

This procedure ensures that there exists a significant relevance gap between any two potential documents in the query, and that the number of documents to rank is sufficient to allow reasonably-sized confidence sets. 
$\Pi(X_i)$ additionally induces a sub-ranking $Y_i \in \biject_K$ on these documents, which we treat as a strong label. 
Similarly to our approach in Section~\ref{sec:exp-ranking-simulation}, we introduce partial labels by assuming that our weakly supervised method can only access the first $K_i^\text{partial}$ elements of $Y_i$, where $K_i^\text{partial} \simiid 1+\text{Poi}(.5)$: this simulates the plausible setting where a user has given feedback on the most relevant documents to the query, but certainly not to all of them.
We repeat the entire simulation procedure $N_\text{trials} = 20$ times.

\paragraph{Building a ranking scoring function~\eqref{eqn:kendall-ranking-score}}
We next describe how we use fully supervised training data to construct the scoring function that we feed Alg.~\ref{alg:partial-supervised-conformal} with.
We first learn a linear query/document relevance function
\begin{align}
\label{eqn:relevance-function-letor}
r_\theta(x,d) \defeq \theta^T  \phi(x,d)
\end{align}
using the ListNet procedure~\eqref{eqn:listnet-model} on LETOR (fully supervised) train data 
\begin{align*}
\left( x_i,  (d_{i,j})_{j=1}^{D(x_i)},  y_i \in \biject_{D(x_i)} \right)_{i=1}^{n_\text{train}},
\end{align*}
containing 55700 different query/document pairs.

We then use a specific implementation of the score function $s^\text{Ranking}$ as in Eqn~\eqref{eqn:kendall-ranking-score}:
if we rank $K$ documents $\{ d_k \}_{k \in [K]}$ for a query $x$,  we rescale our relevance scores to the interval $[0,1]$,
\begin{align*}
\{ r_k(x) \}_{k \in [K]}  \defeq \biggr\{ \dfrac{r_\theta(x,d_k) - \min_{j \in [K]} r_\theta(x,d_j) }{\max_{j \in [K]} r_\theta(x,d_j) - \min_{j \in [K]} r_\theta(x,d_j)} \biggr\}_{k \in [K]},
\end{align*}
and then, for a choice of $c \in \{0,2,5,8 \}$,  apply the scoring mechanism~\eqref{eqn:kendall-ranking-score} with these relevance scores and pairwise comparison function
\begin{align}
\label{eqn:pairwise-ranking-loss}
\psi_c( r_1,  r_2) \defeq \exp(-c r_1 ) \left(r_2 - r_1 \right)_+.
\end{align}
In this example, we guarantee weak coverage if the true ranking $Y_i$ on the first $K_i^\text{partial}$ elements coincides with either one of the predictive rankings.
To keep the predictive set size small, we thus wish to ensure that it doesn't contain two rankings with the same first $K_i^\text{partial}$ elements (as they would be redundant): this is why we introduce the exponential term $\exp(-cr_1)$, which makes sure that when ranking all configurations by their score,  highly ranked configurations have different first elements (rather than different last elements).
%$c \ge 0$ is a parameter controlling how much we penalize inversions of elements with high relevance scores compared to low relevance scores. 
%The higher $c$ is,  the more configurations with inversions at the top of the ranking (rather than at the bottom) will have a higher ranking.
%Here, the value of $c$ is especially relevant because of the form of partial supervision: a configuration $y \in \biject_K$ guarantees weak coverage as soon as its first $K_i^\text{partial}$ elements are the same as $Y_i$, which is why there is no coverage gain if we have two redundant configurations with the same first elements in the confidence set. 
%By favoring inversions at the top of the ranking, we thus expect a higher value of the parameter $c$ to produce smaller confidence sets (with less ``useless" configurations).
To estimate the distribution of $|\what{C}(X)|$, we then compute the $M=100$ best rankings for each query using Alg~\ref{alg:efficient-computation-configurations}, and then replace the  size of the confidence set by $\min(M, |\what{C}(X)|))$, effectively truncating it to $M$. 

\paragraph{Experimental results}
We present our results in Figure~\ref{fig:exp-ranking-letor-summary}.
The confidence sets display the behavior we expect: when the number $K$ of items to classify is small,  the fully supervised conformal (FSC) and weakly supervised conformal (WSC) methods are similar, since partial labels are often equal to strong labels.  Since the overall number of configurations is small,  both methods are also able to maintain fairly small confidence sets.
On the other hand,  when $K$ grows,  the weak supervision method quickly departs from the full supervision one,  and is able to produce confidence sets that are much smaller: when $K \ge 8$,  the FSC method is unable to produce confidence sets with fewer than $100$ configurations,  as the number of configurations is large,  and the problem is inherently noisy,  especially for comparing documents with a fairly small relevance.  
The WSC (partially) overcomes that difficulty with its restrained notion of coverage,  and is able to maintain a majority of confidence sets with size smaller than $M$, at least until $K=15$. 
Of course,  this method pays a price in terms of strong coverage,  as for large $K$,  the confidence set almost never contains the actual ground true ranking.  
That said,  it may not a real issue as we are more interested in detecting which documents are actually relevant, and hence should have a higher rank, rather than correctly ordering documents with very little relevance to the query at the bottom of the list.

In addition, as we predicted, higher values of $c$ in the pairwise comparison function~\eqref{eqn:pairwise-ranking-loss} produce much smaller confidence sets by favoring more diverse rankings at the top of the list.
\subsection{Pedestrian tracking with partial matching information}

\begin{figure}
 \centering
  % To show a grid to better position drawings, uncomment grid
  \begin{overpic}[
%  				grid, %		
  				scale=0.55]{%
     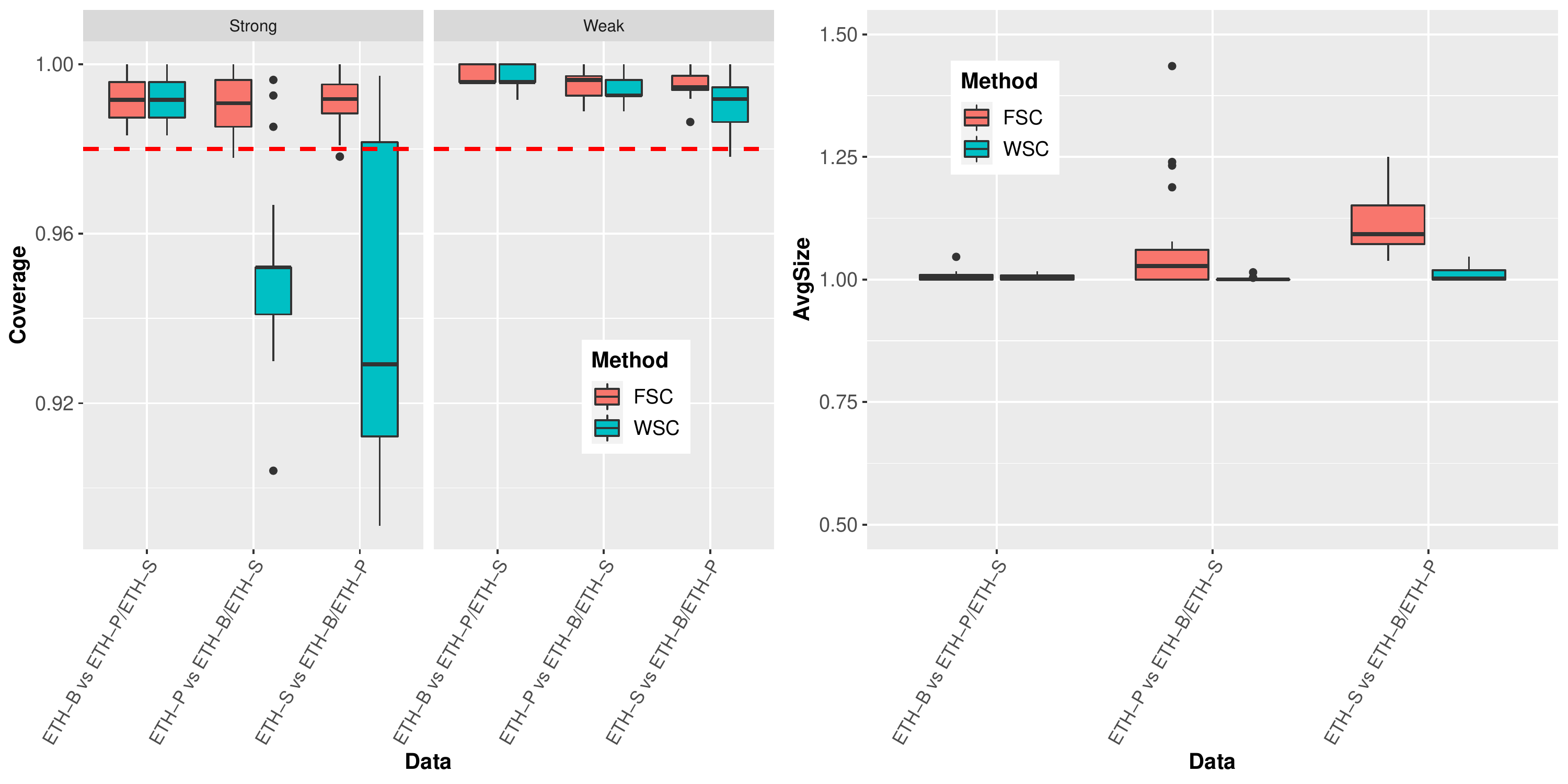}
    % Put white box over vertical axis label in left plot
    \put(-1,10){
      \tikz{\path[draw=white, fill=white] (0, 0) rectangle (.4cm, 6cm)}
    }
%     Put white box over horizontal axis label in left plot
    \put(15, 0){
      \tikz{\path[draw=white, fill=white] (0, 0) rectangle (4cm, .4cm)}
    }
    \put(15, 0){
       \small Train/Validation split
    }
    % Put white box over horizontal axis label in right plot
    \put(70, 0){
      \tikz{\path[draw=white, fill=white] (0, 0) rectangle (4cm, .35cm)}
    }
    \put(70, 0){
       \small Train/Validation split
    }
    % Put white box over vertical axis label in right plot
    \put(49, 13){
      \tikz{\path[draw=white, fill=white] (0, 0) rectangle (.4cm, 5cm)}
    }
    % Vertical axis label on right plot
    \put(49.5, 27){\rotatebox{90}{
        \small $\E \big|\hat{C}(X)\big|$}}

             \put(6,47.5){
      \tikz{\path[draw=mygray, fill=mygray] (0, 0) rectangle (3cm, .3cm)}
    }
       \put(30,47.5){
      \tikz{\path[draw=mygray, fill=mygray] (0, 0) rectangle (3cm, .3cm)}
    }
            \put(10,47.5){
     \small $\P(Y \in \what{C}(X))$
    }
       \put(31,47.5){
      \small $\P( \weakset \cap \what{C}(X) \neq \emptyset)$
    }
   
   	\put(5, 50){
        \bf A}
    \put(52, 50){
        \bf B}
  \end{overpic}
  \caption{Results for the video tracking matching dataset MOT2015~\cite{LealMiReSc15}, over $N_\text{trials}=20$ trials.  We use one video sequence for training (ETH-B, ETH-P or ETH-S) and the two others for calibration and testing.  A: Strong~\eqref{eqn:conformal-inference-guarantee} and weak~\eqref{eqn:partial-conformal-inference-coverage} coverage. B: Average confidence set size for fully supervised (FSC) and weakly supervised conformal (WSC) methods.}
  \label{fig:exp-matching-mot15}
\end{figure}
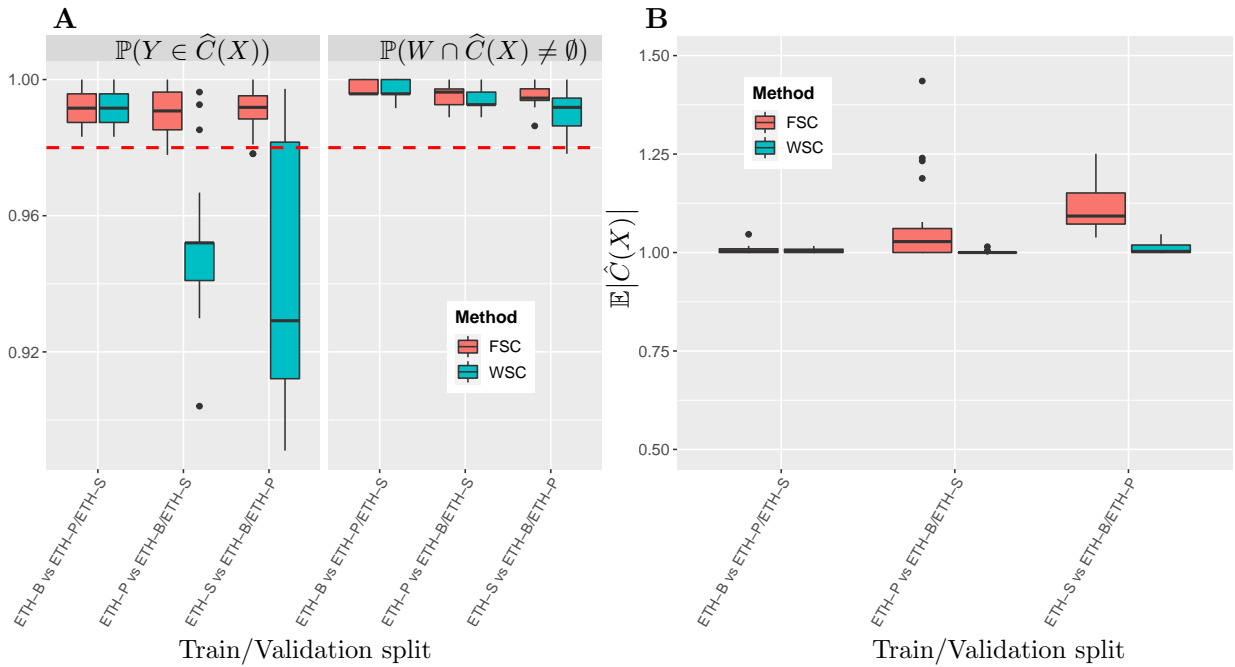

We now apply our weakly supervised conformal methods to a bipartite matching problem.
A common objective in computer vision,  relevant for instance for self-driving cars,  is to track people's trajectory throughout different time frames. 
Since we can leverage powerful algorithms~\cite{RedmonDiGiFa16} to individually detect objects in every single frame, the problem that we study here is actually a matching task where the goal is to match two sets of people appearing in two separate frames: this is an instance of a maximal matching problem.

\paragraph{Weak supervision with partial matchings}
In this context,  we expect partial supervision to come under the form of a partial matching: some people, e.g. people that are easier to track between two consecutive frames because they are in the foreground,  already have their match in the second frame, when others,  perhaps more difficult to track, are still waiting for a potential match.
Given these partially labeled instances, the goal then becomes to return confidence sets of matchings that guarantee $1-\alpha$ coverage: to provide valid weak coverage~\eqref{eqn:partial-conformal-inference-coverage}, we wish to include a configuration that contains at least all the partial matches.

\paragraph{Predicting trajectories in the MOT2D15 data set~\cite{LealMiReSc15}}
We experiment using the MOT2D15 pedestrian video tracking dataset~\cite{LealMiReSc15}.
It is a public benchmark data set that contains short street videos with pedestrians,  and for which the goal is to track each of them while they appear in the frame.  
Specifically, each frame already has a set of bounding boxes corresponding to each individual present in that frame, and we aim at matching boxes representing the same person between two consecutive frames.
Since an individual can enter or exit the frame between two consecutive images, we need to account for potentially unmatched boxes, which we do by including ``virtual" nodes in the bipartite graph, similarly to previous approaches~\cite{KimKwFeHa13, FathonyBeZhZi18}.

We use the same feature representation of~\citet{KimKwFeHa13} and~\citet{FathonyBeZhZi18}: given a pair $x \defeq (x_1, x_2)$ of two consecutive images, and two bounding boxes $u \subset x_1$, $v \subset x_2$, we compute a $d=46$ dimensional vector $\phi(x_1, x_2, u, v)$ that summarizes key features (e.g. position of the bounding box, color distribution) allowing to determine whether $u$ and $v$ represent the same person.
We then train our model using a structured S-SVM approach~\cite{TsochantaridisHoJoAl04},  following the approach of~\citet{KimKwFeHa13}.  
Such model outputs a pairwise score function $s^\text{PW} : (x,u,v) \mapsto \theta^T \phi(x_1,  x_2, u, v)$ for some vector $\theta\in \R^d$ where the feature vector $x = (x_1, x_2) \in \mc{X}$ consists of two consecutive frames,  and $(u,v)$ are two potential bounding boxes (one in each image). 

\paragraph{Experimental set-up and partial labels}
We use the data set MOT2D15 as follows.  For each of the ETH-Bahnhof,  ETH-Pedcross2 and ETH-Sunnyday video sequences,  which contain respectively 1000, 837,  and 354 consecutive images,  we select one of them for training, one of them for calibration and one of them for actual testing, using $\alpha = 0.02$ for conformalization purposes. 
We further introduce weak supervision by assuming that for each pair of images, we observe a partial matching, precisely that among the $K_i$ paired individuals, a user provided us feedback on $K_i^\text{partial} \simiid 1 + \text{Poi}(0.5)$ matches. 

For both our FSC and WSC methods, we use a translated version of the score $\score^ \text{Matching}$~\eqref{eqn:matching-score-function} with pairwise functions $\{ s_{u,v} : x \mapsto \score^\text{PW}(x,u,v) \}_{u,v}$: for each instance $(x,y) \in \mc{X} \times \mc{Y}$, we use the score
\begin{align*}
\tilde{\score}^ \text{Matching}(x,y) \defeq \score^ \text{Matching}(x,y) - \min_{\tilde{y} \in \mc{Y}} \score^\text{Matching}(x,\tilde{y})
\end{align*}
This operation simply ensures that $\min_{y \in \mc{Y}} \tilde{\score}^ \text{Matching}(x,y) = 0$ for every $x \in \mc{X}$ and thus does not change the ordering of configurations, nor the score difference between two configurations. 
We simply use it to make sure to place all the instances $X_i$ on the same scale when applying Algorithm~\ref{alg:partial-supervised-conformal} in the sense that they all have the same best achievable score. 
In particular, in the noiseless case where the true label $Y_i$ is always the minimizer of $y \mapsto s^\text{Matching}(X_i,y)$, it guarantees that $\what{C}(X)$ eventually contains a single configuration (as we should, since the score function in this case outputs perfect predictions).

\paragraph{Experimental results}
This specific problem is actually low-noise,  as it is possible to achieve a very high accuracy with the S-SVM approach, which is not so surprising as we assume perfect detection of every person thanks the bounding boxes. 
As a result, we can expect the FSC and WSC methods to output very similar confidence sets, as the configuration minimizing the score is often the true label itself.  
This is precisely what we observe in Figure~\ref{fig:exp-matching-mot15},  where both methods are actually indistinguishable and where, even with a very high confidence $1-\alpha.=0.98$, both the FSC and WSC methods are able to return confidence sets with a single configuration in average. 
We only notice a slight difference between both methods when training on the ETH-Sunnyday sequence, which contains fewer images, and hence produces slightly worse score functions.

\subsection{Prediction intervals for weakly supervised regression}
As we mentioned earlier, much of our development thus far goes beyond (finite) spaces with combinatorial structure, as we saw in the last few examples.  Therefore, we finish with a regression problem.  We consider predicting the fraction of votes in each United States county for the Democratic Party candidate in the 2016 United States presidential election, using demographic (census) data as covariates and the results of past elections.  It is common during elections for forecasters to build predictive models from both census and historical election data, as well as current polling data.  We view the historical data as strong supervision (it tells us exactly how many people voted for each candidate), and the polling data as weak supervision (as polls always come with a margin of error).  Our goal here is to fit a regression model to the strongly supervised past election data, and then form prediction intervals for the fraction of people in each county who voted Democrat by leveraging the weakly supervised polling data.  We hope by combining both we obtain valid intervals narrower than the polling margins of error.

Our data comes from the 2013--2017 American Community Survey 5-Year Estimates, which is a longitudinal survey that records demographic information about each of the 3220 United States counties.  We use 34 of the available demographic features.  The response is simply the fraction of people in each county who voted Democratic (and is publicly available).  Thus our sample size is 3220, and the number of dimensions is $d=34$.

We construct a stylized version of the above situation in the following way.  First, we split the data set into thirds: 33\% of the counties (and their associated fractions of Democratic voters) go into the training set, 33\% go into the calibration set, and the rest go into the test set; as our splits are random, they are exchangeable.  Then, we fit a Beta regression model to the strongly supervised training data.  To simulate the availability of weakly supervised polling data, we replace each calibration set response $Y_i \in [0,1]$ with a weak response $W_i \subset [0,1]$, $i=1,\ldots,n$, by forming
\begin{align}
W_i = [Y_i - Z_i, Y_i + Z_i], \quad Z_i \sim \normal(\mu, 0.0001), \quad i=1,\ldots,n, \label{eq:voting-how-weak}
\end{align}
for various values of $\mu \in \{ .01, .05, .1, .15, .2 \}$,
so that $W_i$ captures fluctuations around $Y_i$ that are roughly $\pm\mu$.
We conformalize by running Alg.~\ref{alg:partial-supervised-conformal} with the absolute error scoring function,
\[
s(x,y) = |\hat y(x) - y|,
\]
where $\hat y(x) \in [0,1]$ denotes the Beta regression model's prediction for the point $x \in \R^d$.  Here, conformalization boils down to solving a simple linear program, i.e.,
\begin{align*}
s(X_i, Y_i) = \min_{\gamma \in \R} ~ \left\{ |\hat y(X_i) - \gamma| ~ \mid
~ Y_i - Z_i \le \gamma, ~ \gamma \le Y_i + Z_i \right\}, \quad i=1,\ldots,n.
\end{align*}
Finally, we evaluate both strong~\eqref{eqn:conformal-inference-guarantee} and weak~\eqref{eqn:partial-conformal-inference-coverage} coverage on the test set (we apply the same transformation as in~\eqref{eq:voting-how-weak} to generate the weak labels for the test set).  We compute the two types of coverage, as well as the lengths of the prediction intervals, by repeating this process 20 times.  We set the miscoverage level $\alpha = .05$.

Similar to our other experiments, we find that Alg.~\ref{alg:partial-supervised-conformal} achieves weak coverage at the nominal level $.95$, for all values of $\mu$ (governing the amount of weak supervision), shown in the middle panel of Figure~\ref{fig:exp-voting-strong-cvg-lens}.  On the other hand, we expect the strong coverage to be much lower.  The left panel of Figure~\ref{fig:exp-voting-strong-cvg-lens} shows the strong coverage for Alg.~\ref{alg:partial-supervised-conformal} in teal and standard conformal inference in pink, and we can see that this is indeed the case.  However, in return, we expect the prediction intervals that Alg.~\ref{alg:partial-supervised-conformal} generates to be smaller than those coming from standard conformal inference.  The right panel of Figure~\ref{fig:exp-voting-strong-cvg-lens} shows that this is also the case: in particular, when $\mu \geq .1$, the average length of Alg.~\ref{alg:partial-supervised-conformal}'s intervals is at least three times smaller than standard conformal's, and half the length of the average weakly supervised interval $W_i$ from~\eqref{eq:voting-how-weak} ($\approx.2$).  We can also see from these figures, as in our other experiments, that Alg.~\ref{alg:partial-supervised-conformal}'s strong coverage degrades as $\mu$ grows, whereas its weak coverage improves and the length of its prediction intervals shrinks.

\begin{figure}
\begin{center}
\includegraphics[width=.32\linewidth]{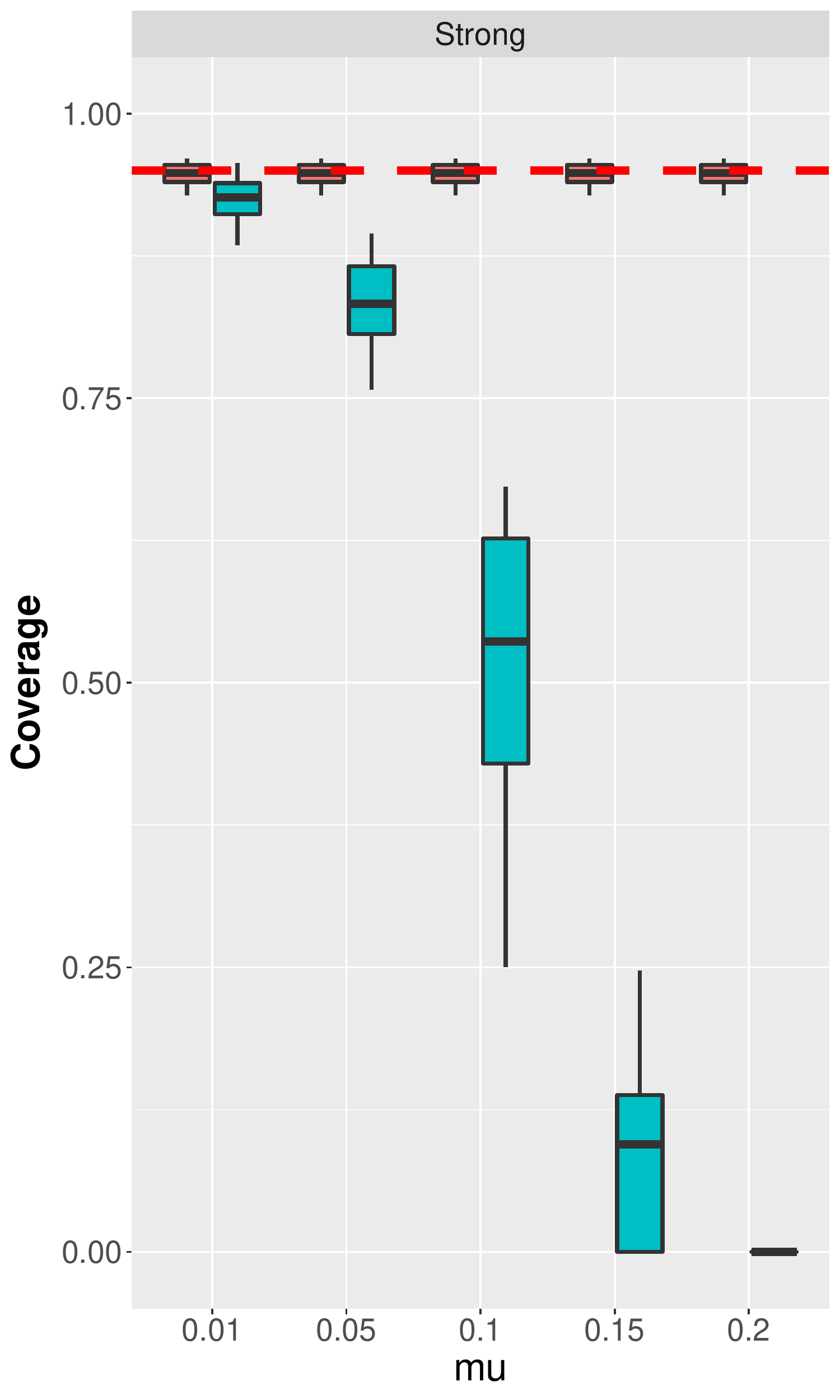}
%\hfill
\includegraphics[width=.32\linewidth]{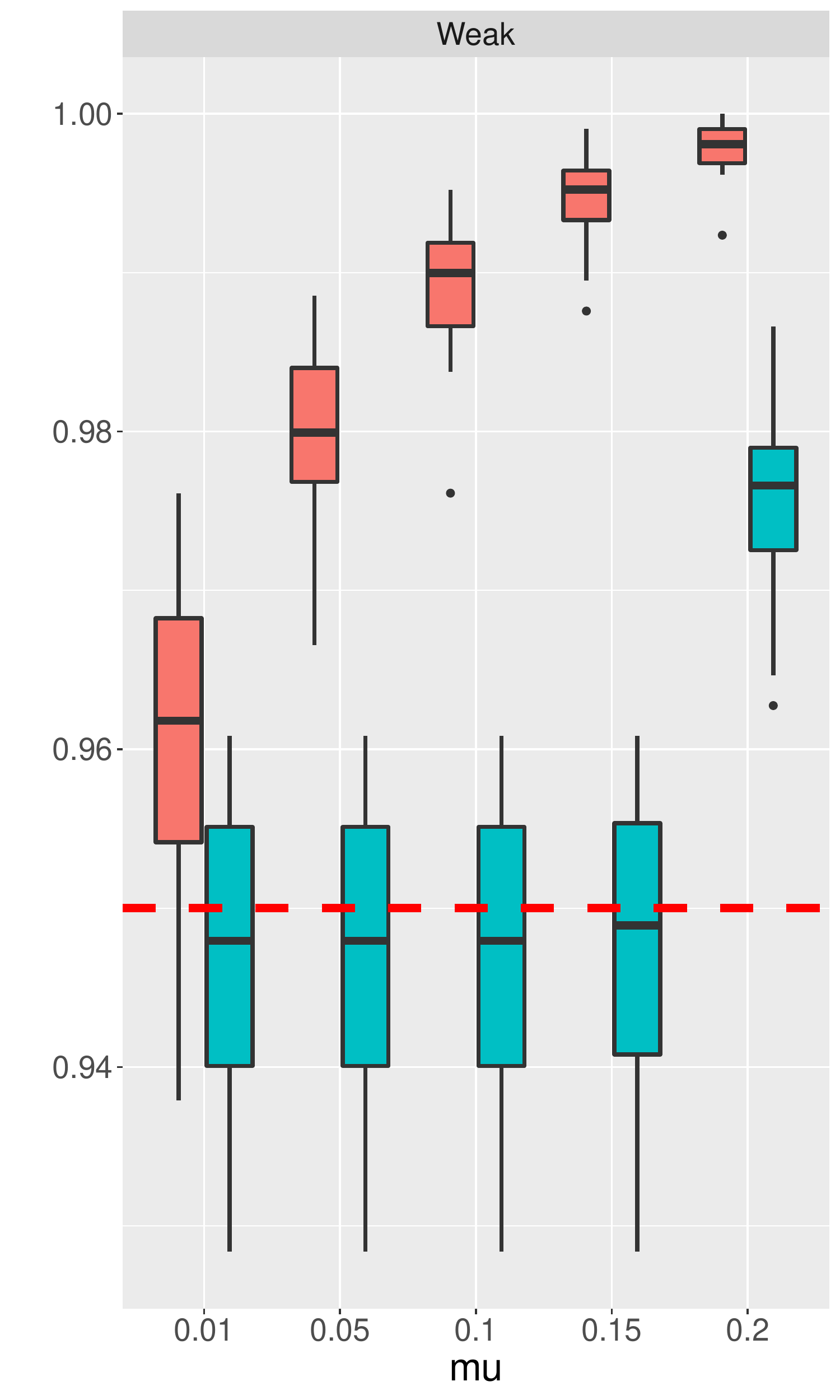}
%\hfill
\includegraphics[width=.32\linewidth]{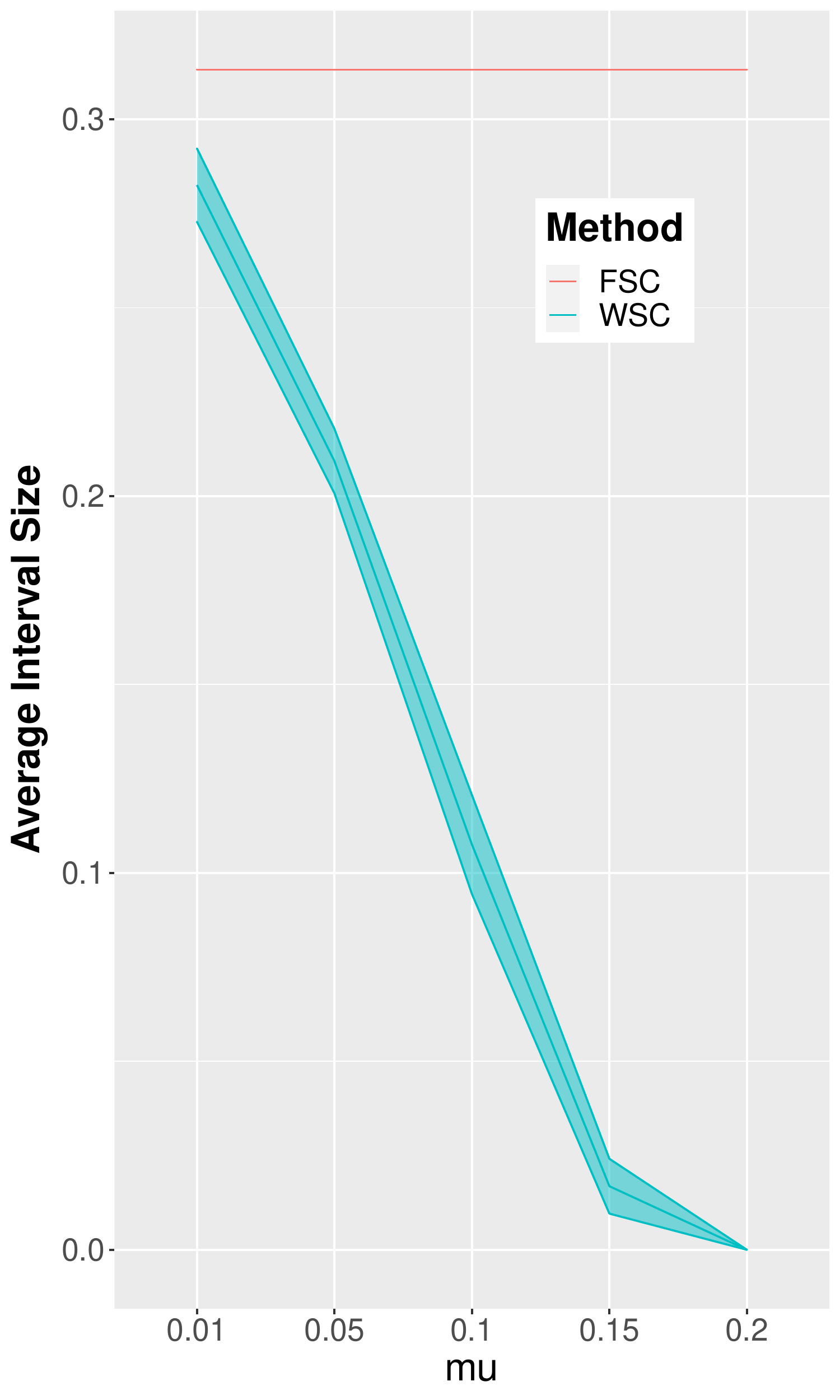}
%\vskip -0.1in
\caption{Results for the regression problem with the election data, over 20 trials.  The left panel shows the strong coverage~\eqref{eqn:conformal-inference-guarantee}, and the middle panel shows the weak coverage~\eqref{eqn:partial-conformal-inference-coverage}.  The dashed red line indicates the nominal coverage level, $1-\alpha=.95$.  The right panel shows the prediction interval lengths.  In these plots, we show Alg.~\ref{alg:partial-supervised-conformal}, denoted ``WSC'', in teal.  We show standard conformal inference, denoted ``FSC'', in pink. }
\label{fig:exp-voting-strong-cvg-lens}
\end{center}
\vskip -0.25in
\end{figure}

We view these results from a slightly different perspective in Figures~\ref{fig:exp-voting-usa-true} and~\ref{fig:exp-voting-usa-lo-hi}.  In Figure~\ref{fig:exp-voting-usa-true}, we show the true fraction of people in each county from the test set that voted Democratic.  In Figure~\ref{fig:exp-voting-usa-lo-hi}, we show the lower and upper endpoints of Alg.~\ref{alg:partial-supervised-conformal}'s prediction intervals, for a randomly chosen repetition with $\mu = .05$.  In these two figures, we color the counties with strong (predicted) Democratic majorities blue, and those with strong (predicted) Republican majorities red.  By comparing the colors, we can see that the prediction intervals only sometimes contain the true response, which is expected.  Finally, we note that the colors of the lower and upper endpoints in Figure~\ref{fig:exp-voting-usa-lo-hi} are similar, because the length of the prediction intervals is usually small.

\begin{figure}
\begin{center}
\includegraphics[width=.95\linewidth]{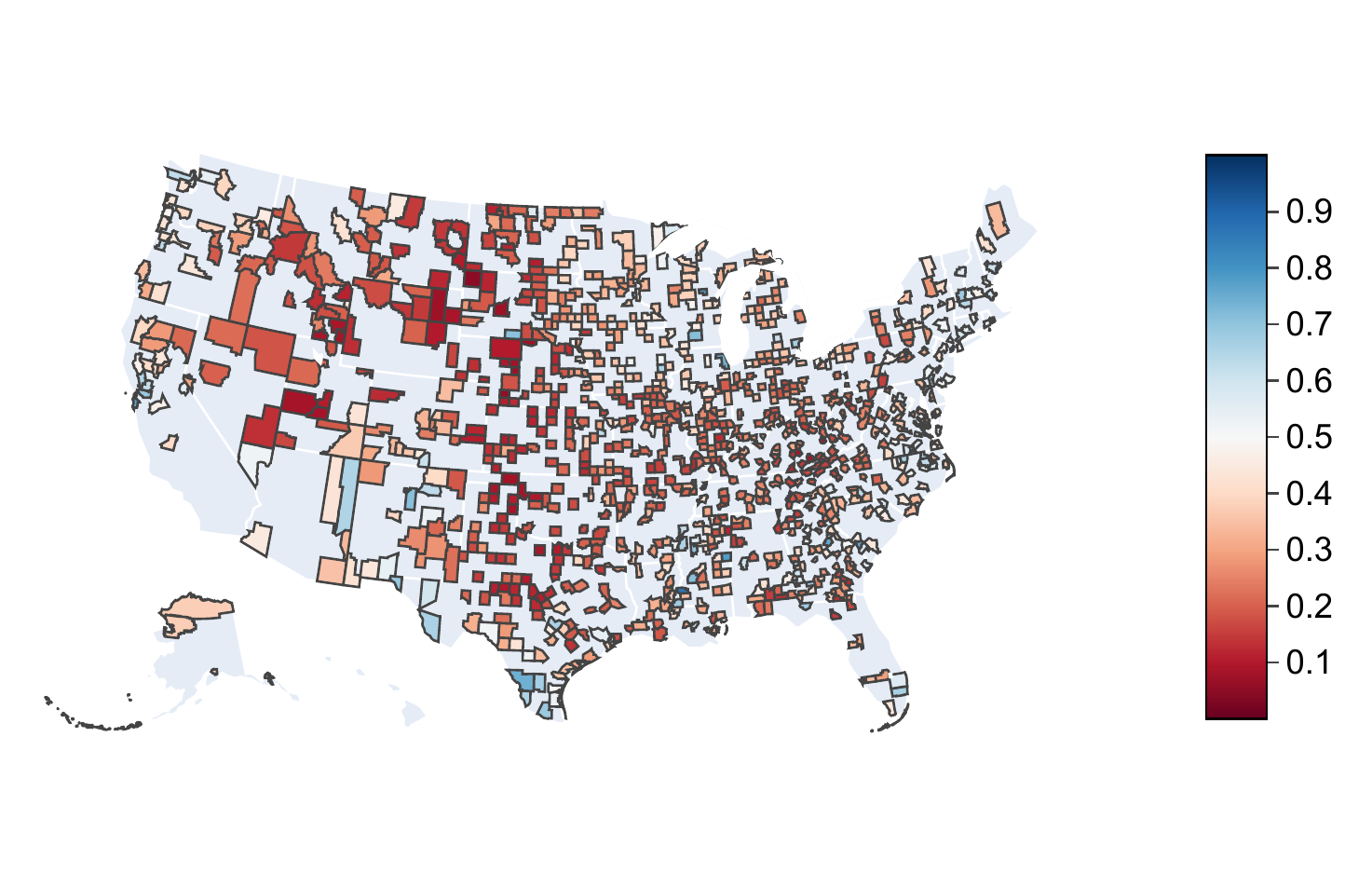}
\vskip -0.1in
\caption{Map of United States counties.  We color each county according to the actual fraction votes for the Democratic candidate in the 2016 United States presidential election.  We color counties with strong Democratic majorities blue, and those with strong Republican majorities red.  We color the counties from the training and calibration sets gray.}
\label{fig:exp-voting-usa-true}
\end{center}
\vskip -0.25in
\end{figure}

%\clearpage
\begin{figure}
\begin{center}
\includegraphics[width=.99\linewidth]{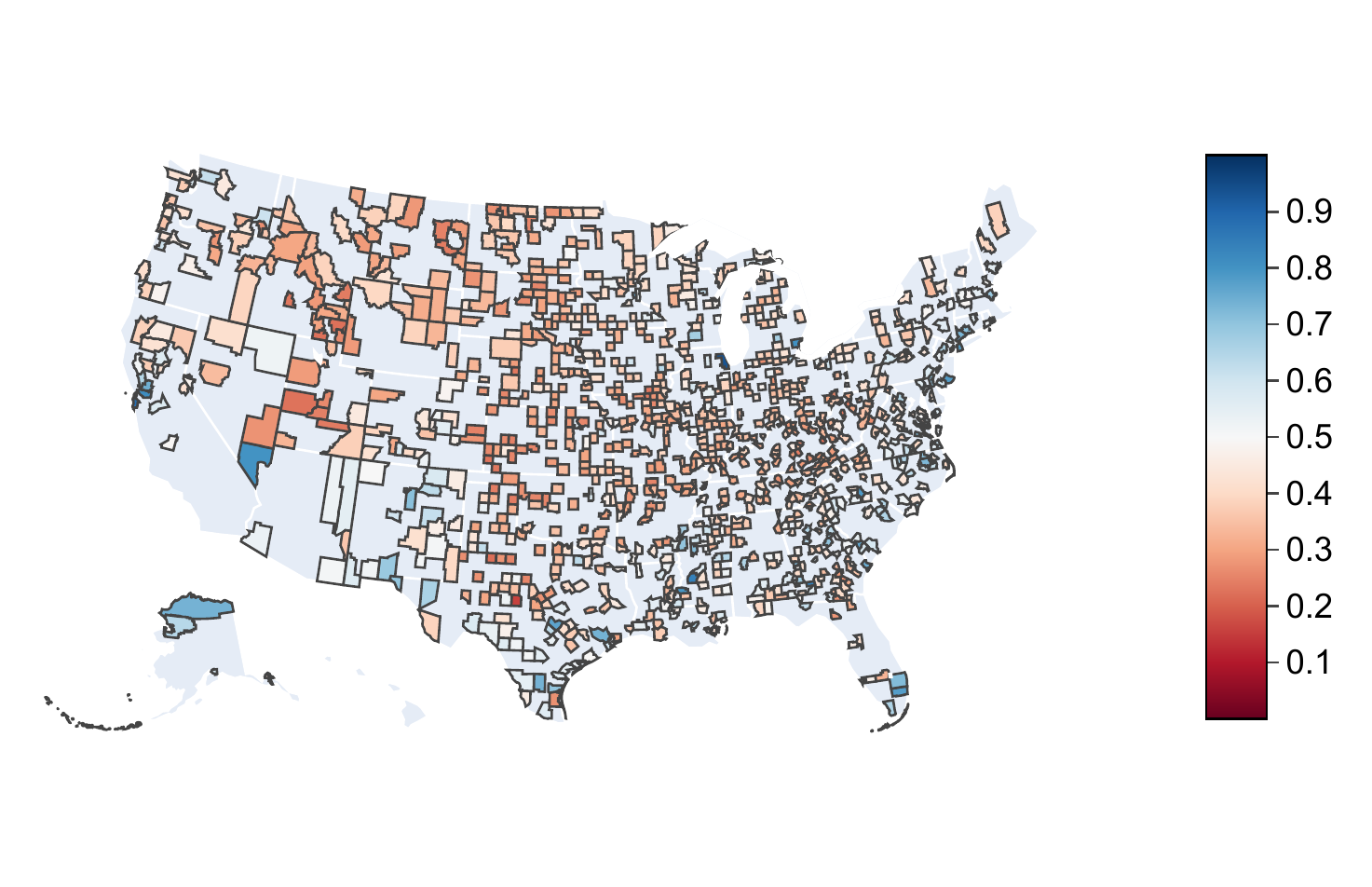} \\
\includegraphics[width=.99\linewidth]{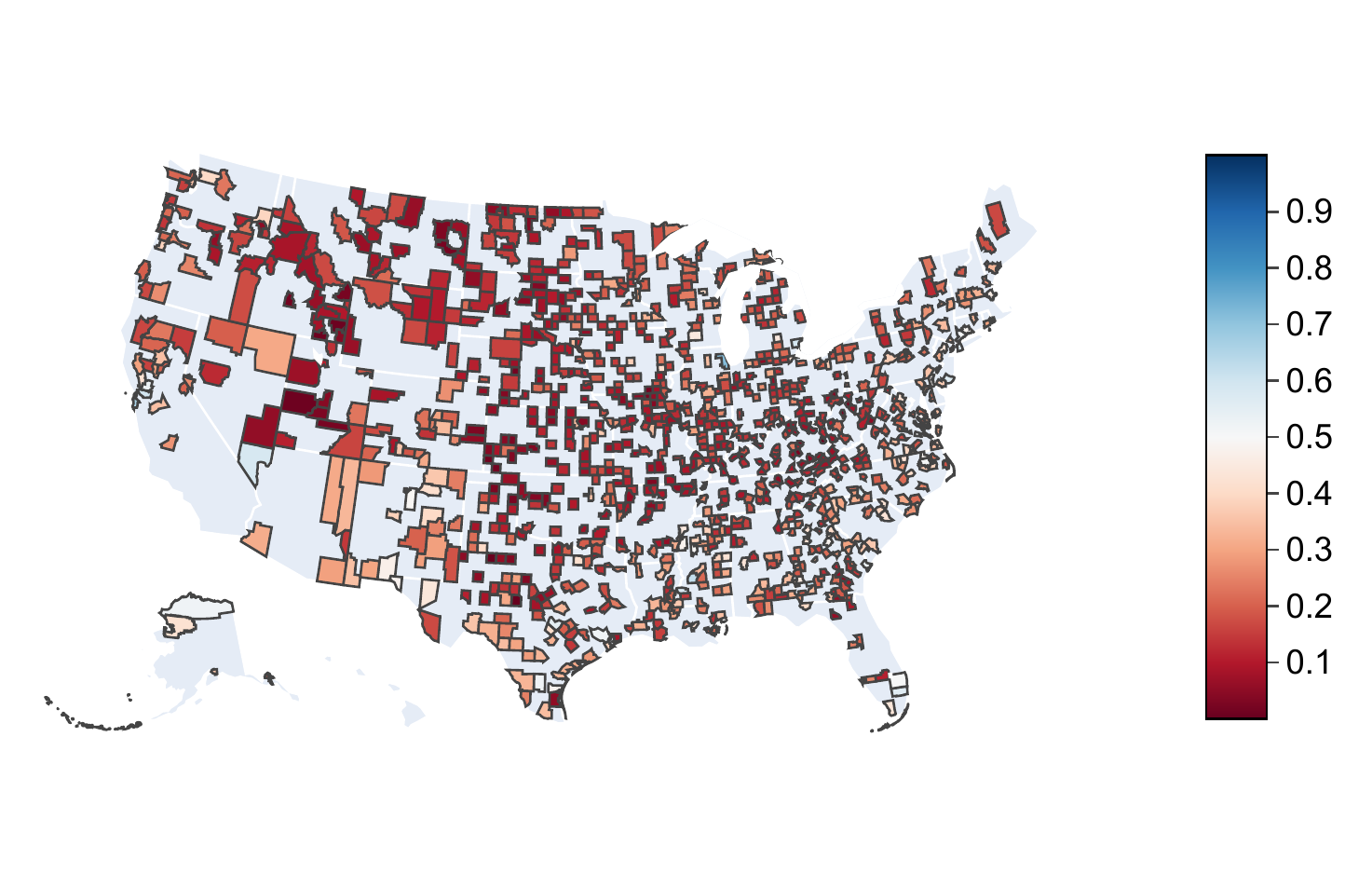}
\vskip -0.1in
\caption{Map of United States counties.  We color each county according to the value of the upper (top panel) and lower (bottom panel) endpoints of the confidence interval that Alg.~\ref{alg:partial-supervised-conformal} returns, when $\mu = .05$.   We color counties with values close to 1 blue, and those with values close to 0 red.  We color the counties from the training and calibration sets gray.}
\label{fig:exp-voting-usa-lo-hi}
\end{center}
\vskip -0.25in
\end{figure}

\section{Discussion}
\label{sec:discussion}

The new measures of coverage we develop here---tailored to partially
supervised data that may be easier to collect in many engineering and
measurement-centric scientific scenarios---help to bridge a gap between
typical conformal predictive inference methods, which require expensive
supervised data, and problems with partial supervision, whose typical focus
is on prediction but not uncertainty quantification.  Our hope is that this
work opens several avenues for future work. The new
definition~\eqref{eqn:partial-conformal-inference-coverage} a 0-1 loss-based
approach, in the sense that the confidence set $\what{C}_n$ either covers
the weakly supervised set or fails. A natural initial extension is thus
similarly to what~\citet{BatesAnLeMaJo21} propose in the strongly supervised
case, recognizing that many structured prediction problems (e.g.,
segmentation tasks or multilabel problems)  benefit from more subtle
and granular loss functions.  In the same vein, we present a few
efficient choices of scoring mechanisms for structured prediction, which
highlight the practicality and potential
application of our general methodology; it seems quite
plausible that more sophisticated scoring models could yield substantial
improvements.

In our view, one of the more exciting potential applications of this work
reposes on the (growing) centrality of partial and weakly labeled data in
statistical learning~\cite{RatnerBaEhFrWuRe17}. Whether this be from partial
reporting in surveys, or because collecting labeled data is quite expensive,
a major challenge in modern machine learning deployments and the release of
statistical models is monitoring their performance. The weaker notions of
predictive inference and coverage here, we might hope to build more
effective and applicable guardrails and uncertainty measures for modern
statistical systems, even as they are released to the world.

\bibliography{bib}

\begin{thebibliography}{56}
\providecommand{\natexlab}[1]{#1}
\providecommand{\url}[1]{\texttt{#1}}
\expandafter\ifx\csname urlstyle\endcsname\relax
  \providecommand{\doi}[1]{doi: #1}\else
  \providecommand{\doi}{doi: \begingroup \urlstyle{rm}\Url}\fi

\bibitem[Ailon et~al.(2008)Ailon, Charikar, and Newman]{AilonMoNe08}
N.~Ailon, M.~Charikar, and A.~Newman.
\newblock Aggregating inconsistent information: Ranking and clustering.
\newblock \emph{Journal of the Association for Computing Machinery},
  55\penalty0 (5), 2008.

\bibitem[Angelopoulos et~al.(2020)Angelopoulos, Bates, Malik, and
  Jordan]{AngelopoulosBaMaJo20}
A.~Angelopoulos, S.~Bates, J.~Malik, and M.~I. Jordan.
\newblock Uncertainty sets for image classifiers using conformal prediction.
\newblock \emph{arXiv:2009.14193 [cs.CV]}, 2020.

\bibitem[Barber et~al.(2021)Barber, Cand\`{e}s, Ramdas, and
  Tibshirani]{BarberCaRaTi19a}
R.~F. Barber, E.~J. Cand\`{e}s, A.~Ramdas, and R.~J. Tibshirani.
\newblock The limits of distribution-free conditional predictive inference.
\newblock \emph{Information and Inference}, 10\penalty0 (2):\penalty0 455--482,
  2021.

\bibitem[Bates et~al.(2021)Bates, Angelopoulos, Lei, Malik, and
  Jordan]{BatesAnLeMaJo21}
S.~Bates, A.~Angelopoulos, L.~Lei, J.~Malik, and M.~I. Jordan.
\newblock Distribution-free, risk-controlling prediction sets.
\newblock \emph{Journal of the Association for Computing Machinery}, 2021.

\bibitem[Boutell et~al.(2004)Boutell, Luo, Shen, and Brown]{BoutellLuShBr04}
M.~R. Boutell, J.~Luo, X.~Shen, and C.~M. Brown.
\newblock Learning multi-label scene classification.
\newblock \emph{Pattern Recognition}, 37\penalty0 (9):\penalty0 1757--1771,
  2004.

\bibitem[Cabannes et~al.(2020)Cabannes, Rudi, and Bach]{CabannesRuBa20}
V.~Cabannes, A.~Rudi, and F.~Bach.
\newblock Structured prediction with partial labelling through the infimum
  loss.
\newblock In \emph{Proceedings of the 37th International Conference on Machine
  Learning}, 2020.

\bibitem[Cao et~al.(2007)Cao, Qin, Liu, Tsai, and Li]{CaoQiLiTsLi07}
Z.~Cao, T.~Qin, T.-Y. Liu, M.-F. Tsai, and H.~Li.
\newblock Learning to rank: from pairwise approach to listwise approach.
\newblock In \emph{Proceedings of the 24th International Conference on Machine
  Learning}, pages 129--136, 2007.

\bibitem[Cauchois et~al.(2021)Cauchois, Gupta, and Duchi]{CauchoisGuDu21}
M.~Cauchois, S.~Gupta, and J.~Duchi.
\newblock Knowing what you know: valid and validated confidence sets in
  multiclass and multilabel prediction.
\newblock \emph{Journal of Machine Learning Research}, 22\penalty0
  (81):\penalty0 1--42, 2021.

\bibitem[Chegireddy and Hamacher(1987)]{ChegireddyHa87}
C.~R. Chegireddy and H.~W. Hamacher.
\newblock Algorithms for finding k-best perfect matchings.
\newblock \emph{Discrete Applied Mathematics}, 18\penalty0 (2):\penalty0
  155--165, 1987.

\bibitem[Cid-Sueiro et~al.(2014)Cid-Sueiro, Garc{\'i}a-Garc{\'i}a, and
  Santos-Rodr{\'i}guez]{CidGaSa14}
J.~Cid-Sueiro, D.~Garc{\'i}a-Garc{\'i}a, and R.~Santos-Rodr{\'i}guez.
\newblock Consistency of losses for learning from weak labels.
\newblock In \emph{Proceedings of the European Conference on Machine Learning
  and Knowledge Discovery in Databases}, pages 197--210, 2014.

\bibitem[Cour et~al.(2011)Cour, Sapp, and Taskar]{CourSaTa11}
T.~Cour, B.~Sapp, and B.~Taskar.
\newblock Learning from partial labels.
\newblock \emph{Journal of Machine Learning Research}, 12:\penalty0
  1501–--1536, 2011.

\bibitem[Deng et~al.(2009)Deng, Dong, Socher, Li, Li, and
  Fei-Fei]{DengDoSoLiLiFe09}
J.~Deng, W.~Dong, R.~Socher, L.~Li, K.~Li, and L.~Fei-Fei.
\newblock Image{N}et: a large-scale hierarchical image database.
\newblock In \emph{Proceedings of the IEEE Conference on Computer Vision and
  Pattern Recognition}, pages 248--255, 2009.

\bibitem[Donoho(2017)]{Donoho17}
D.~L. Donoho.
\newblock 50 years of data science.
\newblock \emph{Journal of Computational and Graphical Statistics}, 26\penalty0
  (4):\penalty0 745--766, 2017.

\bibitem[Duchi et~al.(2013)Duchi, Mackey, and Jordan]{DuchiMaJo13}
J.~C. Duchi, L.~Mackey, and M.~I. Jordan.
\newblock The asymptotics of ranking algorithms.
\newblock \emph{Annals of Statistics}, 41\penalty0 (5):\penalty0 2292--2323,
  2013.

\bibitem[Elisseeff and Weston(2001)]{ElisseeffWe01}
A.~Elisseeff and J.~Weston.
\newblock A kernel method for multi-labeled classification.
\newblock In \emph{Advances in Neural Information Processing Systems 14}, 2001.

\bibitem[Fathony et~al.(2018)Fathony, Behpour, Zhang, and
  Ziebart]{FathonyBeZhZi18}
R.~Fathony, S.~Behpour, X.~Zhang, and B.~Ziebart.
\newblock Efficient and consistent adversarial bipartite matching.
\newblock In \emph{Proceedings of the 35th International Conference on Machine
  Learning}, volume~80, pages 1457--1466, 2018.

\bibitem[Fernandes and Brefeld(2011)]{FernandesBr11}
E.~R. Fernandes and U.~Brefeld.
\newblock Learning from partially annotated sequences.
\newblock In \emph{Proceedings of the European Conference on Machine Learning
  and Knowledge Discovery in Databases}, pages 407--422, Berlin, Heidelberg,
  2011. Springer Berlin Heidelberg.

\bibitem[Fisch et~al.(2021)Fisch, Schuster, Jaakkola, and
  Barzilay]{FischScJaBa21}
A.~Fisch, T.~Schuster, T.~Jaakkola, and R.~Barzilay.
\newblock Efficient conformal prediction via cascaded inference with expanded
  admission.
\newblock In \emph{Proceedings of the Ninth International Conference on
  Learning Representations}, 2021.

\bibitem[Freund et~al.(2003)Freund, Iyer, Schapire, and Singer]{FreundIyScSi03}
Y.~Freund, R.~Iyer, R.~E. Schapire, and Y.~Singer.
\newblock Efficient boosting algorithms for combining preferences.
\newblock \emph{Journal of Machine Learning Research}, 4:\penalty0 933--969,
  2003.

\bibitem[Golovin et~al.(2014)Golovin, Krause, and Streeter]{GolovinKrSt14}
D.~Golovin, A.~Krause, and M.~Streeter.
\newblock Online submodular maximization under a matroid constraint with
  application to learning assignments.
\newblock \emph{arXiv:1407.1082 [cs.LG]}, 2014.

\bibitem[Gupta et~al.(2019)Gupta, Kuchibhotla, and Ramdas]{GuptaKuRa19}
C.~Gupta, A.~K. Kuchibhotla, and A.~K. Ramdas.
\newblock Nested conformal prediction and quantile out-of-bag ensemble methods.
\newblock \emph{arXiv:1910.10562 [stat.ME]}, 2019.

\bibitem[He et~al.(2016)He, Zhang, Ren, and Sun]{HeZhReSu16}
K.~He, X.~Zhang, S.~Ren, and J.~Sun.
\newblock Deep residual learning for image recognition.
\newblock In \emph{Proceedings of the IEEE Conference on Computer Vision and
  Pattern Recognition}, pages 770--778, 2016.

\bibitem[Hechtlinger et~al.(2019)Hechtlinger, P\'{o}czos, and
  Wasserman]{HechtlingerPoWa19}
Y.~Hechtlinger, B.~P\'{o}czos, and L.~Wasserman.
\newblock Cautious deep learning.
\newblock \emph{arXiv:1805.09460 [stat.ML]}, 2019.

\bibitem[H\"{u}llermeier et~al.(2008)H\"{u}llermeier, F\"{u}rnkranz, Cheng, and
  Brinker]{HullermeierFuChBr08}
E.~H\"{u}llermeier, J.~F\"{u}rnkranz, W.~Cheng, and K.~Brinker.
\newblock Label ranking by learning pairwise preferences.
\newblock \emph{Artificial Intelligence}, 172\penalty0 (16–-17):\penalty0
  1897–--1916, Nov. 2008.

\bibitem[Kemeny(1959)]{Kemeny59}
J.~G. Kemeny.
\newblock Mathematics without numbers.
\newblock \emph{Daedalus}, 88\penalty0 (4):\penalty0 577--591, 1959.

\bibitem[Kendall(1938)]{Kendall38}
M.~G. Kendall.
\newblock A new measure of rank correlation.
\newblock \emph{Biometrika}, 30\penalty0 (1/2):\penalty0 81--93, 1938.

\bibitem[Kim et~al.(2013)Kim, Kwak, Feyereisl, and Han]{KimKwFeHa13}
S.~Kim, S.~Kwak, J.~Feyereisl, and B.~Han.
\newblock Online multi-target tracking by large margin structured learning.
\newblock In \emph{Proceedings of the Asian Conference on Computer Vision},
  pages 98--111, 2013.

\bibitem[Korba et~al.(2018)Korba, Garcia, and d'Alch\'{e} Buc]{KorbaGaAl18}
A.~Korba, A.~Garcia, and F.~d'Alch\'{e} Buc.
\newblock A structured prediction approach for label ranking.
\newblock In \emph{Advances in Neural Information Processing Systems 31}, pages
  8994--9004, 2018.

\bibitem[Leal-Taix\'{e} et~al.(2015)Leal-Taix\'{e}, Milan, Reid, Roth, and
  Schindler]{LealMiReSc15}
L.~Leal-Taix\'{e}, A.~Milan, I.~Reid, S.~Roth, and K.~Schindler.
\newblock {MOTC}hallenge 2015: {T}owards a benchmark for multi-target tracking.
\newblock \emph{arXiv:1504.01942 [cs.CV]}, 2015.

\bibitem[Lei(2014)]{Lei14}
J.~Lei.
\newblock Classification with confidence.
\newblock \emph{Biometrika}, 101\penalty0 (4):\penalty0 755--769, 2014.

\bibitem[Lei and Wasserman(2014)]{LeiWa14}
J.~Lei and L.~Wasserman.
\newblock Distribution-free prediction bands for non-parametric regression.
\newblock \emph{Journal of the Royal Statistical Society, Series B},
  76\penalty0 (1):\penalty0 71--96, 2014.

\bibitem[Lei et~al.(2018)Lei, {G'S}ell, Rinaldo, Tibshirani, and
  Wasserman]{LeiGSRiTiWa18}
J.~Lei, M.~{G'S}ell, A.~Rinaldo, R.~J. Tibshirani, and L.~Wasserman.
\newblock Distribution-free predictive inference for regression.
\newblock \emph{Journal of the American Statistical Association}, 113\penalty0
  (523):\penalty0 1094--1111, 2018.

\bibitem[Mayhew et~al.(2019)Mayhew, Chaturvedi, Tsai, and Roth]{MayhewChTsRo19}
S.~Mayhew, S.~Chaturvedi, C.-T. Tsai, and D.~Roth.
\newblock Named entity recognition with partially annotated training data.
\newblock In \emph{Proceedings of the 23rd Conference on Computational Natural
  Language Learning}, pages 645--655, 01 2019.

\bibitem[Negahban et~al.(2016)Negahban, Oh, and Shah]{NegahbanOhSh16}
S.~Negahban, S.~Oh, and D.~Shah.
\newblock Iterative ranking from pair-wise comparisons.
\newblock \emph{Operations Research}, 65\penalty0 (1):\penalty0 266--287, 2016.

\bibitem[Nguyen and Caruana(2008)]{NguyenCa08}
N.~Nguyen and R.~Caruana.
\newblock Classification with partial labels.
\newblock In \emph{Proceedings of the 14th ACM SIGKDD Conference on Knowledge
  Discovery and Data Mining (KDD)}, pages 551--559, 2008.

\bibitem[{Papandreou} et~al.(2015){Papandreou}, {Chen}, {Murphy}, and
  {Yuille}]{PapandreouChMuYu15}
G.~{Papandreou}, L.~{Chen}, K.~P. {Murphy}, and A.~L. {Yuille}.
\newblock Weakly-and semi-supervised learning of a deep convolutional network
  for semantic image segmentation.
\newblock In \emph{Proceedings of the International Conference on Computer
  Vision}, pages 1742--1750, 2015.

\bibitem[Qin and Liu(2013)]{QiniLi13}
T.~Qin and T.~Liu.
\newblock Introducing {LETOR} 4.0 datasets.
\newblock \emph{arXiv:1306.2597 [cs.IR]}, 2013.

\bibitem[Ratner et~al.(2017)Ratner, Bach, Ehrenberg, Fries, Wu, and
  R\'{e}]{RatnerBaEhFrWuRe17}
A.~Ratner, S.~H. Bach, H.~Ehrenberg, J.~Fries, S.~Wu, and C.~R\'{e}.
\newblock Snorkel: rapid training data creation with weak supervision.
\newblock \emph{Proceedings of the VLDB Endowment}, 11\penalty0 (3):\penalty0
  269--282, 2017.

\bibitem[Redmon et~al.(2016)Redmon, Divvala, Girshick, and
  Farhadi]{RedmonDiGiFa16}
J.~Redmon, S.~Divvala, R.~Girshick, and A.~Farhadi.
\newblock You only look once: Unified, real-time object detection.
\newblock In \emph{Proceedings of the 26th IEEE Conference on Computer Vision
  and Pattern Recognition}, pages 779--788, 2016.

\bibitem[Romano et~al.(2019{\natexlab{a}})Romano, Barber, Sabatti, and
  Candes]{RomanoBaSaCa19}
Y.~Romano, R.~Barber, C.~Sabatti, and E.~J. Candes.
\newblock With malice towards none: Assessing uncertainty via equalized
  coverage.
\newblock \emph{arXiv:1908.05428 [stat.ME]}, 2019{\natexlab{a}}.

\bibitem[Romano et~al.(2019{\natexlab{b}})Romano, Patterson, and
  Cand\`{e}s]{RomanoPaCa19}
Y.~Romano, E.~Patterson, and E.~J. Cand\`{e}s.
\newblock Conformalized quantile regression.
\newblock In \emph{Advances in Neural Information Processing Systems 32},
  2019{\natexlab{b}}.

\bibitem[Romano et~al.(2020)Romano, Sesia, and Cand\`{e}s]{RomanoSeCa20}
Y.~Romano, M.~Sesia, and E.~J. Cand\`{e}s.
\newblock Classification with valid and adaptive coverage.
\newblock In \emph{Advances in Neural Information Processing Systems 33}, 2020.

\bibitem[Sadinle et~al.(2019)Sadinle, Lei, and Wasserman]{SadinleLeWa19}
M.~Sadinle, J.~Lei, and L.~Wasserman.
\newblock Least ambiguous set-valued classifiers with bounded error levels.
\newblock \emph{Journal of the American Statistical Association}, 114\penalty0
  (525):\penalty0 223--234, 2019.

\bibitem[Sculley et~al.(2015)Sculley, Holt, Golovin, Davydov, Phillips, Ebner,
  Chaudhary, Young, Crespo, and Dennison]{SculleyHoGoDaPhEbChYoCrDe15}
D.~Sculley, G.~Holt, D.~Golovin, E.~Davydov, T.~Phillips, D.~Ebner,
  V.~Chaudhary, M.~Young, J.-F. Crespo, and D.~Dennison.
\newblock Hidden technical debt in machine learning systems.
\newblock In \emph{Advances in Neural Information Processing Systems 28}, 2015.

\bibitem[Taskar(2005)]{Taskar05}
B.~Taskar.
\newblock \emph{Learning Structured Prediction Models: A Large Margin
  Approach}.
\newblock PhD thesis, Stanford University, 2005.

\bibitem[Taskar et~al.(2003)Taskar, Guestrin, and Koller]{TaskarGuKo03}
B.~Taskar, C.~Guestrin, and D.~Koller.
\newblock Max-margin markov networks.
\newblock In \emph{Advances in Neural Information Processing Systems 16}, 2003.

\bibitem[Tibshirani et~al.(2019)Tibshirani, Barber, Cand\`{e}s, and
  Ramdas]{TibshiraniBaCaRa19}
R.~J. Tibshirani, R.~F. Barber, E.~J. Cand\`{e}s, and A.~Ramdas.
\newblock Conformal prediction under covariate shift.
\newblock In \emph{Advances in Neural Information Processing Systems 32}, 2019.

\bibitem[Triggs and Verbeek(2008)]{TriggsVe08}
B.~Triggs and J.~Verbeek.
\newblock Scene segmentation with crfs learned from partially labeled images.
\newblock In \emph{Advances in Neural Information Processing Systems 21},
  volume~20, pages 1553--1560, 2008.

\bibitem[Tsochantaridis et~al.(2004)Tsochantaridis, Hofmann, Joachims, and
  Altun]{TsochantaridisHoJoAl04}
I.~Tsochantaridis, T.~Hofmann, T.~Joachims, and Y.~Altun.
\newblock Support vector machine learning for interdependent and structured
  output spaces.
\newblock In \emph{Proceedings of the Twenty-First International Conference on
  Machine Learning}, 2004.

\bibitem[van Rooyen and Williamson(2018)]{RooyenWi18}
B.~van Rooyen and R.~C. Williamson.
\newblock A theory of learning with corrupted labels.
\newblock \emph{Journal of Machine Learning Research}, 18\penalty0
  (228):\penalty0 1--50, 2018.

\bibitem[van Zuylen et~al.(2007)van Zuylen, Hegde, Jain, and
  Williamson]{ZuylenHeJaWi07}
A.~van Zuylen, R.~Hegde, K.~Jain, and D.~P. Williamson.
\newblock Deterministic pivoting algorithms for constrained ranking and
  clustering problems.
\newblock In \emph{Proceedings of the Eighteenth ACM-SIAM Symposium on Discrete
  Algorithms (SODA)}, pages 405--414. Society for Industrial and Applied
  Mathematics, 2007.

\bibitem[Vazirani(2001)]{Vazirani01}
V.~V. Vazirani.
\newblock \emph{Approximation Algorithms}.
\newblock Springer, 2001.

\bibitem[Vovk et~al.(2005)Vovk, Grammerman, and Shafer]{VovkGaSh05}
V.~Vovk, A.~Grammerman, and G.~Shafer.
\newblock \emph{Algorithmic Learning in a Random World}.
\newblock Springer, 2005.

\bibitem[Wolsey(1982)]{Wolsey82}
L.~A. Wolsey.
\newblock An analysis of the greedy algorithm for the submodular set covering
  problem.
\newblock \emph{Combinatorica}, 2:\penalty0 385--393, 1982.

\bibitem[Yu et~al.(2014)Yu, Jain, Kar, and Dhillon]{YuJaKaDh14}
H.-F. Yu, P.~Jain, P.~Kar, and I.~S. Dhillon.
\newblock Large-scale multi-label learning with missing labels.
\newblock In \emph{Proceedings of the 31st International Conference on Machine
  Learning}, volume~32, pages 593--–601, 2014.

\bibitem[Zhang et~al.(2017)Zhang, R\'{e}, Cafarella, {De Sa}, Ratner, Shin,
  Wang, and Wu]{ZhangReCaDeRaShWaWu17}
C.~Zhang, C.~R\'{e}, M.~Cafarella, C.~{De Sa}, A.~Ratner, J.~Shin, F.~Wang, and
  S.~Wu.
\newblock Deep{D}ive: Declarative knowledge base construction.
\newblock \emph{Communications of the ACM}, 60\penalty0 (5):\penalty0 93--102,
  2017.

\end{thebibliography}
\bibliographystyle{abbrvnat}

%-*- mode: latex -*- %

\appendix

\section{Proofs of mathematical results}

\subsection{Proofs of lower bounds on confidence set sizes}

\subsubsection{Proof of Theorem~\ref{thm:lower-bound-strong-coverage}}
\label{proof-thm:lower-bound-strong-coverage}
Suppose that $P$ is a consistent distribution on $\mc{X} \times \mc{Y} \times 2^\mc{Y}$, with marginal $\Pweak$ over $\mc{X} \times 2^\mc{Y}$, and consider a procedure $\what{C}_n$ offering $1-\alpha$ strong distribution-free coverage. 
Let $\tilde{P}$ be the distribution on $\mc{X} \times \mc{Y} \times 2^\mc{Y}$ with $\tilde{P}_\weak = \Pweak$, and where we define $\tilde{P}$ by the triple $(\tilde{X}, \tilde{Y}, \tilde{\weakset}) \sim \tilde{P}$ according to
\begin{align*}
\tilde{Y} = \argmin_{y \in \tilde{\weakset}} \left\{ p_n(\tilde{X}, y) \defeq  \P_{(X_i,\weakset)_{i=1}^n \simiid \Pweak}\left[y \in \what{C}_n(\tilde{X}) \right] \right\}.
\end{align*}
Then, $\tilde{P}$ is a consistent distribution on $\mc{X} \times \mc{Y} \times 2^\mc{Y}$, which ensures that
\begin{align*}
\P_{(X_i,Y_i, \weakset_i)_{i=1}^{n+1} \simiid \tilde{P}} \left[ Y_{n+1} \in \what{C}_n(X_{n+1})\right] \ge 1- \alpha.
\end{align*}
By definition of $\tilde{P}$, we have
\begin{align*}
\P_{(X_i,Y_i, \weakset_i)_{i=1}^{n+1} \simiid \tilde{P}} \left[ Y_{n+1} \in \what{C}_n(X_{n+1})\right] = \E_{(X_{n+1},Y_{n+1},\weakset_{n+1}) \sim \tilde{\weakset}} \left[ p_n(X_{n+1}, Y_{n+1}) \right],
\end{align*}
the law of $\what{C}_n$ is identical under $P$ or $\tilde{P}$, as it only depends on $(X_i, \weakset_i)_{i=1}^n \simiid \Pweak$.

On the other hand, we observe that when $(X_{n+1},Y_{n+1},\weakset_{n+1}) \sim \tilde{P}$,
\begin{align*}
p_n(X_{n+1}, Y_{n+1}) = \inf_{y  \in \weakset_{n+1}} p_n(X_{n+1}, y),
\end{align*}
which guarantees that
\begin{align*}
1-\alpha &\le \E_{(X_{n+1},Y_{n+1},\weakset_{n+1}) \sim \tilde{P}} \left[ p_n(X_{n+1}, Y_{n+1}) \right]  \\
&=
 \E_{(X_{n+1},Y_{n+1},\weakset_{n+1}) \sim \tilde{P}} \left[ \inf_{y  \in \weakset_{n+1}} p_n(X_{n+1}, y) \right] \\
 &= \E_{(X_{n+1}, \weakset_{n+1}) \sim \tilde{P}_\weak} \left[ \inf_{y  \in \weakset_{n+1}} p_n(X_{n+1}, y) \right] \\
 &= \E_{(X_{n+1}, \weakset_{n+1}) \sim \Pweak} \left[ \inf_{y  \in \weakset_{n+1}} p_n(X_{n+1}, y) \right].
\end{align*}

\subsubsection{Proof of Corollary~\ref{cor:pop-conf-set-big-aSS}}
\label{proof-cor:pop-conf-set-big-aSS}

Consider $(X, \weakset) \sim \Pweak$ independent of $(X_i, \weakset_i)_{i\ge
  1}$, and define $p(x,y) = \P( y \in C(x))$,
recalling the definition $p_n(x, y) = \P(y \in \what{C}_n(x))$.
By Jensen's inequality,
\begin{align*}
  \lefteqn{\E\left[ \Big|
      \inf_{y \in \weakset} p(X,  y) - \inf_{y \in \weakset} p_n(X, y) \Big|\right]
    \le \sum_{y \in \mc{Y}} \E\left[ \left| p(X,  y) - p_n(X,  y) \right|
      \right]} \\
  & \qquad\quad\qquad\qquad\qquad\qquad \le
  \E \!\left[ \sum_{y\in \mc{Y}} \left|
    \indics{y \in \what{C}_n(X)} - \indic{y \in C(X)} \right| \right]
  =
  \E\left[ |\what{C}_n(X) \setdiff C(X) | \right].
\end{align*}
Taking the limit as $n \to \infty$ and using that $\E[\inf_{y \in \weakset}
  p_n(X, y)] \ge 1-\alpha$ from
Theorem~\ref{thm:lower-bound-strong-coverage}, we must also have
\begin{align*}
  \E \left[  \inf_{y \in \weakset} p(X, y)  \right] \ge 1- \alpha.
\end{align*}
Since $ \inf_{y \in \weakset} p(X, y) \in [0,1]$, we must have $\P( \inf_{y
  \in \weakset} p(X, y)= 0) \le \alpha$, which gives the desired result:
there exists $y \in \weakset \cap \detpart_C(x) \setminus C(x)$ if
and only if $y \in \weakset \cap \{y \in \mc{Y}: p(x,y) = 0 \}$, i.e.\ if and
only if $ \inf_{y \in \weakset} p(x, y)= 0$.

\subsection{Proofs on size set optimality in weak supervision}

\subsubsection{Proof of Lemma~\ref{lemma:equivalence-score-nested}}
\label{proof:equivalence-score-nested}

Fix $\eta_0 \in (0,1)$.  Then $y \in C_{\eta_0}(x, u)$ implies that
$\score\nest(x, y, u) = \inf\{\eta \mid y \in C_\eta(x, u)\} \le \eta_0$ and
so $\score\nest(x, y, u) \le \eta_0$. Conversely, assume that
$\score\nest(x, y, u) \le \eta_0$. Then by definition of $\score\nest$,
$y \not \in C_{\eta_0}(x, u)$ if and only if
for all $\eta > \eta_0$, we have $y \in C_\eta(x, u)$ but
$y \not\in C_{\eta_0}(x, u)$, and therefore $\score\nest(x, y, u) = \eta_0$.
But of course, by continuity, $\P(\score\nest(x, y, U) = \eta_0) = 0$,
and so
\begin{equation*}
  \P(\score\nest(x, y, U) \le \eta_0
  ~ \mbox{and} ~ y \not\in C_{\eta_0}(x, U)) = 0.
\end{equation*}

\subsubsection{Proof of Proposition~\ref{prop:tree-structured-distribution}}
\label{proof-prop:tree-structured-distribution}
The case where $\weakset \mid X=x$ has a label-independent structure is immediate, hence we focus on proving the result when $\weakset \mid X=x$ has a label tree-structure~\eqref{eqn:label-tree-structure}.

We prove the result by induction on the size of $\mc{Y}^\star$, observing that the  result is immediate if $|\mc{Y}^\star|=1$.
If $|\mc{Y}^\star| = K > 1$, we assume that the result holds on sets with at most  $K-1$ elements.

We denote $P_{x}$ the law of $\weakset\mid X=x$, by $P_u$ the law of $U$ and $\P = P_{x} \otimes P_u$ their joint distribution, and similarly for their expectations.

Fix $\eta \in (0,1)$, and let $C: [0,1] \toto \mc{Y}^ \star$ be a confidence set mapping satisfying
\begin{align*}
\P(C(U) \cap \weakset \neq \emptyset) \ge \eta.
\end{align*}
We will prove that
\begin{align*}
E_u |C_\eta^\text{Greedy}(x,U)|  \le E_u |C(U)|.
\end{align*} 
We use the label ranking $y_1(x), \dots, y_K(x)$ that Alg.~\ref{alg:greedy-weakly-supervised-optimal-scoring} defines, omitting $x$ for simplicity, and consider two cases:
\begin{itemize}
\item {\bf{Case 1:}}
		 $P_u( y_K \in C(U)) = 0$.

 Then $C$ provides coverage at level $\eta$ using only the $K-1$ first labels, which also guarantees that $C_\eta^\text{Greedy}(x,u)$ only contains labels in $\{y_1, \dots, y_{K-1} \}$ (since, in that case, $J_\eta \le K-1$ in Alg.~\ref{alg:greedy-weakly-supervised-optimal-scoring}).
 The induction hypothesis applied to the distribution of $\weakset \setminus \{ y_K \}$ thus ensures that $E_u |C_\eta^\text{Greedy}(x,U)|  \le E_u |C(U)|$.
 
 \item {\bf{Case 2:}} $P_u( y_K \in C(U)) > 0$.
 
In that case, we will prove that either $P_u( y_j \in C(U)) =  1$ for all $j \in [K-1]$, or that we can build a new confidence set $C^\text{final}(U)$ such that
\begin{align*}
\P(C^\text{final}(U) \cap \weakset \neq \emptyset \mid X=x) \ge \P(C(U) \cap \weakset \neq \emptyset \mid X=x), 
~ E_u |C^\text{final}(x,U)|  =  E_u |C(U)|,
\end{align*}
and verifies either $P_u( y_j \in C^\text{final}(U)) =  1$ for all $j \in [K-1]$, or $P_u( y_K \in C^\text{final}(U)) = 0$.

%\begin{figure}[h]
%\centering
%\caption{Example of distribution for $\weakset$ given $X=x$, with $\mc{Y}^ \star = \{1, 2, 3, 4 \}$, with labels ordered as Alg.~\ref{alg:greedy-weakly-supervised-optimal-scoring}. The only acceptable configurations for $\weakset$ are each singleton, $\weakset_1 = \{y_1,y_3\}$,  $\weakset_2 = \{y_2, y_5\}$, $\weakset_0 = \{y_1, y_2, y_3, y_5 \}$, and $\mc{Y}^\star$ itself.}
%\label{fig:exmp-distribu-tree}
%\begin{forest} 
%[$\mc{Y}^\star$, tikz={\draw[{Latex}-, thick] (.north) --++ (0,1);}
%	[$\weakset_0$
%    [$\weakset_1$
%        [$y_1$] 
%        [$y_3$] 
%    ]   
%    [$\weakset_2$
%        [$y_2$]   
%        [$y_5$] 
%    ]
%    ]
%    [$y_4$]
%] 
%\end{forest}
%\end{figure}

The distribution $P_x$ induces a tree whose leaves are the labels $y_1, \dots, y_K$, and each inner node $N$ (apart from the root, which is $\mc{Y}^\star$ itself) is a subset of $\mc{Y}^\star$ such that $P_x( \weakset=N)>0$, and two nodes $N_1$ and $N_2$ share the same parent if any subset $C$ containing strictly either $N_1$ or $N_2$ such that $P_x(\weakset = C) >  0$ contains $N_1 \cap N_2$. This parent is then the smallest subset $N_p$ such that $N_1 \cap N_2 \subset N_p$ and $P_x(\weakset=N_p) > 0$. 
Each parent is then the union of all its children. 
Figure~\ref{fig:exmp-distribu-tree} provides an example of such a tree.

Defining $D(C) \defeq \{y \in \mc{Y}^\star \setminus \{ y_K \} \mid P_u( y_j \in C(U)) <  1 \} \neq \emptyset$, we consider the element $\tilde{y} \in D(C)$ sharing the lowest common ancestor with $y_K$ in the tree. 
For instance, in Figure~\ref{fig:exmp-distribu-tree}, if $D(C) = \{  y_3, y_4 \}$, then $y_D = y_3$, as their common ancestor $\weakset_0$ is lower than the common ancestor of $y_5$ and $y_4$ ($\mc{Y}^\star$ itself).

We then proceed to define $\tilde{C}(U)$ from $C(U)$ so that 
\begin{align}
\label{eqn:proof-tree-size-unchanged-1}
&\tilde{C}(U) \setminus \{ y_K, y_D \} = C(U) \setminus \{ y_K, y_D \} 
\end{align}
and
\begin{align}
 \label{eqn:proof-tree-size-unchanged-2}
&E_u |\tilde{C}(U) \cap \{ y_K, y_D \}| = E_u |C(U) \cap \{ y_K, y_D \}|,
\end{align}
but now either
\begin{align*}
P_u( y_D \in \tilde{C}(U)) = 1 ~ \text{or} ~ P_u( y_K \in \tilde{C}(U)) = 0.
\end{align*}
In practice, we do so by replacing $y_K$ by $y_D$ when $C(U) \cap \{ y_K, y_D \} = \{ y_K(x) \}$, or/and decreasing the probabilities that $\tilde{C}(U) \cap \{ y_K, y_D \} = \{ y_D, y_K \}$ and $C(U) \cap \{ y_K, y_D \} = \emptyset$, in such a way that the average size does not vary, but the probability that $C(U) \cap \{ y_K, y_D \} = \{ y_D \}$ increases.

We then proceed to check that such a change cannot hurt our coverage, while it evidently leaves the average confidence set size unchanged (because of equations~\eqref{eqn:proof-tree-size-unchanged-1} and~\eqref{eqn:proof-tree-size-unchanged-2}).

The only way we can have $C(U) \cap \weakset  \neq \emptyset$ and $\tilde{C}(U) \cap \weakset = \emptyset$ is if $\weakset = \left\{ y_K \right\}$. 
On the other hand, when $\weakset=\left\{ y_D  \right\}$, we can have $C(U) \cap \weakset  = \emptyset$ but $\tilde{C}(U) \cap \weakset \neq \emptyset$. 
Because of the definition of $y_D$ with respect to $y_K$, any other value of $\weakset$ such that $C(U) \cap \weakset \neq \emptyset$ will be such that $\tilde{C}(U) \cap \weakset  \neq \emptyset$.
In particular, by independence of $\weakset$ and $U$, we have
\begin{align*}
&\P( \tilde{C}(U) \cap \weakset \neq \emptyset) 
-  \P( C(U) \cap \weakset \neq \emptyset) \\
\begin{split} 
&\ge P_x(\weakset=\{ y_K(x)\})
 \left[ P_u( y_K \in \tilde{C}(U))
- P_u( y_K \in C(U))
\right] 
\\
&+  
P_x(\weakset=\{ y_D \} )
\left[ P_u( y_D\in \tilde{C}(U))
- P_u( y_D \in C(U))
\right]  
 \end{split} \\
 &= \left( P_u( y_D \in \tilde{C}(U))
- P_u( y_D \in C(U))
\right) \left( P_x(\weakset=\{ y_D  \}) - P_x(\weakset=\{ y_K  \}) \right),
\end{align*}
since $P_u( y_K \in \tilde{C}(U)) + P( y_D \in \tilde{C}(U))
= P_u( y_K \in C(U)) + P( y_D\in C(U))$, as the total average size does not  vary.
  
In addition, since $y_K$ gets selected last in Alg.~\ref{alg:greedy-weakly-supervised-optimal-scoring}, we know that for all $y \in \mc{Y}^\star$,
\begin{align*}
P_x \left( \weakset= \{ y_D  \} \right) \ge P_x \left(\weakset=\{ y_K  \}) \right),
\end{align*} 
which achieves to prove that
\begin{align*}
\P( \tilde{C}(U) \cap \weakset \neq \emptyset) 
\ge \P( C(U) \cap \weakset \neq \emptyset) \ge \eta.
\end{align*}

If $P_u( y_D \in \tilde{C}(U)) = 1$, then $|D(\tilde{C})| \le  |D(C)| - 1$, and we can repeat the process until we obtain a final mapping $C^\text{final}$ such that either $D(C^\text{final})=\emptyset$ or $\P( y_K(x) \in C^\text{final}(U)) = 0$. 

In the first scenario where eventually $D(C^\text{final})=\emptyset$, it means that $C^\text{final}(U)$ is either $\mc{Y}^\star$ or $\mc{Y}^\star \setminus \{ y_K \}$, 
and this is immediate to check that since $\P(C^\text{final}(U) \cap \weakset \neq \emptyset \mid X=x) \ge \eta$, we must have $P_u( y_K \in C^\text{final}(U) ) \ge P_u( y_K \in C_\eta^\text{Cond-Prox}(x,U) )$, which in turn ensures that
\begin{align*}
E_u |C_\eta^\text{Cond-Prox}(x,U)|  \le E_u |C^\text{final}(U)| = \E |C(U)|.
\end{align*}
In the second case, since $P_U( y_K \in C^\text{final}(U)) = 0$, $C^\text{final}$ is effectively a confidence set over strictly less than $K$ labels, in which case we can apply the induction hypothesis to conclude that 
\begin{align*}
E_u |C_\eta^\text{Cond-Prox}(x,U)|  \le E_u |C^\text{final}(U)| = E_u |C(U)|.
\end{align*}
\end{itemize} 

\subsection{Proofs of algorithms validity}

\subsubsection{Proof of Lemma~\ref{lem:ranking-config-algo}}
\label{proof-lem:ranking-config-algo}
We prove the result by proving that if we run Algorithm~\ref{alg:efficient-computation-configurations} with $M=K!$,  defining at each step $y_{j,2}^m$ as in equation~\eqref{eqn:ranking-second-best-config},  then at each step $m \le K!$ of the algorithm,  $\{ \mc{Y}_j^{m} \}_{j \in [m]}$ is a valid partition of $\mc{Y}$ that satisfies the following conditions:
\begin{enumerate}
\item For each $j \in [m]$, we have $\mc{Y}_j^m = \{ y_j^m \}$ if and only if there exists no $i \in [K-1]$ such that $(i+1 ~ i) \circ y_j^m \in \mc{Y}_j^m$.
\item If $\mc{Y}_j^m \neq \{ y_j^m \}$ then $s(x, y_j^m) \le s(x, y_{j,2}^m)$.
\end{enumerate}
If the partition satisfies these two conditions at every step, then we can run the algorithm until step $m=K!$,  at which point it returns a partition $ \{ \mc{Y}_j^{K!} \}_{k=1}^{K!}$ such that $s(x, y_{1}^{K!}) \le \dots \le s(x,y_{K!}^{K!})$.
Now, since we have $s_j^m = s_j^{K!}$ for every $1 \le j \le m$, we conclude that, at each step $m \in [K!]$, we have
\begin{align*}
s(x, y_1^m) \le s(x,y_2^m) \le \dots \le s(x,y_m^m) \le \min_{y \in \mc{Y} \setminus \{ y_1^m,  \dots, y_m^m \} } s(x,y),
\end{align*}
which proves the validity of the algorithm.

This is of course true for $m =1$, since the best configuration simply ranks $r_k(x)$ in decreasing order. 

\begin{enumerate}
\item By definition of $y_{\text{ind}(m),2}^m$ and $y_{\text{ind}(m)}$,  $\{ \mc{Y}_j^{m+1} \}_{j=1}^{m+1}$ is a valid partition of $\mc{Y}$ such that each $y_j^{m+1} \in \mc{Y}_j^{m+1}$,  if $\{ \mc{Y}_j^{m} \}$ is itself a valid partition (and $m < K!$), so long as we can prove that if $\mc{Y}_j^m \ne \{ y_j^m \}$,  then there must exist $\alpha \in [K-1]$ such that $(\alpha+1 ~ \alpha) \circ y_j^m \in \mc{Y}_j^m$ (i.e the algorithm does not get stuck and terminates). 
But this is immediate as, if for all $\alpha \in [K-1]$,  we have $(\alpha+1 ~ \alpha) \circ y_j^m \notin \mc{Y}_j^m$, then it must be by construction that
\begin{align*}
\mc{Y}_j^m \subset \bigcap_{\alpha \in [K-1]} \{ y \in \mc{Y} \mid y^{-1}(y_j^m(\alpha)) < y^{-1}(y_j^m(\alpha+1)) \} = \{ y_j^m \}.
\end{align*}
Therefore, at each step $m \le K!$ of the algorithm, $\{ \mc{Y}_j^{m} \}_{j \in [m]}$ is a valid partition of $\mc{Y}$, and the algorithm terminates.  

\item On the other hand,  it requires more care to justify why, 
if we set, for all $j \le m$,  
\begin{align*}
y_{j,2}^m \defeq \argmin_{y \in \mc{Y}_j^m} \{ s(x,y) \mid \exists \alpha \in [K-1],  ~ y = (\alpha + 1 ~   \alpha) \circ y_j^m  \}
\end{align*}
then we should always have
\begin{align}
\label{eqn:ranking-config-algo-property}
s(x,y_j^m) \le s(x,y_{j,2}^m)
\end{align}
for all $j \in [m]$, i.e. why any permutation of the form $(\alpha ~ \alpha + 1) \circ y_j^m$ that belongs to $\mc{Y}_j^m $ cannot strictly decrease the score $s(x,y)$.

Equation~\eqref{eqn:ranking-config-algo-property} actually results from a crucial property of the score function~\eqref{eqn:kendall-ranking-score},  which ensures that if $s(x, y) < s(x, (\alpha ~ \alpha + 1) \circ y)$, then it must hold that
\begin{align*}
r_{y(\alpha)}(x) < r_{y(\alpha + 1)}(x),
\end{align*}
i.e.  the elements $y(\alpha)$ and $y(\alpha + 1)$ were originally in the wrong order in $y$.

But, since we start the partition process with $y_1^1$ such that $r_{y_1^1(1)}(x) \ge \dots \ge r_{y_1^1(k)}(x)$, i.e with all elements in the correct order, it is straightforward to check that at any time $m$, there cannot exist a permutation of the form $(\alpha ~ \alpha + 1) \circ y_j^m$ that also belongs to $\mc{Y}_j^m $
such that 
\begin{align*}
r_{y_j^m(\alpha)}(x) < r_{y_j^m(\alpha + 1)}(x):
\end{align*}
if that were the case, then there would exist $l \le j$ such that $y_j^m(\alpha)$ and $y_j^m(\alpha + 1)$ were the elements exchanged at time $l$ when creating the partition $\{ \mc{Y}^{l+1}_i \}_{i=1}^{l+1}$.
In turn, since $y_j^m \in  \mc{Y}_j^m$, this would guarantee that
\begin{align*}
\mc{Y}_j^m \subset \{ y \in \mc{Y} \mid y^{-1}(y_j^m(\alpha)) < y^{-1}(y_j^m(\alpha + 1)) \},
\end{align*}
and thus that $(\alpha ~ \alpha + 1) \circ y_j^m \notin \mc{Y}_j^m $.
This guarantees that any configuration $(\alpha ~ \alpha + 1) \circ y_j^m \in \mc{Y}_j^m $ satisfies
\begin{align*}
s(x, (\alpha ~ \alpha + 1) \circ y_j^m) \ge s(x, y_j^m),
\end{align*}
and thus that either $\mc{Y}_j^m = \{ y_j^m \}$, or
\begin{align*}
s(x,y_{j,2}^m) \ge s(x,y_j),
\end{align*}
which concludes the proof.
\end{enumerate}

%\subsubsection{Proof of Corollary~\ref{cor:submodular-random-bound}}
%\label{proof-cor:submodular-random-bound}
%XXX

\end{document}